\documentclass[letterpaper]{article} 
\usepackage{aaai25}  
\usepackage{times}  
\usepackage{helvet}  
\usepackage{courier}  
\usepackage[hyphens]{url}  
\usepackage{graphicx} 
\usepackage{natbib}  
\usepackage{caption} 
\usepackage{algorithm}
\usepackage{algorithmic}

\usepackage{newfloat}
\usepackage{listings}
\DeclareCaptionStyle{ruled}{labelfont=normalfont,labelsep=colon,strut=off} 
\floatstyle{ruled}
\newfloat{listing}{tb}{lst}{}
\floatname{listing}{Listing}
\title{Spectra of Cardinality Queries over Description Logic Knowledge Bases}
\author {
	Quentin Manière\textsuperscript{\rm 1,\rm 2},
	Marcin Przybyłko\textsuperscript{\rm 1,\rm 3}
}
\affiliations {
	\textsuperscript{\rm 1}Department of Computer Science, Leipzig University, Germany\\
	\textsuperscript{\rm 2}Center for Scalable Data Analytics and Artificial Intelligence (ScaDS.AI), Dresden/Leipzig, Germany\\
    \textsuperscript{\rm 3}Institute of Informatics, University of Warsaw, Poland\\
	quentin.maniere@uni-leipzig.de, m.przybylko@uw.edu.pl
}

\usepackage{macros}
\usepackage{tikzit}
\usepackage{xspace}

\newcommand{\naturals}{\mathbb{N}}
\newcommand{\ninfty}{\naturals^{\infty}}

\newtheorem{problem}{Problem}

\begin{document}

	\maketitle

	\begin{abstract}
Recent works have explored the use of counting queries coupled with Description Logic ontologies.
The answer to such a query in a model of a knowledge base is either an integer or $\infty$, and its spectrum is the set of its answers over all models.
While it is unclear how to compute and manipulate such a set in general, we identify a class of counting queries whose spectra can be effectively represented.
Focusing on atomic counting queries, we pinpoint the possible shapes of a spectrum over $\ALCIF$ ontologies:
 they are essentially
the subsets of $\NN \cup \{ \infty\}$ closed under addition.
For most sublogics of $\ALCIF$, we show that possible spectra enjoy simpler shapes, being $\llbracket m, \infty \rrbracket$ or variations thereof.
To obtain our results, we refine constructions used for finite model reasoning 
and notably rely on a cycle-reversion technique for the Horn fragment of $\ALCIF$.
We also study the data complexity of computing the proposed effective representation and establish the $\FPwithlogcalls$-completeness of this task under several settings.
	\end{abstract}

	\begin{toappendix}

		To facilitate the navigation between results, and notably to easily identity sublogics to which a statement may or may not apply, Figure~\ref{figure:dl} above gives an overview of the dependencies between investigated description logics.

	\end{toappendix}

\section{Introduction}
\label{section-introduction}

Ontology-mediated query answering (OMQA) uses ontologies to offer a user-friendly vocabulary for formulating queries or to encapsulate domain knowledge that can be utilized
to retrieve more comprehensive answers \cite{DBLP:journals/jods/PoggiLCGLR08,DBLP:conf/ijcai/XiaoCKLPRZ18}.
Ontologies expressed in Description Logics (DLs), a family of knowledge representation languages underpinning the OWL Web Ontology Language, have received special attention \cite{DBLP:journals/jair/ArtaleCKZ09,introduction-to-dl}, and the core reasoning task of OMQA, the query answering task, has been extensively studied 
for conjunctive queries (CQs) and unions thereof.
Under the OMQA framework, answering CQs is addressed by considering every possible model of the knowledge base (KB), that is every extension of the data that satisfies the ontology, and returning so-called \emph{certain answers}, \emph{i.e.}\ answers true in every model.

A recent line of research has explored ways of leveraging OMQA to support counting queries, a well-known class of aggregate queries that allows to perform analytics on data.
Several semantics for such queries have been investigated, differing on how the possibility of multiple models is taken into account.
In \cite{count-guarded}, this has been addressed by returning the number of certain answers to a query, while in \cite{DBLP:conf/cikm/CalvaneseKNT08} an epistemic semantics was adopted -- enforcing the counting operator to only involve known data values and making it possible to use the usual notion of certain answers.

In this paper we adopt the semantic of \cite{kostylevreutter:count,DBLP:conf/ijcai/BienvenuMT20} that defines a counting query as a CQ in which some variables have been designated as \emph{counting variables}.
The answer to a counting query in a model of the KB is then obtained as the number of different assignments for the counting variables when considering every possible homomorphism of the CQ into the model.
Finding uniform bounds on those answers, \emph{i.e.}\ model-independent bounds, has been viewed as a notion of certain answers and is now well-understood for a variety of DLs \cite{DBLP:conf/ijcai/CalvaneseCLR20,BMT22,maniere:thesis}.
The following example highlights that even the tightest uniform bounds give, in general, a poor over-approximation of the set of answers.

\begin{example}
	\label{example:squares}
	Consider an empty KB $\kb$ and a counting query $q$ asking for the number of pairs $(z_1, z_2)$ such that $z_1$ and $z_2$ are friends of $\istyle{alice}$, that is $q = \exists z_1\ \exists z_2\ \textrm{friendOf}(z_1, \istyle{alice}) \land \textrm{friendOf}(z_2, \istyle{alice})$.
	Clearly, the set of possible answers to $q$ across models of $\kb$ is $\{ n^2 \mid n \in \mathbb{N} \} \cup \{ \infty \}$.
	The tightest uniform bounds on this set are given by the interval $\llbracket 0, \infty \rrbracket$.
\end{example}

Rather than aiming for an over-approximation of the set of possible answers, we intent to give a comprehensive description of this subset of $\ninfty = \{0,1,2,\dots, \infty\}$
that we call the \emph{spectrum of the counting query}, inspired by the notion of spectrum of a formula that refers to the possible cardinalities of its models \cite{fagin1974generalized,DBLP:conf/csl/DurandFL97}.
We investigate the possible shapes of these spectra for \emph{counting conjunctive queries (CCQs)} mentioned above and for ontologies expressed in the $\ALCIF$ DL.
This expressive DL is contained in $\mathcal{SHIQ}$, in which traditional CQ answering is well-understood \cite{birte2008cqa-in-SHIQ,lutz08-fork-and-alci}, and supports functionality constraints whose interactions with counting queries have never been studied to the best of our knowledge (those proposed in \cite{DBLP:conf/ijcai/CalvaneseCLR20} and denoted $\mathcal{N}^-$ being much more restricted).

One of the challenges encountered in our work is to clarify how to represent spectra.
Indeed, the set of possible answers of a CCQ across models of a KB might, \emph{a priori}, be an arbitrary set of natural numbers, and thus hard to describe by means other than providing the CCQ-KB couple.
We aim to identify classes of ontology-mediated queries (OMQs) whose spectra admit an \emph{effective representation}.
By effective, we intend a representation that is (i) finite, ideally with a size that can be bounded by the size of the OMQ, (ii) independent from the specific ontology language and (iii) spectrum membership
can be efficiently tested, \emph{i.e.}\ in
polynomial time w.r.t.\ the size of the integer and of the representation.
\smallskip

\paragraph*{Contributions.}

We 
introduce the notion of a spectrum for a CCQ and
show that connected and individual-free CCQs evaluated on $\ALCIF$ KBs always admit
well-behaved spectra, as those are subsets of $\ninfty$ closed under addition.
We then propose an effective representation of such spectra. 

This motivates a focus on cardinality queries, \emph{i.e.}\ Boolean atomic CCQs \cite{DBLP:conf/ijcai/BienvenuMT21}, that fall in the above class.
First,
we fully characterize
possible spectra shapes for concept cardinality queries on $\ALCIF$ KBs, showing that every subset of $\ninfty$ closed under addition is realizable.
We also study several sublogics of $\ALCIF$, 
extending $\EL$ and $\dllitecore$, for most presenting 
full characterizations.
For some, only simpler shapes, such as $\llbracket m, \infty \rrbracket$, are possible. 
For $\ELIFbot$, the Horn fragment of $\ALCIF$, we notably use variations of the cycle-reversion techniques introduced to tackle finite model reasoning in such DLs \cite{CKV90,rosati2008finite,GarciaLS14}.

We further study the data complexity of computing the proposed effective representations of spectra.
For 
many settings, such as concept cardinality queries on $\ALC$ KBs, we are able to establish $\FPwithlogcalls$-completeness of this problem.
Several of our upper bounds notably rely on existing results regarding DLs equipped with closed predicates.

Via connections with the concept cardinality case and refinements of the corresponding constructions, we also investigate the case of role cardinality queries.
This latter class of OMQs features challenging shapes of spectra already for $\EL_\bot$ KBs, and we prove that computing effective representations of those is $\FPwithlogcalls$-complete already for $\EL$ KBs.

This paper is to appear at AAAI 2025. The present version contains an appendix with full proofs.

\section{Preliminaries}
\label{section-prelims}

\begin{toappendix}

\begin{figure*}
	\centering

\tikzstyle{dl}=[fill=white, draw=none, shape=rectangle]
\tikzstyle{dlnodebehind}=[fill=white, draw=none, shape=rectangle, font={\small}]

\tikzstyle{behindedge}=[-, draw=black, fill=none, dashed]
	\begin{tikzpicture}
	\begin{pgfonlayer}{nodelayer}
		\node [style=dl] (58) at (16, 3) {$\ALCF$};
		\node [style=dl] (59) at (16, 2) {$\ALC$};
		\node [style=dl] (60) at (25, 1) {$\ELI$};
		\node [style=dl] (61) at (19, 1) {$\ELbot$};
		\node [style=dl] (63) at (22, 1) {$\ELF$};
		\node [style=dl] (64) at (22, 0) {$\EL$};
		\node [style=dl] (65) at (19, 2) {$\ELFbot$};
		\node [style=dl] (66) at (25, 2) {$\ELIF$};
		\node [style=dl] (67) at (28, 0) {$\dllitecore$};
		\node [style=dl] (68) at (28, 1) {$\dllitef$};
		\node [style=dl] (69) at (19, 3) {$\ALCI$};
		\node [style=dl] (70) at (19, 4) {$\ALCIF$};
		\node [style=dl] (71) at (22, 3) {$\ELIFbot$};
		\node [style=dl] (72) at (22, 2) {$\ELIbot$};
	\end{pgfonlayer}
	\begin{pgfonlayer}{edgelayer}
		\draw (59) to (58);
		\draw (64) to (61);
		\draw (64) to (63);
		\draw (64) to (60);
		\draw (60) to (66);
		\draw (63) to (66);
		\draw (67) to (68);
		\draw (61) to (65);
		\draw (65) to (63);
		\draw (58) to (70);
		\draw [style=behindedge] (59) to (69);
		\draw (61) to (59);
		\draw (65) to (58);
		\draw (70) to (71);
		\draw (71) to (66);
		\draw [style=behindedge] (72) to (60);
		\draw [style=behindedge] (69) to (72);
		\draw (65) to (71);
		\draw [style=behindedge] (72) to (61);
		\draw [style=behindedge] (71) to (72);
		\draw [bend right=15, looseness=0.75] (68) to (71);
		\draw [dash pattern=on 82.5pt off 400pt, bend right=15, looseness=0.75] (67) to (72);
		\draw [dashed, bend right=15, looseness=0.75] (67) to (72);
		\draw [style=behindedge] (70) to (69);
	\end{pgfonlayer}
\end{tikzpicture}
	\caption{Investigated description logics. An edge indicates that the lower DL is subsumed by the higher DL.}
	\label{figure:dl}
\end{figure*}
\end{toappendix}

With $\NN$ we denote the set of natural numbers $\NN = \{0,1,\dots\}$
and by $(\ninfty, +)$ the semigroup of natural numbers with infinity $\infty$ and the usual definition of addition $+$. 
In particular, $a+\infty = \infty+a = \infty$ for all elements $a \in \ninfty$.
We recall that every subsemigroup of $(\ninfty, +)$,
\emph{i.e.}\ every subset closed under addition, is
\emph{ultimately periodic} (see \emph{e.g.}\ \cite{finitely-generated}, Chapter~2, Proposition~4.1), which 
ensures that every subsemigroup of $\ninfty$ takes the following shape.
\begin{lemmarep}
    \label{lemma:ultimately-periodic}
    Let $V$ be a subsemigroup of $(\ninfty, +)$. Then $V = S \cup \{M + \alpha \cdot n \mid n \in \mathbb{N}\}$ where $S$ is a finite subset of $\ninfty$ and $M,\alpha \in \ninfty$.
\end{lemmarep}
\noindent If $\alpha = 1$, we write $S \cup \llbracket M, \infty \llbracket$ for $S \cup \{M + n \mid n \in \mathbb{N}\}$.

\begin{proof}
	To prove the lemma we show that for every set $V$ closed under addition the following are equivalent.
    \begin{itemize}
        \item[A)] $V$ is a subsemigroup of $(\ninfty, +)$.
        \item[B)] $V$ is finitely generated.
        \item[C)] $V = S \cup \{M + \alpha \cdot z \mid z \in \mathbb{N}\}$ for some finite subset $S \subseteq \ninfty$ and numbers $M,\alpha \in \ninfty$.
    \end{itemize}

    Before we prove the equivalence, let us recall some basic facts and notation.
    By $\mathrm{GCD}$ we denote the greatest common divisor,
    i.e. $\mathrm{GCD}(a_1, \dots, a_k)$ is the greatest natural number dividing all $a_i$s.
    A set of numbers is coprime if their $\mathrm{GCD}$ is $1$.

    A set $V$ is generated by a set $S$, denoted $V = \langle S \rangle$, if $V$ is the smallest closed under addition subset of $\ninfty$ that contains $S$.
    Set $V$ is finitely generated if it is generated by a finite set.
    If $S = \{a_1, \dots, a_k\}$ then we may write $\langle a_1, \dots, a_k \rangle$ instead of $\langle S\rangle$.

    We start by showing equivalence between $B)$ and $C)$.
    For the direction $B) {\implies} C)$ we first prove the following claim.

    \begin{itemize}
        \item[$\star$] Let $a_1, \dots, a_k \in \NN \setminus \{0\}$ be a set of positive coprime natural numbers.
        Then $\langle a_1, \dots, a_k \rangle  = S \cup \llbracket M, \infty \llbracket$ for some finite
        set $S \subseteq \NN$ and a number $M \in \NN$.
    \end{itemize}

    For $k=1$, $S_1 = \langle a_1 \rangle = \llbracket a_1, \infty \llbracket$.

    For $k=2$, $S_2 = \langle a_1,a_2 \rangle = S \cup \llbracket a_1\cdot a_2, \infty \llbracket$ for some finite set $S$.
    To prove it, consider set $G = \{i \cdot a_1 \mid i=0, \dots, a_2-1\}$.
    and fix $n \geq a_1\cdot a_2$.

    Since
    $a_1$ and $a_2$ are coprime, $i\cdot a_1 \equiv j\cdot a_1 \text{ mod } a_2$ if and only if $i=j$, for $0 \leq i,j < a_2$.
    Thus, $\{i \cdot a_1 \text{ modulo } a_2 \mid i=0,\dots a_2\} = \llbracket 0, a_2 \llbracket$
    and there is $0 \leq i < a_2$ such that $i\cdot a_1 \equiv n$ modulo $a_2$.
    Therefore, there is $m \geq 0$ such that $n = m\cdot a_2 + i \cdot a_1$, as $i \cdot a_1 < a_2 a_1 \leq n$.
    This proves that $n \in S_2$ and ends the case for $k=2$.

    Let us now assume that $k\geq 2$ and $a_1, \dots, a_{k+1} \in \NN$ is a set of natural numbers such that $\mathrm{GCD}(a_1, \dots, a_{k+1}) = 1$.
    Let $S_i = \langle a_1, \dots, a_i \rangle$.

    Since the $\mathrm{GCD}(a_1, \dots, a_{k+1}) = 1$, there is $a \in S_k$ such that $\mathrm{GCD}(a, a_{k+1}) = 1$.
    Hence,
    \[
    S_{k+1} = \langle a_1, \dots, a_{k+1} \rangle = \langle a_1, \dots, a_{k+1}, a \rangle \supseteq \langle a_{k+1}, a \rangle = S' \cup \llbracket a_{k+1}\cdot a, \infty \llbracket
    \]
    for some finite set $S'$.
    Therefore
    \[
    S_{k+1} = (S_{k+1} \cap \llbracket 0,a_{k+1}\cdot a \rrbracket) \cup \llbracket a_{k+1}\cdot a, \infty \llbracket,
    \]
    which ends the proof of the claim as $S_{k+1} \cap \llbracket 0,a_{k+1}\cdot a \rrbracket$ is a finite set.

    Now, for the direction $B) {\implies} C)$ let $a_1, \dots, a_k$ be a set of generating $V$ except for $\infty$ and $0$.
    Let $d= \mathrm{GCD}(a_1,\dots, a_k)$ and consider the set of numbers $a'_i = a_i/d$ for $i=1,\dots, k$.
    By $\star$, $\langle a'_1, \dots, a'_k\rangle = S' \cup \llbracket M', \infty \llbracket$ for some finite set $S'$ and some number $M' \in \NN$.
    Thus,
    \[
        \langle a_1, \dots, a_k\rangle = d\cdot S' \cup \{dM' + nd \mid n \in \NN\},
    \]
    \[
        \langle a_1, \dots, a_k, \infty \rangle = (d\cdot S' \cup \{\infty\}) \cup \{dM' + nd \mid n \in \NN\},
    \]
    \[
    \langle a_1, \dots, a_k, 0 \rangle = (d\cdot S' \cup \{0\}) \cup \{dM' + nd \mid n \in \NN\},
    \]
    \[
    \langle a_1, \dots, a_k, 0, \infty \rangle = (d\cdot S' \cup \{0,\infty\}) \cup \{dM' + nd \mid n \in \NN\}
    \]
    for $k\geq 1$, and
    \[
    \langle \infty \rangle = \{\infty\} \cup \{\infty + n \mid n \in \NN\},
    \]
    \[
    \langle 0 \rangle = \{0\} \cup \{0 + 0\cdot n \mid n \in \NN\},
    \]
    \[
    \langle \infty, 0 \rangle = \{0\} \cup \{\infty + n \mid n \in \NN\}
    \]
    for $k = 0$.
    This ends the proof of this implication.

    For the reverse implication $C) {\implies} B)$, let $V = S \cup \{M + \alpha \cdot z \mid z \in \mathbb{N}\}$.
    It is easy to verify that the set $S \cup \{ M, M+\alpha, \dots, M + M\alpha\}$ generates $V$.

    Now, to end the proof it is enough to show equivalence between $A)$ and $B)$.
    Implication $B) {\implies} A)$ follows directly from definitions.
    For the reverse direction let $d = \mathrm{GCD}(V)$. As before, consider the set $V' = \{v/d \mid v \in V\}$.
    Since $\mathrm{GCD}(V') = 1$, there is a finite subset $V'' \subseteq V'$ such that $\mathrm{GCD}(V'')=1$.
    Since $V'$ is closed under addition $\langle V'' \rangle \subseteq V'$. Moreover, by $\star$,
    $\langle V'' \rangle = S'' \cup \llbracket M', \infty \llbracket$ for some finite set $S''$ and number $M'$.
    Therefore, the set $Z = V' \setminus \langle V'' \rangle \subseteq \llbracket 0, M' \rrbracket$ is finite.
    It easy to check that $V' = \langle V'' \cup Z\rangle$ which implies that $V$ is finitely generated.

\end{proof}

\begin{toappendix}

  Since every subsemigroup of $\langle \ninfty, + \rangle$ is finitely generated, one might consider to represent those sets by their generators.
  Unfortunately, such representation would be not effective in the sense proposed in the paper. On the other hand, representing sets
  as the sum of a finite set $S$ and a periodic set $\{m + \alpha n \mid n \in \NN\}$ is effective.
\begin{lemma}
Let $G \subseteq \langle \ninfty, + \rangle$ be a subsemigroup.
\begin{itemize}
    \item $G =S \cup \{m + \alpha n \mid n \in \NN\}$ is an effective representation.
    \item $G = \langle g_1, \dots, g_n\rangle$ is \textbf{not} an effective representation, unless $\P = \NP$.
\end{itemize}
\end{lemma}

For the first bullet observe that we the representation is finite, determined by a finite set and two natural numbers, and independent from the ontology languages defining KB.
Moreover, the membership test $v \in \mathsf{Sp}_{\kb}(q)$ can be performed in linear time by first searching $S$ for a possible match and if that fails then deciding if $(v-m)\%\alpha=0$,
which can be done in polynomial time.

For the second bullet, recall that the \emph{Subset Sum} problem, i.e. given multiset $T$ and a number $v$ decide if there is a subset $V \subseteq T$ with sum $v$ is $\NP$-complete.
Hence, the membership problem is $\NP$-hard when the spectrum is represented by a set of generators.
\end{toappendix}

\subsection{$\ALCIF$ and other description logics}

Let $\cnames$, $\rnames$, and $\inames$ be countably infinite and mutually disjoint sets of \emph{concept names}, \emph{role names}, and \emph{individual names}.
An \emph{inverse role} takes the form $\rstyle{r}^-$ with $\rstyle{r}$ a role name, and a \emph{role} is a role name or an inverse role.
We denote $\posroles$ the set of roles.
If $\rstyle{r}=\rstyle{s}^-$ is an inverse role, then $\rstyle{r}^-$ denotes $\rstyle{s}$.
An \emph{$\ALCI$ concept} is built according to the rule
$
\cstyle{C},\cstyle{D} ::= \top \mid \cstyle{A} \mid \neg \cstyle{C} \mid \cstyle{C} \sqcap \cstyle{D} \mid \exists \rstyle{r} . \cstyle{D}
$
where $\cstyle{A} \in \cnames$ and $\rstyle{r} \in \posroles$.
We use $\bot$ as an abbreviation for $\neg \top$, $\cstyle{C} \sqcup \cstyle{D}$ for $\neg(\neg \cstyle{C} \sqcap \neg \cstyle{D})$, $\forall \rstyle{r} . \cstyle{C}$ for $\neg \exists \rstyle{r}. \neg \cstyle{C}$ and $\exists \rstyle{r}$ for $\exists \rstyle{r}.\top$.

An \emph{$\ALCIF$ TBox} is a finite set of \emph{concept inclusions (CIs)} $\cstyle{C} \sqsubseteq \cstyle{D}$ and of \emph{functionality restrictions} $\cstyle{C} \sqsubseteq\ \leq 1\ \rstyle{r}.\cstyle{D}$ where $\cstyle{C},\cstyle{D}$ are $\ALCI$ concepts and $\rstyle{r}$ is a
role.
An \emph{ABox} is a finite set of \emph{concept assertions} $\cstyle{A}(\istyle{a})$ and \emph{role assertions} $\rstyle{r}(\istyle{a},\istyle{b})$
where $\cstyle{A} \in \cnames$, $\rstyle{r} \in \rnames$ and $\istyle{a},\istyle{b} \in \inames$.
The set of individual names used in the ABox $\abox$ is denoted $\individuals(\abox)$.
An \emph{$\ALCIF$ knowledge base (KB)} takes the form $\kb =(\tbox,\abox)$ with $\tbox$ an $\ALCIF$ TBox and $\abox$ an ABox.

We also investigate restrictions of $\ALCIF$.
Each fragment is obtained by disallowing concepts, roles constructors, or axiom shapes in the standard way.
An \emph{$\EL$ concept} is an $\ALCI$ concept that uses neither negation nor inverse roles and an $\EL$ TBox only supports CIs of $\EL$ concepts.
Allowing inverse roles is indicated by $\mathcal{I}$; \textit{concept disjointness} axioms of shape $\cstyle{C} \sqcap \cstyle{D} \sqsubseteq \bot$
	 by subscript $\bot$; unrestricted use of negation by replacing $\EL$ with $\ALC$; and functionality restrictions by $\mathcal{F}$.
A $\dllite$ concept has shape $\cstyle{A} \mid \exists \rstyle{r}$ for $\cstyle{A} \in \cnames$ and $\rstyle{r} \in \posroles$.
A $\dllitecore$ TBox only supports CIs and CDs of $\dllite$ concepts.
$\dllitef$ extends $\dllitecore$ with unqualified functionality restrictions $\top \sqsubseteq\ \leq 1\ \rstyle{r}.\top$
\cite{DBLP:conf/kr/CalvaneseGLLR06}.

The semantics is defined in terms of interpretations $\I =(\Delta^\I,\cdot^\I)$ in the standard way.
$\Delta^\I$ is a non-empty \emph{domain} and $\cdot^\I$ an \emph{interpretation function}.
We refer to \cite{introduction-to-dl} for details.
An interpretation $\I$ satisfies a CI $\cstyle{C} \sqsubseteq \cstyle{D}$ if $\cstyle{C}^\I \subseteq \cstyle{D}^\I$ and a functionality restriction $\cstyle{C} \sqsubseteq\ \leq 1\ \rstyle{r}.\cstyle{D}$ if for each $d \in \cstyle{C}^\I$, there is at most one $e \in \cstyle{D}^\I$ such that $(d, e) \in \rstyle{r}^\I$.
It satisfies an assertion $\cstyle{A}(\istyle{a})$ if $\istyle{a} \in \cstyle{A}^\I$ and $\rstyle{r}(\istyle{a},\istyle{b})$ if $(\istyle{a},\istyle{b}) \in \rstyle{r}^\I$.
We make the \emph{standard names assumption}: in every interpretation $\I$, we assume $\istyle{a}^\I = \istyle{a}$ for every $\istyle{a} \in \individuals(\abox)$.
An interpretation $\I$ is a \emph{model} of a TBox $\tbox$, denoted $\I \models \tbox$, if it satisfies all its axioms.
Models of ABoxes and KBs are defined likewise.
A TBox $\tbox$ (resp.\ a KB $\kb$) entails an assertion, a CI or a functionality restriction $\alpha$, denoted $\tbox \models \alpha$ (resp.\ $\kb \models \alpha$) if all its models satisfy $\alpha$.

\subsection{Spectra of Counting Queries}

We consider two countably infinite and mutually disjoint sets: a set
of \emph{variables} and a set
of \emph{counting variables}.
A \emph{counting conjunctive query (CCQ)} takes the form $q(\bar x) = \exists \bar y\ \exists \bar z\ \psi(\bar x, \bar y, \bar z)$, where $\bar x$ and $\bar y$ are tuples of distinct variables, $\bar z$ is a tuple of distinct counting variables and $\psi$ is a conjunction of concept and role atoms whose terms are drawn from $\inames \cup \bar x \cup \bar y \cup \bar z$.
We call $\bar x$ (resp. $\bar y$, resp. $\bar z$) the \emph{answer} (resp.\ \emph{existential}, resp.\ \emph{counting}) variables of $\query$.
A CCQ is \emph{Boolean} if $\bar x = \emptyset$.

For a tuple $\bar a \in \inames^{\sizeof{\bar a}}$ of individuals and a model $\I$ of a KB $\kb$, we define $\# q(\bar a)^{\I}$ the answer of $q(\bar a)$ on $\I$ as:
\begin{center}
	\(
	\# \{ \match_{|\bar z} \mid \match : q \rightarrow \I \textrm{ homomorphism s.t.}\ \match(\bar x) = \bar a
	\}.
	\)
\end{center}
The spectrum of $q(\bar a)$ on $\kb$ is further defined as:
\begin{center}
\(
\spectrum{\kb}{q(\bar a)} := \{ \# q(\bar a)^{\I} \mid \I \models \kb \}.
\)
\end{center}
Note that $\spectrum{\kb}{q(\bar a)} = \spectrum{\kb}{q[\bar a/\bar x](\bar a_\emptyset)}$, where $q[\bar a/\bar x]$ denotes the Boolean CCQ obtained from $q$ by substituting every answer variable $x_i \in \bar x$ by the corresponding $a_i \in \bar a$, and $\bar a_\emptyset$ the empty tuple.
We thus focus w.l.o.g.\ on Boolean CCQs $q(\bar x_\emptyset)$, denoted simply $q$ for readability.

	The main interest of this paper is to compute representations of spectra that are effective in the sense of Points~(i)--(iii) in the introduction.
While we do not know whether all spectra can be effectively represented, we identify a class of OMQs, namely connected and individual-free CCQs on $\ALCIF$ KBs, whose spectra admit such a representation.
We recall that $q$ is \emph{connected} if its Gaifman graph is, and is \emph{individual-free} if none of its atom involves a term from $\inames$.

\begin{lemmarep}
	\label{lemma-closure-under-addition}
	If $\kb$ is an $\ALCIF$ KB and $q$ a connected and individual-free CCQ, then $\spectrum{\kb}{q}$ is closed under addition.
	Furthermore, if $q$ is satisfiable w.r.t.\ $\kb$, then $\infty \in \spectrum{\kb}{q}$.
\end{lemmarep}

\begin{proof}
	Assume $q$ is connected and individual-free.
	If $\spectrum{\kb}{q} = \emptyset$, then it is clearly closed under addition.
	Otherwise consider $m, n \in \spectrum{\kb}{q}$.
	By definition, there exist models $\I_m$ and $\I_n$ of $\kb$ such that $\# q^{\I_m} = m$ and $\# q^{\I_n} = n$.
	We construct a model $\I_m \uplus \I_n$ of $\kb$ such that $\# q^{\I_m \uplus \I_n} = m + n$, proving $m +n \in \spectrum{\kb}{q}$ as desired.
	This is simply achieved by putting models $\I_m$ and $\I_n$ aside and choosing to interpret individuals in the first one.
	More formally, for every element $d \in \domainof{\I_n}$, we consider a dedicated fresh element $c_d \notin \domainof{\I_m}$.
	We let $\domainof{\I_m \uplus \I_n} := \domainof{\I_m} \cup \{ c_d \mid d \in \domainof{\I_n} \}$ and for every individual $\istyle{a}$, every concept name $\cstyle{A}$ and role name $\rstyle{p}$, we further set:
	\begin{align*}
		\istyle{a}^{\I_m \uplus \I_n} & := \istyle{a}^{\I_m}
		\\
		\cstyle{A}^{\I_m \uplus \I_n} & := \cstyle{A}^{\I_m} \cup \{ c_d \mid d \in \cstyle{A}^{\I_n} \}
		\\
		\rstyle{p}^{\I_m \uplus \I_n} & := \rstyle{p}^{\I_m} \cup \{ (c_d, c_e \mid (d, e) \in \rstyle{p}^{\I_n} \}.
	\end{align*}
	It is easily verified that $\I_m \uplus \I_n$ is a model of $\kb$.
	Since $q$ is individual free, it is clear that it still maps $n$ times in $\I_m \uplus \I_n$ restricted to domain $\{ c_d \mid d \in \domainof{\I_n} \}$.
	In addition with $\I_m$ being contained in $\I_m \uplus \I_n$, and $\domainof{\I_m}$ being disjoint from $\{ c_d \mid d \in \domainof{\I_n} \}$, we clearly have $\# q^{\I_m \uplus \I_n} \geq m + n$.
	Finally, there cannot be any additional match for $q$ as domains $\domainof{\I_m}$ and $\{ c_d \mid d \in \domainof{\I_n} \}$ are disconnected while $q$ is connected, proving $\# q^{\I_m \uplus \I_n} = m + n$ as desired.

	Assume that, additionally, $q$ is satisfiable w.r.t.\ $\kb$.
	Therefore, there exists a model $\I_1$ of $\kb$ in which $\# q^{\I_1} \geq 1$.
	We now form $\I_{k+1} := \I_k \uplus \I_1$ for every $k \geq 1$.
	As $\I_k \subseteq \I_{k+1}$, we can safely define $\I_\infty := \biguplus_{k \geq 1} \I_k$ to obtain a model of $\kb$ such that $\# q^{\I_\infty} = \infty$, proving $\infty \in \spectrum{\kb}{q}$.
\end{proof}

In other words, spectra of connected and individual-free CCQs are subsemigroups of $\ninfty$ and, by Lemma~\ref{lemma:ultimately-periodic}, are of form
$S \cup \{M + \alpha \cdot n \mid n \in \mathbb{N}\}$.
Thus, for this class, computing representations of spectra can be defined as follows.
		\begin{problem}
			\label{problem:compute-spectrum}
			\noindent
			Given a KB $\kb$ and a CCQ $q$, compute a special value $\emptyset$ if $\spectrum{\kb}{q} = \emptyset$, otherwise a finite set $S \subseteq \ninfty$ and numbers $M,\alpha \in \ninfty$ s.t.\ $\spectrum{\kb}{q} = S \cup \{M + \alpha \cdot n \mid n \in \mathbb{N}\}$.
		\end{problem}

	It can be verified that such representations as triples $(S, M, \alpha)$ comply with Points~(i)--(iii) from the introduction and are, in this sense, effective.

\begin{remark}
	\label{remark:trivial-sets}
	Notice $\spectrum{\kb}{q} = \emptyset$ iff $\kb$ is unsatisfiable;
	and, likewise, $\spectrum{\kb}{q} = \{ 0 \}$ iff $\kb$ is satisfiable but $q$ is unsatisfiable w.r.t.\ $\kb$.
	In particular, if $\kb$ is an $\ELI$ KB, then $\spectrum{\kb}{q}$ cannot be $\emptyset$ nor $\{ 0 \}$.
	Similarly, if $\kb$ is an $\ELIF$ KB, then $\spectrum{\kb}{q}$ cannot be $\{ 0 \}$.
	An effective representation of $\{0\}$ in the sense of Problem~1 is $(S, M, \alpha) = (\emptyset, 0, 0)$.
\end{remark}

A subset of $\ninfty$ is \emph{trivial} if it is either $\emptyset$ or $\{ 0 \}$.

We
highlight
that Example~\ref{example:squares} illustrates a situation in which the individual-freeness condition is not met.

\section{Spectrum of a concept cardinality query}
\label{section:concepts}

\newcommand{\nos}{\raisebox{.2ex}{$\star$}}
\newcommand{\nope}{$\cdot$}

\begin{toappendix}

	The results of Section~\ref{section:concepts} regarding the possible shapes of spectra of concept cardinality queries are summarized in Table~\ref{table:concepts}.

\begin{table*}
	\centering
	\begin{tabular}{ccccccccc}
		$\mathcal{L}$
		&
		&
		$\llbracket m, \infty \rrbracket$
		&
		$\emptyset$
		&
		$\{ 0 \}$
		&
		$\{ 0 \} \cup \llbracket m, \infty \rrbracket$
		&
		$\{ \infty \}$
		&
		$\{ 0, \infty \}$
		&
		$V \cup \{ \infty \}$
		\\ \midrule
		$\ALCIF$
		& \nos		& \checkmark & \checkmark & \checkmark & \checkmark & \checkmark & \checkmark & \checkmark
		\\
		$\ELIFbot$
		& \nope	& \checkmark & \checkmark & \checkmark & \checkmark & \checkmark & \checkmark & \nope
		\\
		$\dllitef$
		& \nos 		& \checkmark & \checkmark & \checkmark & \checkmark & \checkmark & \checkmark  & \nope
		\\
		$\dllitecore$, $\ELbot$, $\ALCI$, $\ALCF$
		& \nos 		& \checkmark & \checkmark & \checkmark & \checkmark & \nope & \nope & \nope
		\\
		$\ELIF$
		& \nos 		& \checkmark & \checkmark & \nope & \nope & \checkmark & \nope & \nope
		\\
		$\ELF$
		& \nos 		& \checkmark & \checkmark & \nope & \nope & \nope & \nope & \nope
		\\
		$\EL$, $\ELI$
		& \nos 		& \checkmark & \nope & \nope & \nope & \nope & \nope & \nope
		\\
		\bottomrule
	\end{tabular}
	\caption{$\mathcal{L}$-concept realizable sets, where $m \in \NN$ and $V$ is any subsemigroup of $\ninfty$.~\nos\ indicates no other shape is possible.}
	\label{table:concepts}
\end{table*}
\end{toappendix}

In this section, we focus on concept cardinality queries $q_\cstyle{C} := \exists z\ \cstyle{C}(z)$, where $\cstyle{C} \in \cnames$ is a concept name and $z$ a counting variable.
Computing the spectrum of $q_\cstyle{C}$ over a KB $\kb$ thus corresponds to the natural task of deciding the possible values of $\sizeof{\cstyle{C}^\I}$ across the models $\I$ of $\kb$.
Every concept cardinality query satisfies preconditions of Lemma~\ref{lemma-closure-under-addition} and thus
its spectrum can be represented as in Problem~\ref{problem:compute-spectrum}.
Conversely, one can ask which sets are spectra of such queries.
We say a set $V$ is \emph{$\mathcal{L}$-concept realizable} if there is a concept $\cstyle{C}$ and a $\mathcal{L}$ KB $\kb$ s.t.\ $\spectrum{\kb}{q_\cstyle{C}} = V$.
We begin with $\ALCIF$ KBs and prove they can realize
all subsemigroups of $\ninfty$.

\begin{theoremrep}
	\label{theorem-shape-concept-alcif}
	A non-trivial subset of $\ninfty$ is $\ALCIF$-concept realizable iff it is a subsemigroup of $\ninfty$ containing $\infty$.
\end{theoremrep}

\begin{proof}
	Let $V$ be a subsemigroup of $\ninfty$ generated by $n_1,\dots, n_k$.
    We will construct a KB $\kb = (\tbox,\abox)$ such that $\spectrum{\kb}{q_\cstyle{C}} = V$.

    First we resolve some corner cases.
    For $V = \{0\}$ take $\abox = \emptyset$ and $\tbox = \{C \sqsubseteq \bot\}$.
    For $V = \{\infty\}$ see Example~\ref{example-concept-elif-infinity}.
    For $V = \{0, \infty\}$ see Example~\ref{example-dllitef-zero-infinity}.

    From now on, we assume that there is $0<a<\infty$ in $V$.

    We first enforce that every model consists of cycles and infinite linear orders by
    requesting the following.

    \[
    \top \sqsubseteq\ \leq 1\ \rstyle{r^-}.\top, \top \sqsubseteq\ \leq 1\ \rstyle{r}.\top
    \]

    Furthermore, we assure that for every such cycle either all elements realise $C$
    or none.

    \[
    C \sqsubseteq \exists r.C,\ C \sqsubseteq \exists r^-.C
    \]

    Finally, we need to ensure that every cycle consists of a number of elements that is divisible by $n_j$, where $n_j$
    is one of the generators. We do it using concepts to encode a binary counter on elements of the cycle. Every generator $n_j$ will use its own set of concept names, distinguished by superscript $j$, for $1 \leq j \leq k$, to mark cycles.

    \[
    N^j \sqsubseteq \exists r.N^j,\ N^j \sqsubseteq \exists r^-.N^j
    \]

    We first ensure that $C$ appears only on marked cycles.

    \[C \equiv \bigsqcup_{1 \leq i \leq l} N^j\]
    where the notation $\cstyle{D} \equiv \cstyle{E}$, for any concepts $\cstyle{D}$ and $\cstyle{E}$, stands for the two concept inclusions $\cstyle{D} \sqsubseteq \cstyle{E}$ and $\cstyle{E} \sqsubseteq \cstyle{D}$.
    Now we encode the counter.
    Let $I^j$ be the set of bits used by $n_j$ and $l_j = \max (I^j)$ be the greatest bit in $n_j$.
    For example, with $n_j= 10 = 1101_b$,
    $I^J = \{0,1,3\}$ and $l_j = 3$. For every generator $n_j$
    we add rule
    \[
    E^j \equiv N^j \sqcap \bigsqcap_{i \in I^j} B^j_i \sqcap \bigsqcap_{i \notin I^j} \lnot B^j_i
    \]
    to mark the end of the calculation and
    \[
    S^j \equiv N^j \sqcap \bigsqcap_{0 \leq i \leq l_j} \lnot B_i
    \]
    to mark a beginning of it. End and start are two consecutive positions
    $E^j \equiv \exists r.S^j$.

    The concepts $B^j_i$ are used to represent binary encodings of the computation simulating binary counter from $0$ to $n_j$,
    with $B^j_i$ representing the $i$th bit of a number on a cycle belonging to the computation for $n_j$.

    Concepts $C^j_i$ are used to encode incrementation. They mark
    the block of least important bits set to $1$.
    \[
    B^j_0 \equiv C^j_0
    \]

    \[
    B^j_{i+1} \sqcap C^j_{i} \equiv C^j_{i+1}
    \]
    for $0 \leq i \leq l$.

    The increment rules apply to all non-end positions and
    enforce clearing the block of least important bits
    \[
    \lnot E^j \sqcap C^j_i \equiv \exists r. \lnot B^j_i
    \]
    incrementing the next bit
    \[
    \lnot E^j \sqcap  C^j_i \equiv \exists r. B^j_i
    \]
    and copying the rest
    \[
    \lnot E^j \sqcap  \lnot C^j_i \sqcap B^j_{i+1} \equiv \exists r. B^j_{i+1}
    \]
    \[
    \lnot E^j \sqcap  \lnot C^j_i \sqcap \lnot B^j_{i+1} \equiv \exists r. \lnot B^j_{i+1}
    \]

    Observe that this computation guarantees that the distance between two $S^j$ starting positions is divisible by $n_j$, thus enforcing that any finite cycle with containing a starting position $S^j$ has a number of elements divisible by $n_j$.

    Now, to guarantee that every cycle has a starting position,
    we enforce that for every $n_j$, for every model $\mathcal{I}$, and for every individual $c$ the set of realised $B_i^j$ concepts, i.e. $\{B_i^j \mid c \in (B_i^j)^{\mathcal{I}}\}$, encodes a number smaller than $n_j$. This, with incrementation, will enforce the existence of an end position and, thus, the existence of a starting position.

    To do so let $G^j$ be a concept that describes individuals with counter bigger than $n_j$.

    \[
    G^j \equiv \bigsqcup_{1 \leq i \leq l_j} \big( B^j_i \sqcap \bigsqcap_{z \in I^j, z > i} B^j_z \big)
    \]

    Now we simply demand that for every individual the counter is not too big.
    \[
    \bot \equiv \bigsqcup_{j = 1}^{k} G^j
    \]

    Observe that since for every generator $n_j$ we use no more than $||n_j||$ symbols and rules, the KB is of size linear with respect to the set of all generators.
\end{proof}

Notice Lemma~\ref{lemma-closure-under-addition} already ensures that being a subsemigroup of $\ninfty$ containing $\infty$ is necessary.
The other direction is a generalization of the following example that illustrates how to realize a shape of spectrum with $\alpha = 2$.

\begin{example}
Consider the $\ALCIF$ TBox $\tbox
= \{ ~ \top \sqsubseteq \cstyle{C},
$
$
\cstyle{A} \sqsubseteq \exists \rstyle{r}.\lnot \cstyle{A},
~
\lnot \cstyle{A} \sqsubseteq \exists \rstyle{r}.\cstyle{A},
~
\top \sqsubseteq\ \leq 1\ \rstyle{r}.\top,
~
\top \sqsubseteq\ \leq 1\ \rstyle{r^-}.\top
~ \}$.
Then, $\spectrum{(\tbox, \emptyset)}{q_\cstyle{C}} = 2 \mathbb{N} \cup \{ \infty \}$.

\end{example}

Notice that to achieve the non-trivial period of $\alpha = 2$ in the spectrum $2 \NN \cup \{ \infty \}$,
we rely on a role that is both functional and inverse functional.
In fact, limiting one of these two features forces a trivial periodic behavior, \emph{i.e.}\ $\alpha = 1$, and further allows for easier computation of the spectra.

\subsection{Limiting inverse functional roles}
\label{subsection:concept-easy-cases}

\begin{toappendix}
	\subsection*{Proofs for Section~\ref{subsection:concept-easy-cases} (Limiting inverse functional roles)}
\end{toappendix}

We now move towards $\ALCF$ and $\ALCI$, in which the functionality of an inverse role cannot be expressed.
As a consequence, spectra of a concept cardinality query over such KBs enjoy the following well-behaved shapes.

\begin{theoremrep}
	\label{theorem-shape-concept-alci-alcf}
	A non-trivial subset of $\ninfty$ is $\ALCI$- (resp.\ $\ALCF$-) concept realizable iff it has shape $\{ 0 \} \cup \llbracket M, \infty \rrbracket$ or shape $\llbracket M, \infty \rrbracket$ for some $M \in \NN$.
\end{theoremrep}

	The main ingredient for the `only-if' part of Theorem~\ref{theorem-shape-concept-alci-alcf} is a technique
    that 
    extends any model $\I$ in which $\sizeof{\cstyle{C}^\I} > 0$ into a model $\J$ with $\sizeof{\cstyle{C}^\J} = \sizeof{\cstyle{C}^\I} + 1$, as used in \cite{BaaderBR:ECAI2020} for $\ALCF$ KBs.
Conversely, it is not difficult to find KBs, already in $\dllitecore$ or $\ELbot$, that realize these shapes notably relying on CD axioms for the shape $\{ 0 \} \cup \llbracket M, \infty \rrbracket$.

\begin{toappendix}

	We start with the easy direction of Theorem~\ref{theorem-shape-concept-alci-alcf}, that is the `if' direction.

	By mean of the following example, we prove that the proposed shapes are $\EL_\bot$-concept realized and $\dllitecore$-concept realizable, and thus clearly $\ALCI$- (resp.\ $\ALCF$-) concept realizable.

	\begin{example}
		\label{example:concept-one-from-min}
		Let $m \in \NN$.
		Consider the $\EL_\bot$ TBox $\tbox$ that contains the following axioms:
		\[
		\begin{array}{c}
			\cstyle{C} \sqsubseteq \exists \rstyle{r}.\cstyle{A_k}
			\qquad
			\cstyle{A_k} \sqsubseteq \cstyle{C}
			\qquad
			(1 \leq k \leq M)
			\smallskip \\
			\cstyle{A_i} \sqcap \cstyle{A_j} \sqsubseteq \bot
			\qquad
			(1 \leq i < j \leq M)
		\end{array}
		\]
		It is immediate to verify that $\spectrum{(\tbox, \emptyset)}{q_\cstyle{C}} = \{ 0 \} \cup \llbracket M, \infty \rrbracket$ while $\spectrum{(\tbox, \{ \cstyle{C}(\istyle{a}) \})}{q_\cstyle{C}} = \llbracket M, \infty \rrbracket$.
		By replacing each axiom $\cstyle{C} \sqsubseteq \exists \rstyle{r}.\cstyle{A_k}$ in $\tbox$ by the two axioms $\cstyle{C} \sqsubseteq \exists \rstyle{r_k}$ and $\exists \rstyle{r_k^-} \sqsubseteq \cstyle{A_k}$, we obtain the same result for a $\dllitecore$ KB.
	\end{example}

	For the `only-if' direction of Theorem~\ref{theorem-shape-concept-alci-alcf}, we proceed in two steps which are respectively Lemmas~\ref{lemma:fmp-alci-alcf} and \ref{lemma-concept-alci-alcf-plus-one} below.
	First, we prove a sort finite model property for the extension of a concept of interest $\cstyle{C}$: if this concept admits a non-empty interpretation, then it admits a non-empty interpretation that is also finite.
	In a second step, we prove that if there exists a model in which the interpretation of $\cstyle{C}$ is non-empty and finite, then there exists a model with exactly one additional instance of $\cstyle{C}$.
	It is easy to see that those two ingredient are sufficient to conclude the proof: assume a non-trivial subset $S$ of $\ninfty$ to be $\ALCI$- (resp.\ $\ALCF$-) concept realizable.
	Let $\cstyle{C}$ a concept name and $\kb$ an $\ALCI$ (resp.\ $\ALCF$) KB such that $\spectrum{\kb}{q_\cstyle{C}} = S$.
	Since $S$ is non-trivial, $q_\cstyle{C}$ is satisfiable w.r.t.\ $\kb$ and thus, by Lemma~\ref{lemma-closure-under-addition}, we obtain $\infty \in \spectrum{\kb}{q_\cstyle{C}}$.
	Therefore, from the first step (Lemma~\ref{lemma:fmp-alci-alcf}), there exists a model $\I$ in which the interpretation $\cstyle{C}^\I$ is of size $0 < N < \infty$.
	In particular, $N \in \spectrum{\kb}{q_\cstyle{C}}$.
	We let $M$ to be the smallest such $N$.
	Since $M$ is non-zero, we apply the second step (Lemma~\ref{lemma-concept-alci-alcf-plus-one}) starting from a model $\I_M$ in which $\sizeof{\cstyle{C}^\I} = M$.
	We obtain a model $\I_{M+1}$ in which $\sizeof{\cstyle{C}^{\I_{M+1}}} = M + 1$.
	Iterating this process, we construct models $\I_K$ in which $\sizeof{\cstyle{C}^{\I_K}} = K$ for every $K > M$.
	Recalling that $\infty \in \spectrum{\kb}{q_\cstyle{C}}$, we thus have $\llbracket M, \infty \rrbracket \subseteq \spectrum{\kb}{q_\cstyle{C}}$.
	By minimality of $M$, we know that $\llbracket 1, M-1 \rrbracket \cap \spectrum{\kb}{q_\cstyle{C}} = \emptyset$.
	Thus $\spectrum{\kb}{q_\cstyle{C}} = \{ 0 \} \cup \llbracket M, \infty \rrbracket$ or $\llbracket M, \infty \rrbracket$ depending on whether $0 \in \spectrum{\kb}{q_\cstyle{C}}$ or not, which are exactly the desired shapes.

\begin{lemma}
	\label{lemma:fmp-alci-alcf}
	Let $\kb = (\tbox, \abox)$ be an $\ALCI$ (resp.\ $\ALCF$) KB and $\cstyle{C} \in \cnames$.
	If there exists a model $\I$ of $\kb$ with ${\cstyle{C}^\I} \neq \emptyset$,
	then there exists a model $\J$ of $\kb$ with ${\cstyle{C}^{\J}} \neq \emptyset$ and $\sizeof{\domainof{\J}} \leq \sizeof{\individuals(\abox)} + 4^{\sizeof{\tbox}}$.
	\emph{A fortiori}, $1 \leq \sizeof{\cstyle{C}^{\J}} \leq \sizeof{\individuals(\abox)} + 4^{\sizeof{\tbox}}$.
\end{lemma}

\begin{proof}
	\newcommand{\concepts}{\mathfrak{C}}
	\newcommand{\witnesses}{\mathsf{Wit}}
	Let $\kb = (\tbox, \abox)$ be an $\ALCI$ or an $\ALCF$ KB and $\cstyle{C} \in \cnames$.
	Let $\I$ be a model of $\kb$ with ${\cstyle{C}^\I} \neq \emptyset$.
	The construction of model $\J$ will depend on whether $\kb$ is an $\ALCI$ or an $\ALCF$ KB, but we first give some preliminaries that are useful for both cases.

	We denote $\concepts$ the set of concepts occurring in the TBox $\tbox$, closed under subconcepts and negation.
	Up to introducing the axiom $\top \sqsubseteq \cstyle{C}$ in the TBox, we can safely assume that $\cstyle{C} \in \concepts$.
	For instance, if $\tbox$ only contains the axiom $\top \sqsubseteq \exists \rstyle{r}.(\cstyle{B} \sqcap \lnot \cstyle{C})$, then we have:
	\[
	\concepts = \{ \top, \bot, \cstyle{B}, \lnot \cstyle{B}, \cstyle{C}, \lnot \cstyle{C}, \cstyle{B} \sqcap \lnot \cstyle{C}, \lnot \cstyle{B} \sqcup \cstyle{C}, \exists \rstyle{r}.(\cstyle{B} \sqcap \lnot \cstyle{C}), \forall \rstyle{r}.(\lnot \cstyle{B} \sqcup \cstyle{C}) \}.
	\]
	Notice that $\sizeof{\concepts} \leq 2\sizeof{\tbox}$ as each symbol in $\tbox$ can only be the root of one (sub)concept occurring in $\tbox$, and closure under negation further yields the factor $2$.

	A type $t$ is a subset of $\concepts$ s.t.\ there exists an element $e \in \domainof{\I}$ with $\type_\I(e) = t$, where:
	\[
	\type_\I(e) := \{ \cstyle{D} \mid \cstyle{D} \in \concepts, e \in \cstyle{D}^\I \}.
	\]
	Note that this notion of types does not coincide with the types used in Section~\ref{subsection:concept-elifbot}.
	Whenever convenient, we see $t$ as a concept itself, being $\bigsqcap_{\cstyle{D} \in t} \cstyle{D}$.
	We define $\types(\I) := \{ \type_\I(e) \mid e \in \domainof{\I} \}$.
	From the previous remark on the size of $\concepts$, it follows that $\sizeof{\types(\I)} \leq 4^\sizeof{\tbox}$.
	For each type $t \in \types(\I)$, we choose a witness $w_t \in \domainof{\I}$ s.t.\ $\type_\I(w_t) = t$, and we let $\witnesses := \{ w_t \mid t \in \types(\I) \}$ be the set of all those witnesses.

	The interpretation $\J$ that we construct has domain $\domainof{\J} := \individuals(\abox) \cup \witnesses$ and interprets each concept name $\cstyle{A}$ as follows:
	\[
	\cstyle{A}^\J :=  \cstyle{A}^\I \cap \domainof{\J}.
	\]
	The interpretation of role names differs depending on whether we consider an $\ALCI$ KB or an $\ALCF$ KB.
	If $\kb$ is in $\ALCI$, we interpret each role name $\rstyle{r}$ as follows:
	\[
		\rstyle{r}^\J :=
		(\rstyle{r}^\I \cap (\domainof{\J} \times \domainof{\J}))
		\cup
		\{ (d, w_t) \mid d \in (\exists \rstyle{r}.t)^\I \} \cup \{ (w_t, d) \mid d \in (\exists \rstyle{r^-}.t)^\I \}.
	\]
	Otherwise, if $\kb$ is in $\ALCF$, we interpret each role name $\rstyle{r}$ as follows:
	\[
		\rstyle{r}^\J :=
		(\rstyle{r}^\I \cap (\domainof{\J} \times \domainof{\J}))
		\cup
		\{ (d, w_t) \mid d \in (\exists \rstyle{r}.t)^\I \cap \domainof{\J} \text{ and } \forall e \in t^\I \cap \domainof{\J}, (d, e) \notin \rstyle{r}^\I \}.
	\]

	We now prove that for all $\cstyle{D} \in \concepts$, for all $d \in \domainof{\J}$, $d \in \cstyle{D}^\J$ iff $d \in \cstyle{D}^\I$ ($\dagger$).
	We proceed by induction on $\cstyle{D}$.
	The difference between $\ALCI$ and $\ALCF$ is in the case $\cstyle{D} = \exists \rstyle{r}.\cstyle{E}$.
	\begin{itemize}[leftmargin=1cm]
		\item[$\cstyle{D} = \top$.]
			By definition, for all $d \in \domainof{\J}$, we have $d \in \top^\J$.
			Since $\domainof{\J} \subseteq \domainof{\I}$, we have $d \in \top^\I$ and thus the equivalence holds (both sides being always true).
			\smallskip

		\item[$\cstyle{D} \in \cnames$.]
			It directly follows from the definition of the interpretation of concept names in $\J$.
			\smallskip

		\item[$\cstyle{D} = \cstyle{D_1} \sqcap \cstyle{D_2}$.]
			Let $d \in \domainof{\J}$.
			From the semantics of the conjunction, we have $d \in \cstyle{D}^\J$ iff $d \in \cstyle{D_1}^\J$ and $d \in \cstyle{D_2}^\J$.
			Note that $\cstyle{D_1} \sqcap \cstyle{D_2} \in \concepts$ guarantees $\cstyle{D_1}, \cstyle{D_2} \in \concepts$ from closure under subconcept.
			We can thus apply the induction hypothesis on both $\cstyle{D_1}$ and $\cstyle{D_2}$ to obtain $d \in \cstyle{D}^\J$ iff $d \in \cstyle{D_1}^\I$ and $d \in \cstyle{D_2}^\I$.
			By the semantics of the conjunction again, this now gives $d \in \cstyle{D}^\J$ iff $d \in \cstyle{D}^\I$.
			\smallskip

		\item[$\cstyle{D} = \lnot \cstyle{E}$.]
			Let $d \in \domainof{\J}$.
			From the semantics of the negation, we have $d \in \cstyle{D}^\J$ iff $d \notin \cstyle{E}^\J$.
			Note that $\lnot \cstyle{E} \in \concepts$ guarantees $\cstyle{E} \in \concepts$ from closure under negation.
			We can thus apply the induction hypothesis on $\cstyle{E}$ to obtain $d \in \cstyle{D}^\J$ iff $d \notin \cstyle{E}^\I$.
			By the semantics of the negation again, this now gives $d \in \cstyle{D}^\J$ iff $d \in \cstyle{D}^\I$.
			\smallskip

		\item[$\cstyle{D} = \exists \rstyle{r}.\cstyle{E}$.]
			Let $d \in \domainof{\J}$.

			$(\Leftarrow)$.
			Assume $d \in \cstyle{D}^\I$, that is there exists $e \in \cstyle{E}^\I$ with $(d, e) \in \rstyle{r}^\I$.
			Let $t := \type_\I(e)$.
			In particular, we have $\cstyle{E} \in t$.
			Thus, by definition of $w_{t}$ we have $w_{t} \in \cstyle{E}^\I$.
			Note that $\exists \rstyle{r}.\cstyle{E} \in \concepts$ guarantees $\cstyle{E} \in \concepts$ by closure under subconcept.
			Therefore, we can apply the induction hypothesis on $\cstyle{E}$ to obtain that $w_{t} \in \cstyle{E}^\J$.

			\begin{itemize}
				\item
			In the $\ALCI$ case, if $\rstyle{r}$ is a role name, we obtain $(d, w_{t}) \in \rstyle{r}^\J$ by definition of $\rstyle{r}^\J$ and we are done.
			Otherwise $\rstyle{r} = \rstyle{s}^-$ with $\rstyle{s}$ a role name, in which case by definition of $\rstyle{s}^\J$ we obtain $(w_t, d) \in \rstyle{s}^\J$, which in turn gives $(d, w_t) \in \rstyle{r}^\J$ and we are done.

				\item
			In the $\ALCF$ case, if there exists $e' \in t^\I \cap \domainof{\J}$ such that $(d, e') \in \rstyle{r}^\I$, then $e' \in \cstyle{E}^\I$ since $\cstyle{E} \in t$ and we are done.
			Otherwise the definition of $\rstyle{r}^\J$ ensures $(d, w_t) \in \rstyle{r}^{\J}$ and we are done.
			\end{itemize}

			$(\Rightarrow)$.
			Assume $d \in \cstyle{D}^\J$, that is there exists $e \in \cstyle{E}^\J$ with $(d, e) \in \rstyle{r}^\J$.
			Let $t := \type_\I(e)$.
			In particular, we have $\cstyle{E} \in t$.
			Note that $\exists \rstyle{r}.\cstyle{E} \in \concepts$ guarantees $\cstyle{E} \in \concepts$ by closure under subconcept.
			Therefore, we can apply the induction hypothesis on $\cstyle{E}$ to obtain that $e \in \cstyle{E}^\I$.

			\begin{itemize}
				\item
			In the $\ALCI$ case, we treat the case of $\rstyle{r}$ being a role name (the case of $\rstyle{r}$ being an inverse role being similar).
			If $(d, e) \in \rstyle{r}^\J$ comes from $\rstyle{r}^\I \cap (\domainof{\J} \times \domainof{\J})$ then, \emph{a fortiori} we have $(d, e) \in \rstyle{r}^\I$ and we are done.
			If $d \in (\exists \rstyle{r}.t)^\I$ and $e = w_t$ then in particular $d \in (\exists \rstyle{r}.\cstyle{E})^\I$ and we are done.
			Otherwise $e \in (\exists \rstyle{r^-}.t')^\I$ and $d = w_{t'}$, then there must exists $(d', e) \in \rstyle{r}^\I$ with $\type_\I(d') = t'$.
			In particular, $d' \in (\exists \rstyle{r}.\cstyle{E})^\I$, and since $\exists \rstyle{r}.\cstyle{E} = \cstyle{D} \in \concepts$, it ensures $\cstyle{D} \in t'$.
			From $d = w_{t'}$ and the definition of witnesses, this yields $d \in \cstyle{D}^\I$ and we are done.

				\item
			In the $\ALCF$ case, if $(d, e) \in \rstyle{r}^\J$ comes from $\rstyle{r}^\I \cap (\domainof{\J} \times \domainof{\J})$ then, \emph{a fortiori} we have $(d, e) \in \rstyle{r}^\I$ and we are done.
			Otherwise, $d \in (\exists \rstyle{r}.t)^\I$ and $e = w_t$ then in particular $d \in (\exists \rstyle{r}.\cstyle{E})^\I$ and we are done.
			\end{itemize}

	\end{itemize}
	From ($\dagger$), it is clear that every concept inclusion from $\tbox$ is satisfied in $\J$.
	Every assertion from $\abox$ is also satisfied as, by construction, we have $\cstyle{A}^\I \cap \individuals(\abox) \subseteq \cstyle{A}^\J$ and $\rstyle{r}^\I \cap (\individuals(\abox) \times \individuals(\abox)) \subseteq \rstyle{r}^\J$.
	In the $\ALCI$ case, this already proves that $\J$ is a model of $\kb$ as desired.
	In the $\ALCF$ case, it remains to verify that every functionality assertion $\cstyle{D} \sqsubseteq\ \leq 1\ \rstyle{r}.\cstyle{E}$ is satisfied.
	Assume we have $d \in \cstyle{D}^\J$ and $e_1, e_2 \in \cstyle{E}^\J$ such that $(d, e_1), (d, e_2) \in \rstyle{r}^\J$.
	Note that from Property~($\dagger$), this gives $d \in \cstyle{D}^\I$ and $e_1, e_2 \in \cstyle{E}^\I$.
	We aim to prove that $e_1 = e_2$.
	Based on $(d, e_1), (d, e_2) \in \rstyle{r}^\J$, we distinguish four cases:
	\begin{itemize}
		\item If $(d, e_1), (d, e_2) \in \rstyle{r}^\I \cap (\domainof{\J} \times \domainof{\J})$ then from $\I$ being a model of $\cstyle{D} \sqsubseteq\ \leq 1\ \rstyle{r}.\cstyle{E}$, we get $e_1 = e_2$ and we are done.
		\item If $(d, e_1) \in \rstyle{r}^\I \cap (\domainof{\J} \times \domainof{\J})$ and $e_2 = w_t$ with $d \in (\exists \rstyle{r}.t)^\I \cap \domainof{\J} \text{ and } \forall e \in t^\I \cap \domainof{\J}, (d, e) \notin \rstyle{r}^\I$, then, there exists some $e_2' \neq e_1$ with $\type_\I(e_2') = t$ and $(d, e_2') \in \rstyle{r}^\I$.
		However, since $\cstyle{E} \in t$, we have $e_2' \in \cstyle{E}^\I$ and this contradicts $\I$ being a model of $\cstyle{D} \sqsubseteq\ \leq 1\ \rstyle{r}.\cstyle{E}$.
		\item The case of $(d, e_2) \in \rstyle{r}^\I \cap (\domainof{\J} \times \domainof{\J})$ and $e_1 = w_t$ with $d \in (\exists \rstyle{r}.t)^\I \cap \domainof{\J} \text{ and } \forall e \in t^\I \cap \domainof{\J}, (d, e) \notin \rstyle{r}^\I$ is symmetric.
		\item If $e_1 = w_{t_1}$ with $d \in (\exists \rstyle{r}.{t_1})^\I \cap \domainof{\J} \text{ and } \forall e \in {t_1}^\I \cap \domainof{\J}, (d, e) \notin \rstyle{r}^\I$ and $e_2 = w_{t_2}$ with $d \in (\exists \rstyle{r}.{t_2})^\I \cap \domainof{\J} \text{ and } \forall e \in {t_2}^\I \cap \domainof{\J}, (d, e) \notin \rstyle{r}^\I$, then there exists some $e_1'$ with $\type_\I(e_1') = t_1$ and $(d, e_1') \in \rstyle{r}^\I$ but also some $e_2'$ with $\type_\I(e_2') = t_2$ and $(d, e_2') \in \rstyle{r}^\I$.
		From $\cstyle{E} \in t_1$ and $\cstyle{E} \in t_2$, we get $e_1', e_2' \in \cstyle{E}^\I$.
		Since $\I$ is a model of $\cstyle{D} \sqsubseteq\ \leq 1\ \rstyle{r}.\cstyle{E}$, this ensures $e_1' = e_2'$.
		Therefore, $t_1 = t_2$ and we obtain $e_1 = e_2$.
	\end{itemize}

	We can now conclude the proof of Lemma~\ref{lemma:fmp-alci-alcf} as:
	\begin{itemize}
		\item we just proved that $\J$ is a model of $\kb$;
		\item from the discussed bound on $\sizeof{\types(\I)}$, we derive $\sizeof{\domainof{\J}} \leq \sizeof{\individuals(\abox)} + 4^{\sizeof{\tbox}}$ as desired;
		\item from the assumption that $\cstyle{C}^\I \neq \emptyset$ and $\cstyle{C} \in \concepts$, there exists a type $t_0 \in \types(\I)$ such that $\cstyle{C} \in t_0$. Therefore, by ($\dagger$),  $w_{t_0} \in \cstyle{C}^\J$ thus $\cstyle{C}^\J \neq \emptyset$.
	\end{itemize}

\end{proof}

\begin{lemma}
	\label{lemma-concept-alci-alcf-plus-one}
	Let $\kb$ be an $\ALCI$ (resp.\ $\ALCF$) KB and $\cstyle{C} \in \cnames$.
	If there exists a model $\I$ of $\kb$ with $1 \leq \sizeof{\cstyle{C}^\I} < \infty$,
	then there exists a model $\J$ of $\kb$ with $\sizeof{\cstyle{C}^{\J}} = \sizeof{\cstyle{C}^\I} + 1$.
\end{lemma}

\begin{proof}
	We treat separately the $\ALCI$ and $\ALCF$ cases.
	\medskip

	{\bf $\ALCI$ case.}
	Assume there exists a model $\I$ of an $\ALCI$ KB $\kb$ with $1 \leq \sizeof{\cstyle{C}^\I} < + \infty$.
	In particular, $\cstyle{C}^\I$ is non-empty so there is some $u \in \cstyle{C}^\I$.
	We set $\domainof{\J} := \domainof{\I} \cup \{ v \}$ where $v$ is a fresh element.
	The interpretation $\J$ is now defined by extending $\I$ and copying on $v$ all facts that holds on $u$, that is for each concept name $\cstyle{A}$ and role name $\rstyle{p}$:
	\begin{align*}
		\cstyle{A}^\J & := \cstyle{A}^\I \cup \{ v \mid u \in \cstyle{A}^\I \}
		\\
		\rstyle{p}^\J & := \, \rstyle{p}^\I \cup \{ (v, e) \mid (u, e) \in \rstyle{p}^\I \} \cup \{ (e, v) \mid (e, u) \in \rstyle{p}^\I \}
	\end{align*}
	It is trivial to verify that $\J$ is a model of $\kb$ and that $\sizeof{\cstyle{C}^{\J}} = \sizeof{\cstyle{C}^\I} + 1$ as desired.
	\medskip

	{\bf$\ALCF$ case.}
	Assume there exists a model $\I$ of an $\ALCF$ KB $\kb$ with $1 \leq \sizeof{\cstyle{C}^\I} < + \infty$.
	In particular, $\cstyle{C}^\I$ is non-empty so there is some $u \in \cstyle{C}^\I$.
	We set $\domainof{\J} := \domainof{\I} \cup \{ v \}$ where $v$ is a fresh element.
	The interpretation $\J$ is now defined by extending $\I$ and copying on $v$ all concepts that holds on $u$ and all of its outgoing role names, that is for each concept name $\cstyle{A}$ and role name $\rstyle{p}$:
	\begin{align*}
		\cstyle{A}^\J & := \cstyle{A}^\I \cup \{ v \mid u \in \cstyle{A}^\I \}
		\\
		\rstyle{p}^\J & := \, \rstyle{p}^\I \cup \{ (v, e) \mid (u, e) \in \rstyle{p}^\I \}
	\end{align*}
	It is trivial to verify that $\J$ is a model of $\kb$ and that $\sizeof{\cstyle{C}^{\J}} = \sizeof{\cstyle{C}^\I} + 1$ as desired.
\end{proof}


\end{toappendix}
Moreover, 
if we focus on negation-free DLs, the situation becomes even more favorable:

\begin{theoremrep}
	\label{theorem-shape-concept-eli-elf-elif}
	A non-trivial subset of $\ninfty$ is $\ELI$- (resp.\ $\ELF$-) concept realizable iff it has shape $\llbracket M, \infty \rrbracket$ for some $M \in \NN$.
	For $\ELIF$, the shape $\{ \infty \}$ is also permitted.
\end{theoremrep}

\begin{toappendix}
Here as well, we use a intermediate statement similar in spirit to Lemma~\ref{lemma-concept-alci-alcf-plus-one}, but which already holds if the starting model $\I$ does not feature any instance of $\cstyle{C}$.

\begin{lemma}
	\label{lemma-concept-elif-plus-one}
	Let $\kb$ be an $\ELIF$ KB and $\cstyle{C} \in \cnames$.
	If there exists a model $\I$ of $\kb$ with $\sizeof{\cstyle{C}^\I} < \infty$,
	then there exists a model $\J$ of $\kb$ with $\sizeof{\cstyle{C}^{\J}} = \sizeof{\cstyle{C}^\I} + 1$.
\end{lemma}

\begin{proof}
	We extend such a model $\I$ with an element satisfying every (positive) fact, that is we set $\domainof{\J} := \domainof{\I} \cup \{ e \}$ where $e$ is a fresh element and define for every $\cstyle{A} \in \cnames$ and $\rstyle{p} \in \rnames$:
	\[
		\cstyle{A}^\J := \cstyle{A}^\I \cup \{ e \}
		\qquad\text{and}\qquad
		\rstyle{p}^\J := \, \rstyle{p}^\I \cup \{ (e, e) \}.
	\]
	It can be verified that $\J \models \kb$ and $\sizeof{\cstyle{C}^{\J}} = \sizeof{\cstyle{C}^\I} + 1$.
\end{proof}

To conclude the `only-if' part of Theorem~\ref{theorem-shape-concept-eli-elf-elif}, it suffices to recall that while $\ELI$ and $\ELF$ enjoy the finite model property, an $\ELIF$ KB may enforce infinitely many instances of $\cstyle{C}$ via inverse functional roles (see Example~\ref{example-concept-elif-infinity} in the main part of the paper).
Regarding the `if' direction, it is trivial that a spectrum with shape $\llbracket M, \infty \rrbracket$ for some $M \in \NN$ can be realized with an empty TBox and an ABox specifying $M$ different instances of $\cstyle{C}$. 

\end{toappendix}

For the shape $\{ \infty \}$, we use the following well-known example of an $\ELIF$ KB (notice it is also a $\dllitef$ KB).

\begin{example}
	\label{example-concept-elif-infinity}
	Consider the $\ELIF$ KB $\kb = (\tbox, \abox)$ with $\abox = \{ \rstyle{r}(\istyle{a}, \istyle{a}), \rstyle{r}(\istyle{a}, \istyle{b}) \}$ and $\tbox = \{ ~
	\cstyle{C} \sqsubseteq \exists \rstyle{r}.\top, ~ \exists \rstyle{r^-}.\top \sqsubseteq \cstyle{C}, ~ \top \sqsubseteq\ \leq 1\ \rstyle{r^-}.\top
	~ \}$.
	It can be verified that $\spectrum{\kb}{q_\cstyle{C}} = \{ \infty \}$.
\end{example}

\subsection{$\ELIFbot$ KBs and cycles reversion}
\label{subsection:concept-elifbot}

\begin{toappendix}
	\subsection*{Proofs for Section~\ref{subsection:concept-elifbot} ($\ELIFbot$ KBs and cycles reversion)}
\end{toappendix}

We now turn to the two remaining DLs, namely $\ELIFbot$ and $\dllitef$, in which inverse functional roles and negation are supported.
We begin with an example illustrating that, compared to the previously investigated restrictions of $\ALCIF$, new spectrum shapes can be realized.

\begin{example}
	\label{example-bad-concept}
	We construct an $\ELIFbot$ KB $\kb = (\tbox, \abox)$ s.t.\ $\spectrum{\kb}{q_\cstyle{C}} = \{ 4 \} \cup \llbracket 6, \infty \rrbracket$.
	The TBox $\tbox$ contains the axioms:
	\begin{center}\(
	\begin{array}{c}
		\top\ \sqsubseteq\ \cstyle{C}\ \sqcap\ \exists \rstyle{r}.\cstyle{A_1}\ \sqcap\ \exists \rstyle{r}.\cstyle{A_2}
		\qquad
		\top\ \sqsubseteq\ \leq 1\ \rstyle{s}^-.\top
		\smallskip \\
		\exists \rstyle{r}.\cstyle{X} \sqsubseteq \cstyle{Y}
		\qquad
		\exists \rstyle{r}.\cstyle{Y} \sqsubseteq \cstyle{X}
		\qquad
		\cstyle{Y} \sqsubseteq \exists \rstyle{s}.\cstyle{X}
		\qquad
		\cstyle{X} \sqsubseteq \exists \rstyle{s}.\cstyle{Y}
		\smallskip \\
		\cstyle{A_1} \sqcap \cstyle{A_2} \sqsubseteq \bot
		\qquad
		\cstyle{X} \sqcap \cstyle{Y} \sqsubseteq \bot
	\end{array}
	\)
	\end{center}
	The ABox $\abox$
	contains the $8$ concept assertions
	$\cstyle{A_1}(\istyle{x_1})$, $\cstyle{A_2}(\istyle{x_2})$,
	$\cstyle{A_1}(\istyle{y_1})$, $\cstyle{A_2}(\istyle{y_2})$,
	$\cstyle{X}(\mathsf{x_1})$, $\cstyle{X}(\mathsf{x_2})$,
	$\cstyle{Y}(\mathsf{y_1})$, $\cstyle{Y}(\mathsf{y_2})$,
	and the $12$ role assertions
	$\rstyle{s}(u_1, v_1)$, $\rstyle{s}(u_2, v_2)$, $\rstyle{r}(u_i, v_j)$ for each $i, j \in \lbrace 1, 2 \rbrace$ and each $u, v \in \lbrace \mathsf{x}, \mathsf{y} \rbrace$ with $u \neq v$.
\end{example}

A
representation of this spectrum according to Problem~1 is
$(S, M, \alpha) := (\{ 4 \}, 6, 1)$.
Such possibly non-trivial part $S$ of the spectrum
 make a full characterization of realizable sets hard to reach.
Interestingly, however, every $\ELIFbot$-concept realizable set can be represented with $\alpha = 1$.

\begin{theorem}
	\label{theorem-shape-concept-elifbot}
	If a non-trivial subset of $\ninfty$ is $\ELIFbot$-concept realizable, then it has shape $\{ \infty \}$, $\{ 0, \infty \}$, or $S \cup \llbracket M, \infty \rrbracket$ for some $M \in \NN$ and $S \subseteq \llbracket 0, M \rrbracket$.
\end{theorem}

The remainder of this section is devoted to the proof of this theorem.
Let us first eliminate the easy cases, proving $\{ \infty \}$ and $\{ 0, \infty \}$ are already $\dllitef$-concept realizable.
The $\{ \infty \}$ shape has been obtained in Example~\ref{example-concept-elif-infinity}.
To realize $\{ 0, \infty \}$, we rely on concept disjointness
as follows:

\begin{example}
	\label{example-dllitef-zero-infinity}
	Consider the $\dllitef$ TBox $\tbox$ containing:
	\smallskip\newline
	\(
	\begin{array}{c@{\quad~}c@{\quad~}c@{\quad~}c}
		\cstyle{C} \sqsubseteq \exists \rstyle{r}
		&
		\exists \rstyle{r^-} \sqsubseteq \cstyle{C}
		&
		\top \sqsubseteq\ \leq 1\ \rstyle{r^-}.\top
		&
		\multirow{2}{*}{
			$
			\exists \rstyle{r^-}\sqcap \exists \rstyle{s^-} \sqsubseteq \bot
			$
		}
		\smallskip \\
		\cstyle{C} \sqsubseteq \exists \rstyle{s}
		&
		\exists \rstyle{s^-} \sqsubseteq \cstyle{C}
		&
		\top \sqsubseteq\ \leq 1\ \rstyle{s^-}.\top
	\end{array}
	\)%
	\smallskip\newline\noindent
It is immediate to verify that $\spectrum{(\tbox, \emptyset)}{q_\cstyle{C}} = \{ 0, \infty \}$.
\end{example}

It remains to verify that every other non-trivial subset $V$ of $\ninfty$ that is $\ELIFbot$-concept realizable has shape $S \cup \llbracket M, \infty \rrbracket$ for some $M \in \NN$ and $S \subseteq \llbracket 0, M \rrbracket$.
Let $V$ be such a set and $\kb$ an $\ELIFbot$ KB s.t.\ $V = \spectrum{\kb}{q_\cstyle{C}}$ for some concept name $\cstyle{C}$.
We prove that $\spectrum{\kb}{q_\cstyle{C}}$ actually contains two consecutive non-zero integers $n$ and $n+1$, which guarantees, from closure under addition, that every integer greater than $n(n+1)$ is also in $\spectrum{\kb}{q_\cstyle{C}}$.
Setting $M = n(n+1)$ and $S = \spectrum{\kb}{q_\cstyle{C}} \cap \llbracket 0, M \rrbracket$ will then conclude the proof.

Since $V$ is non-trivial and neither $\{ 0, \infty \}$ nor $\{ \infty \}$, it contains a non-zero integer.
In other words, the concept $\cstyle{C}$ admits a finite interpretation in some (potentially infinite) model $(\star)$.
To exploit this fact, we refine cycle-reversion techniques which have been developed to study finite reasoning in similar logics \cite{CKV90,rosati2008finite,GarciaLS14}.
More precisely, we tailor the notion of cycles to characterize under which conditions the interpretation of $\cstyle{C}$ may be finite.
By $(\star)$, those conditions are satisfied and we adapt a construction from the latter reference to produce models $\I$ and $\J$ of $\kb$ s.t\ $\sizeof{\cstyle{C}^\J} = \sizeof{\cstyle{C}^{\I}} + 1 < \infty$ as desired.
Henceforth, we assume $\ELIFbot$ KBs to be in normal form, that is every axiom in the TBox has one of the following shapes:
\begin{center}
\(
\begin{array}{cccc}
	K \sqsubseteq \cstyle{A} & K \sqsubseteq \exists \rstyle{r}.K' & \exists \rstyle{r}. K \sqsubseteq  K' & K \sqsubseteq\ \leq 1\ \rstyle{r}. K'
\end{array}
\)
\end{center}
where $\cstyle{A} \in \cnames \cup \{\bot\}$, $\rstyle{r} \in \posroles$ and $K, K'$ are conjunctions of concepts names.
This is a reformulation of the normal form used in \cite{GarciaLS14} and it can be verified that putting a KB in such a normal form does not affect spectra of queries on this KB.

\begin{toappendix}
	We briefly highlight the minor difference between the normal form for Horn-$\ALCIF$ proposed in \cite{GarciaLS14} and the one presented in Section~\ref{subsection:concept-elifbot} for $\ELIFbot$ KBs.
	First notice that every $\ELIFbot$ KB is in particular an Horn-$\ALCIF$ KB and thus, from the above reference we obtain a normal form in which every axiom in the TBox has one of the following shape:
	\begin{center}
		\(
		\begin{array}{cccc}
			K \sqsubseteq \cstyle{A} & K \sqsubseteq \exists \rstyle{r}.K' & K \sqsubseteq \forall \rstyle{r}. K' & K \sqsubseteq\ \leq 1\ \rstyle{r}. K'
		\end{array}
		\)
	\end{center}
	where $\cstyle{A} \in \cnames \cup \{\bot\}$, $\rstyle{r} \in \posroles$ and $K, K'$ are conjunctions of concepts names.
	Syntactically, the third shape of CI is not supported in an $\ELIFbot$ TBox.
	However, a simple semantical argument allows to transform any such CI $K \sqsubseteq \forall \rstyle{r}. K'$ into the equivalent $\ELIFbot$ CI $\exists \rstyle{r}^- . K \sqsubseteq K'$, leading to normal form used in Section~\ref{subsection:concept-elifbot}.
\end{toappendix}

We now present our refined notion of cycles
which itself relies on the following definition of inverse functional paths.

\begin{definition}
	An \emph{inverse functional path (IFP) in $\tbox$} is a sequence $K_0, \rstyle{r}_1, K_1, \dots, \rstyle{r}_n, K_n$ where $n \geq 1$, $K_0, \dots, K_n$ are conjunctions of concept names and $\rstyle{r}_1, \dots, \rstyle{r}_n$ are (potentially inverse) roles s.t.\ for all $0 \leq i < n$:
	\begin{center}
	\(
	\tbox \models K_i \incl \exists \rstyle{r}_{i+1}.K_{i+1}
	\quad \textrm{and} \quad
	\tbox \models K_{i+1} \incl\ \leq 1\ \rstyle{r}_{i+1}^-.K_i.
	\)
	\end{center}
\end{definition}

The interesting cycles for a concept $\cstyle{C}$ are the IFPs looping on themselves and forcing the presence of (at least) one instance of $\cstyle{C}$ ``per instance of the cycle''.
This latter property can also be expressed in terms of IFPs.

\begin{definition}
	\label{def:cycle}
	An IFP $K_0, \rstyle{r}_1, K_1, \dots, \rstyle{r}_n, K_n$ is a \emph{$\cstyle{C}$-generating cycle in $\tbox$} if
	$\tbox \models K_n  \incl K_0$ and there exists an IFP $L_0, \rstyle{s}_1, L_1, \dots, \rstyle{s}_m, L_m$ such that
	$\tbox \models K_i \incl L_0$ for some $0 \leq i \leq n$ and $\tbox \models L_m \incl \cstyle{C}$.
\end{definition}

We now reconcile with existing cycle reversion techniques by considering a completion of the original TBox containing reversed versions of each $\cstyle{C}$-generating cycle.

\newcommand{\finclosure}{\mathsf{finClosure}}
\newcommand{\finclosureof}[1]{\finclosure(#1)}

\begin{definition}
	We denote $\tbox_\cstyle{C}$ the $\ELIFbot$ TBox obtained from $\tbox$ by adding the following axioms, for each $\cstyle{C}$-generating cycle $K_0, \rstyle{r}_1, K_1, \dots, \rstyle{r}_n, K_n$ in $\tbox$ and each $0 \leq i < n$:
	\begin{center}
	\(
	K_{i+1} \incl \exists \rstyle{r}_{i}^-.K_{i}
	\quad \textrm{and} \quad
	K_{i} \incl\ \leq 1\ \rstyle{r}_{i+1}.K_{i+1}
	\)
	\end{center}
\end{definition}

The key result regarding this cycle reversion technique focused on a single concept is the following lemma:

\begin{lemmarep}
	\label{lemma:cycle-reversion}
	Let $(\tbox, \abox)$ be an $\ELIFbot$ KB and $\cstyle{C}$ a concept name.
	There exists a model $\I$ of $\kb$ s.t.\ $\sizeof{\cstyle{C}^\I} < \infty$ iff the KB $(\tbox_{\cstyle{C}}, \abox)$ is satisfiable.
	Furthermore, every such model $\I$ is a model of $(\tbox_{\cstyle{C}}, \abox)$.
\end{lemmarep}

\begin{toappendix}
	We prove the `only-if' direction in more details by proving the additional (and stronger) claim that every model $\I$ of $\kb$ s.t.\ $\sizeof{\cstyle{C}^\I} < \infty$ is a model of $(\tbox_{\cstyle{C}}, \abox)$.
	The `if' direction is an immediate consequence of Lemma~\ref{lemma:ils}, noticing that $\cstyle{C}$ is safe in $\tbox_\cstyle{C}$.

	Consider a model $\I$ of $\kb$ s.t.\ $\sizeof{\cstyle{C}^\I} < \infty$.
	We need to prove every axiom obtained by reversing $\cstyle{C}$-generating cycles of $\tbox$ is satisfied in $\I$.
	Let $K_0, \rstyle{r}_1, K_1, \dots, \rstyle{r}_n, K_n$ be a $\cstyle{C}$-generating cycle in $\tbox$ and let $L_0, \rstyle{s}_1, L_1, \dots, \rstyle{s}_m, L_m$ be an IFP satisfying Point~2 from the definition of $\cstyle{C}$-generating cycles.
	From $\I$ being a model of $\tbox$ and the definition of an IFP, it follows that $\sizeof{L_1^\I} \leq \sizeof{L_2^\I} \dots \leq \sizeof{L_m^\I}$.
	From Point~2.(a) and the assumption that $\cstyle{C}^\I$ is finite, we obtain that $L_1^\I$ is also finite.
	Again by definition af of IFP, applied this time to $K_0, \rstyle{r}_1, K_1, \dots, \rstyle{r}_n, K_n$, we have $\sizeof{K_1^\I} \leq \sizeof{K_2^\I} \dots \leq \sizeof{K_n^\I}$.
	By Point~1, $\sizeof{K_n^\I} \leq \sizeof{K_0\I}$ and thus $\sizeof{K_1^\I} = \sizeof{K_2^\I} \dots = \sizeof{K_n^\I}$.
	From Point~2.(a), there is some $0 \leq \iota \leq n$ such that $\sizeof{K_\iota^\I} \leq \sizeof{L_1^\I}$ and thus all $\sizeof{K_i^\I}$ are finite for all $0 \leq i \leq n$.
	Finiteness of every $K_i^\I$ joint with each $\rstyle{r}_i$ defining an injection from $K_{i}^\I$ to $K_{i+1}^\I$ yield that $\rstyle{r}_i$ actually defines a bijection between those two.
	In particular, $\I \models K_{i+1} \incl \exists \rstyle{r}_{i}^-.K_{i}$ and $\I \models K_{i} \incl\ \leq 1\ \rstyle{r}_{i+1}.K_{i+1}$ as desired.

	The `only-if' direction is now trivial: if there exists a model $\I$ of $\kb$ s.t.\ $\sizeof{\cstyle{C}^\I} < \infty$, then, by the above, $\I$ is a model of $(\tbox_{\cstyle{C}}, \abox)$ and thus the latter KB is satisfiable.
\end{toappendix}

The `only-if' direction of the above is the easy one: 
the IFPs in Definition~\ref{def:cycle} enforce that for every $K_i$ on a $\cstyle{C}$-generating cycle, there is an injection from $K_i^\I$ to $\cstyle{C}^\I$.
Since $\cstyle{C}^\I$ is finite, so are all these $K_i^\I$. 
It follows that the injective function from $K_i^\I$ to $K_{i+1}^\I$ defined by $\rstyle{r}_{i+1}^\I$  is actually a bijection.
From there, it is readily checked that
$\I$ is a model of $(\tbox_{\cstyle{C}}, \abox)$
as claimed,
and thus $(\tbox_{\cstyle{C}}, \abox)$ is satisfiable.

For the `if' direction of Lemma~\ref{lemma:cycle-reversion}, assume $\kb_\cstyle{C} = (\tbox_\cstyle{C}, \abox)$ is satisfiable.
We adapt a construction from \cite{GarciaLS14} to assemble a model $\I$ of $\kb_\cstyle{C}$ (thus, of $\kb$) in which ${\cstyle{C}^\I}$ is finite.
Our construction actually takes as input any $\ELIFbot$ KB $\kb = (\tbox, \abox)$ and guarantees the above finiteness condition for all ``{safe}'' concepts of $\tbox$.
A concept $\cstyle{C}$ is a \emph{safe concept} of $\tbox$ if every axiom from $\tbox_\cstyle{C}$ is already entailed by $\tbox$.
In particular, $\cstyle{C}$ is safe in $\tbox_\cstyle{C}$.

We introduce some relevant preliminaries.
Let $\cnames(\tbox)$ be the set of concept names used in $\tbox$.
A type for $\tbox$ is a subset $t \subseteq \cnames(\tbox)$ s.t.\ there is a model $\I$ of $\tbox$ and a $d \in \domainof{\I}$ s.t.\ $\type_\I(d) = t$, where $\type_\I(d)$ is the type realized at $d$ in $\I$, \emph{i.e.}:
\begin{center}
\(
\type_\I(d) := \{ \cstyle{A} \in \cnames(\tbox) \mid d \in \cstyle{A}^\I \}
\)
\end{center}
We use $\types(\tbox)$ to denote the set of all types of $\tbox$.
A type is \emph{critical} in $\tbox$ if it occurs on a $\cstyle{C}$-generating cycle for some safe concept $\cstyle{C}$ of $\tbox$.
Otherwise it is a \emph{free} type in $\tbox$.
For $t, t'\in \types(\tbox)$ and $\rstyle{r}$ a role, we write:
\begin{itemize}
	\item
	$t \rightgen{r} t'$ if $\tbox \models t \sqsubseteq \exists r.t'$ and $t'$ is maximal for this property;
	\item
	$t \rightdep{r} t'$ if $t \rightgen{r} t'$ and $\tbox \models t' \sqsubseteq\ \leq 1\ \rstyle{r^-}.t$;
	\item
	$t \bidep{r} t'$ if $t \rightdep{r} t'$ and $t' \rightdep{r^-} t$.
\end{itemize}

A type class is a non-empty set $P \subseteq \types(\tbox)$ such
that $t \in P$ and $t \bidep{r} t'$ implies $t' \in P$, and $P$ is minimal
with this condition. Note that the set of all type classes is
a partition of $\types(\tbox)$. We set $P \prec P'$ if there are $t \in P$
and $t' \in P'$ with $t' \subsetneq t$. Let $\prec^+$ be the transitive closure of
$\prec$.
It is known from \cite{GarciaLS14} that $\prec^+$ is a strict partial order.

The initial interpretation $\ils{\kb}{0}$ is defined by introducing an element for every ABox individual and an element $d_t$ for each $t \in \types(\tbox_f)$.
Formally, we define:
\begin{center}
\(
\begin{array}{rl}
	\domainof{\ils{\kb}{0}} & = \mathsf{Ind}(\abox) \cup \{ d_t \mid t \in \types(\tbox_f) \}
	\\[1mm]
	\cstyle{A}^{\ils{\kb}{0}} & = \{ \istyle{a} \in \mathsf{Ind}(\abox) \mid \cstyle{A} \in \type_\kb(\istyle{a}) \} \cup \{ d_t \mid \cstyle{A} \in d_t \}
	\\[1mm]
	\rstyle{r}^{\ils{\kb}{0}} & = \{ (\istyle{a}, \istyle{b}) \mid \rstyle{r}(\istyle{a}, \istyle{b}) \in \abox \}
\end{array}
\)
\end{center}
where $\type_\kb(\istyle{a}) := \{ \cstyle{A} \in \cnames \mid \kb \models \cstyle{A}(\istyle{a}) \}$.

	We describe three completion rules $\crule{1}$, $\crule{2}$, $\crule{3}$ applicable to an interpretation $\I$.
	Informally, whenever an existing $d$ with type $t'$ needs a $\rstyle{r}.t$-successor for some $t$, then $\crule{3}$ connects $d$ to $d_t$ if the chosen witness may be used by several such elements $d$ (that is $t' \not\rightdep{r} t$).
	If, on the other hand, the witness cannot be reused, then $\crule{1}$ simply introduces a dedicated fresh element $e$ if $t$ is either free or not in the type class of $t'$.
	Otherwise $t$ is critical and in the type class $P$ of~$t'$. Then $\crule{2}$ introduces or reuses existing elements to instantiate the whole type class $P$ at once.
	This requires only finitely many fresh instances of each type in $P$, in particular, critical types in $P$ are instantiated only finitely many times.

\begin{itemize}
	\item[$\crule{1}$.]
	For each $d \in \domainof{\I}$, each $t \in \types(\tbox)$ and $\rstyle{r} \in \posroles$ s.t.:
	 $\type_\I(d) \rightdep{r} t$,
	 $d \notin (\exists \rstyle{r}.t)^\I$,
	 and either $t \not\rightdep{r^-} \type_\I(d)$ or $t$ is a free type in $\tbox$,
	add a fresh domain element $e$ and modify the interpretation of concept names such that $\type_\I(e) = t$ and $(d, e) \in \rstyle{r}^\I$.

	\item[$\crule{2}$.]
	Choose a type class $P$ that is minimal w.r.t. the order $\prec^+$, a $\lambda = s \bidep{r} s'$ with $s \in P$, and an element $d \in s^\I \setminus (\exists \rstyle{r}.s')^\I$.
	If such a choice is not possible, then the application of $\crule{2}$ just returns the original model $\I$.
	Otherwise, for each $\lambda = s \bidep{r} s'$ with $s \in P$, set:
\begin{center}
\(
X^\I_{\lambda, 1} = s^\I \setminus (\exists \rstyle{r}.s')^\I
\qquad
X^\I_{\lambda, 2} = s'^\I \setminus (\exists \rstyle{r}^-.s)^\I.
\)
\end{center}
	Take (i) a fresh set $\Delta_s$ for each $s \in P$ such that $\sizeof{\Delta_s} \leq \max \{ \sizeof{t^\I} \mid t \in P \}$ and (ii) a bijection $\pi_\lambda$ from $X^\I_{\lambda, 1} \cup \Delta_s$ to $X^\I_{\lambda, 2} \cup \Delta_{s'}$ for each $\lambda = s \bidep{r} s'$ with $s, s' \in P$ and $\rstyle{r} \in \rnames$.
	A concrete construction of such sets and bijections can follow the one detailed in \cite{GarciaLS14}.
	We additionally require the above to minimize $\sizeof{\biguplus_{s \in P} \Delta_s}$.
	Now extend $\I$ as follows:
	\begin{itemize}
		\item add all domain elements in $\biguplus_{s \in P} \Delta_s$;
		\item extend $\rstyle{r}^\I$ with $\pi_\lambda$, for each $\lambda = s \bidep{r} s'$ with $s, s' \in P$ and $\rstyle{r}$ a role name;
		\item interpret concept names so that $\type_\I(d) = s$ for all $d \in \Delta_s$, $s \in P$.
	\end{itemize}

	\item[$\crule{3}$.]
	For each $d \in \domainof{\I}$, each $t \in \types(\tbox)$ and each $\rstyle{r} \in \posroles$ s.t.\ $\type_\I(d) \rightgen{r} t$, $\type_\I(d) \not\rightdep{r} t$, and $d \notin (\exists \rstyle{r}.t)^\I$.
	Add the edge $(d, d_t)$ to $\rstyle{r}^\I$.

\end{itemize}
We denote $\crule{k}(\I)$ the application of $\crule{k}$ to interpretation $\I$.
For $\crule{2}$, this is ambiguous since its application may depend on several choices (a minimal type class, etc).
This does not matter for our construction and we simply assume a fixed choice. 
While it is easily verified that $\crule{3}$ is idempotent, that is $\crule{3}(\crule{3}(\I)) = \crule{3}(\I)$, it is not the case for $\crule{1}$ in general.
However, since applying $\crule{1}$ on $\I$ does not alter the interpretation of concept and roles names on the original domain $\domainof{\I}$, we can safely define $\crule{1}^\infty(\I) = \bigcup_{n = 1}^\infty \crule{1}^n(\I)$, where $\crule{1}^n$ denotes $n$ successive applications of $\crule{1}$.
We now view $\crule{1}^\infty$ as a completion rule, which is clearly idempotent.

Starting from the initial interpretation $\ils{\kb}{0}$ previously defined, we complete it as follows:
\begin{center}
\(
\ils{\kb}{n+1} = \crule{3}(\crule{2}(\crule{1}^{\infty}(\ils{\kb}{n})))
\qquad\textrm{ and }\qquad
\ils{\kb}{} = \bigcup_{n = 0}^\infty \ils{\kb}{n}.
\)
\end{center}
Here again, notice that each rule application on $\I$ preserves the interpretation of concept names on $\domainof{\I}$ and can only extend those of role names, so $\ils{\kb}{}$ is well-defined.
In fact, we prove that $\ils{\kb}{}$ is obtained after finitely many steps: there exists $N \in \NN$ such that $\ils{\kb}{N+1} = \ils{\kb}{N}$.
Crucially, this guarantees that only finitely many instances of every critical type and safe concepts are introduced.
This culminates in the following lemma, which also concludes the proof of Lemma~\ref{lemma:cycle-reversion}.

\newcommand{\safe}{\mathsf{safe}}

\begin{lemmarep}
	\label{lemma:ils}
	If $\kb = (\tbox, \abox)$ is a satisfiable $\ELIFbot$ KB, then $\ils{\kb}{}$ is a model of $\kb$ and $\cstyle{C}^{\ils{\kb}{}}$ is finite for all safe $\cstyle{C}$ of $\tbox$.
\end{lemmarep}

\begin{proof}
	The main challenge to prove Lemma~\ref{lemma:ils} is to guarantee that the procedure is well-defined and terminates after finitely many steps.
	Once this is achieved, modelhood easily follows and it remains to prove that each step only introduces finitely many instances of each safe concept to obtain the desired property on interpretations of safe concepts.

	Regarding the application of each rule being well-defined, we have omitted a point in the main body of the paper: to construct sets $\Delta_s$ and bijections $\pi_\lambda$ as proposed in the reference \cite{GarciaLS14}, it is needed to guarantee that if $\crule{2}$ is applied on $\I$ and $\tbox \models K \sqsubseteq\ \leq 1\ \rstyle{r}.K'$, then $\I \models K \sqsubseteq\ \leq 1\ \rstyle{r}.K'$.
	To this end, we reproduce here the invariants proposed in the above reference, notably the third one that exactly states the desired property.
	\begin{lemma}[\cite{GarciaLS14}]
		\label{lemma:imported}
		Applications of rules $\crule{1}$, $\crule{2}$ and $\crule{3}$ preserve the following invariants when applied on $\I$:
		\begin{itemize}
			\item[(i1)] $\type_\I(d) \in \types(\tbox)$ for all $d \in \Delta^\I$;
			\item[(i2)] if $(d, d') \in \rstyle{r}^\I \setminus (\individuals(\abox) \times \individuals(\abox))$, then we have $\type_\I(d) \rightgen{r} \type_\I(d')$ or $\type_\I(d') \rightgen{r^-} \type_\I(d)$;
			\item[(i3)] If $\tbox \models K \sqsubseteq\ \leq 1\ \rstyle{r}.K'$, then $\I \models K \sqsubseteq\ \leq 1\ \rstyle{r}.K'$.
		\end{itemize}
	Furthermore, in a concrete application of $\crule{2}$, if $\lambda = s \bidep{r} s'$ with $s, s' \in P$ and $(d, d') \in \pi_\lambda$, then $\type_\I(d) = s$ and $\type_\I(d') = s'$.
	\end{lemma}
	It is easily verified that the proof in the reference\footnote{See the corresponding technical report available here: \textsf{www.informatik.uni-bremen.de/tdki/research/papers/2014/ILS-KR14.pdf}} still carries on with our definition of $\crule{1}$ being slightly more liberal regarding the types it may introduce (and, consequently, $\crule{2}$ being slightly more restricted).
	This difference does not affect proofs of invariants but impacts the size of the resulting models: our construction does not guarantee that the overall procedure yields a finite model, but only that the interpretations of safe concepts are finite.
	The additional statement corresponds to Lemma~19 in the above reference, whose proof is independent from our adapted construction.

	We also state the following basic observations:
	\begin{enumerate}
		\item By construction, once an element is introduced, its type remains the same.
		Formally, if $e \in \domainof{\ils{\kb}{n}}$, then for every $m \geq n$, we have $\type_{\ils{\kb}{m}}(e) = \type_{\ils{\kb}{n}}(e)$.
		\item If $\J$ has been obtained from $\crule{2}$ applied on $\I$ with type class $P$, then for all $s, s' \in P$ and $\lambda = s \bidep{r} s'$, we have $s^{\J} \setminus (\exists \rstyle{r}.s')^{\J} = \emptyset$.
	\end{enumerate}

	\medskip

	{\noindent \bf Termination.} We prove that the procedure terminates in at most $(2^\sizeof{\tbox} + 1)^{2^\sizeof{\tbox} + 1}$ steps, that is there exists an integer $N \leq (2^\sizeof{\tbox} + 1)^{2^\sizeof{\tbox} + 1}$ s.t.\ $\ils{\kb}{N+1} = \ils{\kb}{N}$.

	Both $\crule{1}$ and $\crule{3}$ being idempotent, it suffices to prove that $\crule{2}$ is applied at most that many times.
	To keep track of elements introduced along steps, we define for all $n \geq 1$ a subset $\delta_n$ of $\domainof{\ils{\kb}{n}}$ defined as:
	\[
	\delta_0 := \domainof{\crule{1}^\infty(\ils{\kb}{0})} \textrm{ and }\delta_n := \domainof{\crule{1}^\infty(\ils{\kb}{n})} \setminus \domainof{\crule{1}^\infty(\ils{\kb}{n-1})} \textrm{ for all } n > 0.
	\]
	Consider now the (directed) tree $G = (V, E)$ obtained by setting $V := \{ \delta_n \mid n \geq 0, \delta_n \neq \emptyset \}$ and $(\delta_{n_1}, \delta_{n_2}) \in E$ iff $n_1$ is the smallest integer s.t., in the application of $\crule{2}$ on $\crule{1}(\ils{\kb}{n_2 - 1})$, there are $s, s' \in P$ and $i \in \{ 1, 2 \}$ s.t.\ $\delta_{n_1} \cap X_{\lambda, i}^{\crule{1}(\ils{\kb}{n_2 - 1})} \neq \emptyset$ where $\lambda = s \bidep{r} s'$ and $P$ was the chosen type class for the $\crule{2}$ rule application.
	In words, $(\delta_{n_1}, \delta_{n_2})$ is an edge in $G$ if, when constructing $\ils{\kb}{n_2}$, we reused an element from $\crule{1}^\infty(\ils{\kb}{n_1})$ to instantiate the chosen type class and $n_1$ is minimal for this property.
	It is clear that every such $\delta_{n_2}$ has exactly one predecessor (except for $\delta_0$, having none), so that $G$ is indeed a tree.
	Note that every element in $\delta_{n_2}$ is: either obtained directly by the application of $\crule{2}$ on $\crule{1}^\infty(\ils{\kb}{n_2 - 1})$, or by a sequence of applications of $\crule{1}$ starting from an element introduced by the above application of $\crule{2}$, that is starting from an element of $\domainof{\ils{\kb}{n_2}} \setminus \domainof{\crule{1}^\infty(\ils{\kb}{n_2 - 1})}$.
	In particular, for every $e \in \delta_{n_2}$, there exists a sequence of elements introduced by successive applications of rules that starts from some element $d \in \delta_{n_1}$ and ends in $e$ ($\dagger$).

	We would now like to prove that $V$ is finite, which implies $\delta_n \neq \emptyset$ for only finitely many $n$ and concludes.
	From Observations~1 and~2 above, note that $G$ has a branching degree that is bounded by the number of type classes, that is by $2^\sizeof{\tbox}$.
	It thus suffices to prove that is has finite depth.
	Assume by means of contradiction that we can find a branch of length greater than $2^\sizeof{\tbox} + 1$, that is we can find a (directed) path $\delta_{n_0}, \dots, \delta_{n_K}$ with $n_0 = 0$ and $K > 2^\sizeof{\tbox} + 1$.
	Note that $n_0 < n_1 \dots < n_K$ and denote $P_{n_k}$ the type class used in the application of $\crule{2}$ to build $\ils{\kb}{n_k}$ from $\ils{\kb}{n_k - 1}$ for $k = 1, \dots, K$.
	Since the number of type classes is at most the number of types, we must have $P_{n_{\kappa}} = P_{n_\ell}$ for some $1 \leq \kappa < \ell \leq K$.
	Now, pick an element $e_\ell \in \delta_{n_\ell}$ that has been introduced by application of $\crule{2}$ with type class $P_{n_\ell}$ on $\ils{\kb}{n_\ell - 1}$.
	By iterated application of $(\dagger)$, there exists some $e_\kappa \in \delta_{n_\kappa}$ that starts a sequence of elements $e_\kappa = d_1, d_2, \dots, d_m = e_\ell$ introduced by successive applications of rules ending in $e_\ell$.
	Since every element in $\delta_{n_\kappa}$ has been obtained directly from the application of $\crule{2}$ with type class $P_{n_\kappa}$ on $\ils{\kb}{n_\kappa - 1}$ or by applications of $\crule{1}$ on those latter elements, we can assume w.l.o.g.\ that $e_\kappa$ has been obtained by the above application of $\crule{2}$.

	We now argue that this is a contradiction.
	Notice that $e_{\kappa}$ and $e_{\ell}$ belong to the same type class, namely $P_{n_{\kappa}}$ ($= P_{n_{\ell}}$) and have critical types since they have been introduced by $\crule{2}$.
	In particular, since $d_2, \dots, d_m$ are obtained by applications of $\crule{1}$ or $\crule{2}$, we have in any case:
	$\type_{\ils{\kb}{n_\ell}}(d_1) \rightdep{r_2} \type_{\ils{\kb}{n_\ell}}(d_2), \dots, \rightdep{r_m} \type_{\ils{\kb}{n_\ell}}(d_m)$ for some roles $\rstyle{r_1}, \dots, \rstyle{r}_m$.
	Since $\type_{\ils{\kb}{n_\ell}}(e_\kappa)$ and $\type_{\ils{\kb}{n_\ell}}(e_\ell)$ are critical and on the same type class, this can be completed into a $\cstyle{C}$-generating cycle for some safe concept $\cstyle{C}$ of $\tbox$ witnessing $\type_{\ils{\kb}{n_\ell}}(e_\kappa)$ and $\type_{\ils{\kb}{n_\ell}}(e_\kappa)$ being critical.
	Since $\cstyle{C}$ is safe, $\tbox$ entails $\tbox_\cstyle{C}$ and therefore we have $\type_{\ils{\kb}{n_\ell}}(d_1) \bidep{r_2} \type_{\ils{\kb}{n_\ell}}(d_2), \dots, \bidep{r_m} \type_{\ils{\kb}{n_\ell}}(d_m)$.
	Therefore, none of the elements $d_2, \dots, d_m$ has been introduced by an application of $\crule{1}$ as their respective type does not match precondition of $\crule{1}$.
	Thus, they all have been introduced by applications of $\crule{2}$.
	From Observation~2, it follows that they all have been introduced simultaneously, contradicting $\kappa < \ell$.

	We have proved that the depth and branching degree of $G$ are both bounded by $2^\sizeof{\tbox} + 1$, and therefore $\delta_n \neq \emptyset$ for at most $(2^\sizeof{\tbox} + 1)^{2^\sizeof{\tbox} + 1}$ distinct $n \geq 0$.
	Since $\delta_n = \emptyset$ implies $\delta_{n+1} = \emptyset$ for all $n \geq 0$, this guarantees that for $N_{\mathsf{steps}} := (2^\sizeof{\tbox} + 1)^{2^\sizeof{\tbox} + 1}$, we have $\ils{\kb}{N+1} = \ils{\kb}{N+1}$, proving the procedure terminates after at most $N_{\mathsf{steps}}$ steps.

	\medskip

	{\noindent \bf Modelhood.}
	We now turn to the modelhood of $\ils{\kb}{}$ and recall we assumed the KB to be in normal form.

	Axioms with shape $K \sqsubseteq \cstyle{A}$ are satisfied as every element in $\ils{\kb}{}$ has a type from $\types(\tbox)$ (see invariant (i)), which by definition, satisfies those.
	For axioms with shape $K \sqsubseteq \exists \rstyle{r}.K'$, it suffices to notice that, since the construction of $\ils{\kb}{}$ only requires finitely many steps, no application of a rule can be infinitely delayed and therefore such an existential requirement is ultimately fulfilled by one of the three rules.
	Satisfaction of axioms with shape $\exists \rstyle{r}. K \sqsubseteq  K'$ follows from invariant (ii) and the maximality condition in the definition of the relation $\rightgen{r}$.
	Axioms with shape $K \sqsubseteq\ \leq 1\ \rstyle{r}. K'$ are satisfied from invariant $(i3)$.

	\medskip

	{\noindent \bf Finiteness and size of the interpretations of safe concepts.}
	We now prove that the interpretation in $\ils{\kb}{}$ of every safe concept of $\tbox$ is finite by proving jointly that interpretations of all critical types are also finite.
	In interpretation $\I$, we denote $M_\I$ the number of elements in $\domainof{\I}$ that are either instances of a safe concept or of a critical type.

	In $\ils{\kb}{0}$, this is clear: we have $M_{\ils{\kb}{0}} \leq \sizeof{\abox} + 2^\sizeof{\tbox}$.

	We now prove every rule application only introduces finitely many new instances of $\cstyle{C}$ and of critical types.
	Recall $\crule{3}$ does not introduce any new element.

	By construction, applying $\crule{2}$ on a model $\I$ only introduces up to $\max \{ \sizeof{t^\I} \mid t \in P \}$ fresh elements per type $s$ in the considered type class $P$ (see the construction of sets $\Delta_s$).
	From $\crule{1}$ not applying on the selected $d \in \domainof{\I}$ used to trigger this application of $\crule{2}$, it follows that every type in $P$ is critical.
	Therefore $\crule{2}$ introduces at most $2^\sizeof{\tbox} \times M_\I$ new elements, and, \emph{a fortiori}, a most that many new instances of safe concepts and critical types.

	For $\crule{1}^\infty$, note that the successive applications of $\crule{1}$ on starting model $\I$ can be seen as a collection of trees, each rooted by some element $d \in \Delta^\I$.
	Within each tree, moving afar from the root corresponds to following a sequence of elements whose types form an IFP.
	We claim that all instances of $\cstyle{C}$ or of critical types in a tree can be found a depth at most $2^\sizeof{\tbox}$.
	Assume by contradiction that we can follow a branch whose elements have types $t_0, \dots, t_n$ with $n \geq 2^\sizeof{\tbox}$ and either $\cstyle{C} \in t_n$ or $t_n$ is a critical type for a safe concept $\cstyle{C}$.
	Then there must exists two integers $0 \leq i < j \leq n$ s.t.\ $t_i = t_j$.
	The IFP with types $t_i, t_{i+1}, \dots, t_j$ is a $\cstyle{C}$-generating cycle.
	Since $\cstyle{C}$ is safe in $\tbox$, we have $\tbox \models \tbox_\cstyle{C}$ and thus this $\cstyle{C}$-generating cycle is reversed in $\tbox$, and in particular $t_i \bidep{r} t_{i+1}$.
	This contradicts the element with critical type $t_{i + 1}$ being introduced by an application of $\crule{1}$.
	Since each tree has also a branching degree of at most $2^\sizeof{\tbox}$, we obtain that each tree contains at most $2^{\sizeof{\tbox} \times 2^\sizeof{\tbox}}$ instances of safe concepts and critical types.
	It remains to argue that there are finitely many roots each time $\crule{1}^\infty$ is applied on $\ils{\kb}{n}$ for some $n \geq 0$.
	This is trivial for $\ils{\kb}{0}$ which is finite.
	When applying $\crule{1}^\infty$ to $\ils{\kb}{n+1}$, notice that all elements from $\domainof{\ils{\kb}{n}}$ are already saturated for $\crule{1}$ and therefore the only possible roots for new applications of $\crule{1}$ are the elements just introduced by the latest application of $\crule{2}$, which, as we have seen above, introduces at most $2^\sizeof{\tbox} \times M_{\ils{\kb}{n}}$ fresh elements.
	Since those latter element all have some critical type, we bound the above quantity by $M_{\ils{\kb}{n+1}}$, so that the application of $\crule{1}^\infty$ on $M_{\ils{\kb}{n+1}}$ always introduces at most $2^{\sizeof{\tbox} \times 2^\sizeof{\tbox}} \times M_{\ils{\kb}{n+1}}$ new instances.
	To uniformize this with the case $n = 0$, we further use the more brutal bound $f(n) = 2^{\sizeof{\tbox} \times 2^\sizeof{\tbox}} \times (M_{\ils{\kb}{n}} + \sizeof{\domainof{\ils{\kb}{0}}})$, so that the number of instances of safe concepts and critical types introduced by $\crule{1}^\infty$ in $\crule{1}^\infty(\ils{\kb}{n+1})$ is at most $f(n)$.

	Overall, we obtained that for all $n \geq 0$:
	\[
		M_{\ils{\kb}{n+1}} \leq M_{\ils{\kb}{n}} + f(n + 1) + 2^\sizeof{\tbox} \times M_{\ils{\kb}{n}}.
	\]
	Coupled with the bound $N_\mathsf{steps}$ on the number of steps in the algorithm, it is not hard to see that this relation yields a finite number of instances of safe concepts and critical types in $\ils{\kb}{}$.
	Furthermore, this bound is polynomial (actually, even linear) w.r.t.\ $\sizeof{\abox}$, which will further play a role in our complexity results.

\end{proof}

Now, to finish the proof of Theorem~\ref{theorem-shape-concept-elifbot}, we build a model with exactly one more instance of $\cstyle{C}$ than in the model $\ils{\kb_\cstyle{C}}{}$ obtained by the above procedure on $\kb_\cstyle{C} := (\tbox_\cstyle{C}, \abox)$.
To do so, we essentially relaunch this procedure on a simpler KB $\kb' = (\tbox_\cstyle{C}, \{ \cstyle{C}(\istyle{a})\})$, where $\istyle{a}$ is a fresh individual name,
 and then form the disjoint union of $\ils{\kb'}{}$ with $\ils{\kb_\cstyle{C}}{}$.
However, this approach is too naive as the model $\ils{\kb'}{}$ might contain several instances of the concept $\cstyle{C}$, due to the initial elements $d_t$ for each type $t \in \types(\tbox)$.
This cannot easily be solved by identifying the respective $d_t$ elements from $\ils{\kb'}{}$ and $\ils{\kb_\cstyle{C}}{}$, as such an operation may violate some functionality constraints.

Instead, we produce an incomplete version $\preils{\kb'}{}$ of $\ils{\kb'}{}$ in which elements $d_t$ are absent and applications of rule $\crule{3}$ are ignored.
Formally, for an $\ELIFbot$ KB $\kb$, the interpretation $\preils{\kb}{0}$ is defined as $\ils{\kb}{0}$, but without the $d_t$ elements, and we further define, for all $n \geq 0$:
\begin{center}
\(
\preils{\kb}{n+1} = \crule{2}(\crule{1}^{\infty}(\preils{\kb}{n}))
\qquad\textrm{ and }\qquad
\preils{\kb}{} = \bigcup_{n = 0}^\infty \preils{\kb}{n}.
\)
\end{center}
The resulting interpretation $\preils{\kb}{}$ is in general not a model of $\kb$ due to the non-applied $\crule{3}$ rules.
It is however possible to reuse $d_t$ elements of another model, \emph{e.g.}\ those from $\ils{\kb}{}$.

\begin{lemmarep}
	\label{lemma:ils-and-preils}
	If $\kb = (\tbox, \abox)$ is a satisfiable $\ELIFbot$ KB and $\cstyle{C}$ is safe in $\tbox$, then $\J = \crule{3}(\ils{\kb}{} \cup \preils{\kb'}{})$ is a model of $\kb$, where $\kb' = (\tbox, \{ \cstyle{C}(\istyle{a})\})$ with $\istyle{a} \notin \mathsf{Ind}(\abox)$.
	Furthermore,
 $\cstyle{C}^{\J} = \cstyle{C}^{\I_\kb} \cup \{ \istyle{a} \}$.
\end{lemmarep}

\begin{proof}
	Proving that $\crule{3}(\ils{\kb}{} \cup \preils{\kb'}{})$ is a model of $\kb$ follows the same line as for proving that $\ils{\kb}{}$ is.
	Notably, it can be verified that all the intermediate $\preils{\kb}{n}$ satisfies invariants from Lemma~\ref{lemma:imported} (it suffices to verify it on $\preils{\kb}{0}$).
	This entails termination of the procedure generating $\preils{\kb}{}$, with the same bound on the number of steps, and satisfaction of axioms with shapes $K \sqsubseteq \cstyle{A}$, $\exists \rstyle{r}. K \sqsubseteq K'$ and $K \sqsubseteq\ \leq 1\ \rstyle{r}. K'$.
	Joint with $\ils{\kb}{}$ being a model of $\kb$, it is readily checked that $\crule{3}(\ils{\kb}{} \cup \preils{\kb'}{})$ is also a model of these three shapes of axioms.
	For axioms with shape $K \sqsubseteq \exists \rstyle{r}.K'$, it suffices to notice that the final application of $\crule{3}$ provides the desired witnesses for the elements of $\preils{\kb}{}$ which were violating these axioms.

	We now give more details regarding the proof of $\cstyle{C}^{\preils{\kb'}{}} = \{ \istyle{a} \}$.

	In the body of the paper, it is already made clear that another instance of $\cstyle{C}$ cannot possibly be obtained by an application of $\crule{1}$ as it would imply existence of an underlying $\cstyle{C}$-generating cycle. The latter, whose reversion is entailed by $\tbox$, contradicts the prerequisite to apply $\crule{1}$ (a similar argument has also been used to study the number of instances of safe concepts introduced by rule $\crule{1^\infty}$ in the proof of Lemma~\ref{lemma:ils}).
	For the same reason, no critical type can be generated by an application of $\crule{1}$, since it could then be completed into a similar $\cstyle{C}$-generating cycle.
	\emph{A fortiori}, no type from the type class of $\type_{\kb'}(\istyle{a})$ can be generated by $\crule{1}$ ($\star$).

	Concerning applications of $\crule{2}$, assume by contradiction that an application of $\crule{2}$ with type class $P$ introduces an instance of $\cstyle{C}$.
	Using Lemma~\ref{lemma:cycling-types-are-equal-types}, we know that $P$ has to be exactly the type class of $\type_{\kb'}(\istyle{a})$.
	Recall now the graph $G$ that has been used in the proof of termination in Lemma~\ref{lemma:ils} and that describes the dependencies between successive introductions of elements along the construction of $\preils{\kb}{}$ (resp.\ of $\ils{\kb}{}$ in proof of Lemma~\ref{lemma:ils}).
	We reuse the argument that derives a contradiction if twice the same type class is used to define the edges of a same branch in the graph $G$.
	Joint with $(\star)$ and fact that $\domainof{\preils{\kb}{0}} = \{ \istyle{a} \}$, we obtain that $P$ can only be used in applications of $\crule{2}$ exactly once.
	Now, using $(\star)$ again, when $\crule{2}$ triggers with $P$ on the current $\I$, in the construction of sets $\Delta_s$, we have $X^\I_{\lambda, 1} = \{ \istyle{a} \}$ (resp.\ $X^\I_{\lambda, 2} = \{ \istyle{a} \}$) if $\lambda = \type_{\kb'}(\istyle{a}) \bidep{r} s'$ (resp.\ if $\lambda = s' \bidep{r} \type_{\kb'}(\istyle{a})$), and $X^\I_{\lambda, 1} = \emptyset$ (resp.\ $X^\I_{\lambda, 2} = \emptyset$) otherwise.
	It is thus possible to set $\Delta_s = \emptyset$ if $s = \type_{\kb'}(\istyle{a})$ and $\Delta_s = \{ a_s \}$ where $a_s$ is a fresh element for the other $s \in P$.
	Subsequent bijections are trivial since every $X^\I_{\lambda, i} \cup \Delta_s$ are now singletons.
	This clearly minimizes the number of introduced elements and proves that the only application of $\crule{2}$ with the type class of $\type_{\kb'}(\istyle{a})$ does not introduce new instance of $\type_{\kb'}(\istyle{a})$, thus of $\cstyle{C}$.
\end{proof}

This concludes the proof of Theorem~\ref{theorem-shape-concept-elifbot} as we obtain two models $\ils{\kb_\cstyle{C}}{}$ and $\crule{3}(\ils{\kb_\cstyle{C}}{} \cup \preils{(\tbox_\cstyle{C}, \{ \cstyle{C}(\istyle{a})\})}{})$ with respectively $n$ and $n+1$ instances of $\cstyle{C}$, both $n$ and $n+1$ being finite.

\begin{toappendix}
Intuitively, this comes from the observation that, in this interpretation, all types are on some IFP starting from $\type_{\kb'}(\istyle{a})$.
Indeed, rule $\crule{1}$, resp.\ $\crule{2}$, only introduces a type $t'$ if is connected to an already instantiated type $t$ by $t \rightdep{r} t'$, resp.\ by $t \bidep{r_0} t_1 \bidep{r_{1}} \dots t_k \bidep{r_{k}} t'$, and $\type_{\kb'}(\istyle{a})$ is the only initial type.
Now, since $\type_{\kb'}(\istyle{a})$ is the type induced by $\cstyle{C}$, any instance of $\cstyle{C}$ in $\preils{\kb'}{}$ has a type $s$ containing $\type_{\kb'}(\istyle{a})$.
The following technical lemma shows that this situation enforces $s = \type_{\kb'}(\istyle{a})$.

\begin{lemmarep}
	\label{lemma:cycling-types-are-equal-types}
	Let $\tbox$ be a satisfiable $\ELIFbot$ TBox, $\cstyle{C}$ a safe concept in $\tbox$ and
	$t_0 \in \types(\tbox)$ a type containing $\cstyle{C}$.
	If there exists an IFP
	$t_0$, $\rstyle{r_1}$, $K_1$, $\dots$, $\rstyle{r_n}$, $K_n$ with $t_0 \subseteq K_n$ (where $t_0$ is seen as the conjunction of its concepts),
	then $t_0 = K_n$.
\end{lemmarep}

\begin{proof}
	Since $\cstyle{C} \in t_0 \subseteq K_n$, notice $K_n$, $\rstyle{r_1}$, $K_1$, $\dots$, $\rstyle{r_n}$, $K_n$ is a $\cstyle{C}$-generating cycle that has thus been reversed in $\tbox$.
	In particular, we have $(i)$ $\tbox \models K_1 \sqsubseteq \exists \rstyle{r_1^-}.K_n$.
	We now prove that $K_n \subseteq t_0$.
	Since $t_0$ is a type for $\tbox$, there exists a model $\I$ of $\tbox$ and $e_0 \in \domainof{\I}$ whose type is exactly $t_0$.
	From $t_0$, $\rstyle{r_1}$, $K_1$ initiating the inverse functional sequence, we have $(ii)$ $\tbox \models t_0 \sqsubseteq \exists \rstyle{r_1}.K_1$ and $(iii)$ $\tbox \models K_1 \sqsubseteq \leq 1\ \rstyle{r_1^-}.t_0$.
	Thus, from $(ii)$, there exists $e_1 \in K_1^\I$ such that $(e_0, e_1) \in \rstyle{r_1}^\I$.
	Similarly, from $(i)$, there exists $d \in K_n^\I$ such that $(d, e_1) \in \rstyle{r_1}^\I$.
	But now, from $t_0 \subseteq K_n$ and $(iii)$, we must have $d = e_0$.
	Thus $e_0 \in K_n^\I$ and since $e_0$ has type $t_0$, we obtain $t_0 \subseteq K_n$.
\end{proof}

Consequently, all instances of $\cstyle{C}$ in $\preils{\kb'}{}$ have type $\type_{\kb'}(\istyle{a})$.
It follows that $\crule{1}^\infty$ doesn't introduce any instance of $\type_{\kb'}(\istyle{a})$, thus of $\cstyle{C}$, and that the type class of $\type_{\kb'}(\istyle{a})$ is chosen exactly once along applications of $\crule{2}$.
When chosen, the minimality condition on the number of elements introduced by $\crule{2}$ guarantees no other instance of $\type_{\kb'}(\istyle{a})$, thus of $\cstyle{C}$, is created.

\end{toappendix}

\subsection{The case of $\dllitef$}
\label{subsection:concept-dllitef}

\begin{toappendix}
	\subsection*{Proofs for Section~\ref{subsection:concept-dllitef} (The case of $\dllitef$)}
\end{toappendix}

We now build upon the above technique to obtain the following complete characterization for $\dllitef$ KBs.

\begin{theorem}
	\label{theorem-shape-concept-dllitef}
	A non-trivial subset of $\ninfty$ is $\dllitef$-concept realizable iff it has shape $\{ \infty \}$, $\{ 0, \infty \}$, $\{ 0 \} \cup \llbracket M, \infty \rrbracket$ or $\llbracket M, \infty \rrbracket$ for some $M \in \NN$.
\end{theorem}

The `if' direction follows notably from Examples~\ref{example-concept-elif-infinity} and \ref{example-dllitef-zero-infinity}.
The converse is a consequence of the following lemma.

\begin{lemmarep}
	\label{lemma-concept-dllitef-plus-one}
	Let $\kb$ be a $\dllitef$ KB and $\cstyle{C}$ a concept name.
	If there exists a model $\I$ of $\kb$ with $1 \leq \sizeof{\cstyle{C}^\I} < \infty$,
	then there exists a model $\J$ of $\kb$ with $\sizeof{\cstyle{C}^{\J}} = \sizeof{\cstyle{C}^\I} + 1$.
\end{lemmarep}

\begin{proof}
	Assume there exists a model $\I$ of $\kb = (\tbox, \abox)$ with $1 \leq \sizeof{\cstyle{C}^\I} < \infty$.
	In particular, from Lemma~\ref{lemma:cycle-reversion}, $\kb_\cstyle{C}$ is satisfiable.
	Consider the interpretation $\preils{\kb'}{}$, with $\kb' = (\tbox_{\cstyle{C}}, \{ \cstyle{C}(\istyle{a} ) \})$.
	Lemma~\ref{lemma:ils-and-preils} ensures $\preils{\kb'}{}$ has exactly one instance of $\cstyle{C}$.
	Furthermore, since $\I$ also has at least one instance of $\cstyle{C}$, all missing role witnesses for elements in $\preils{\kb'}{}$ can be found in $\I$.
	Formally, for each $d \in \domainof{\preils{\kb'}{}}$ and each $\rstyle{r} \in \posroles$ s.t.\ $\tbox \models \type_\I(d) \sqsubseteq \exists\rstyle{r}.\top$ and $d \notin (\exists \rstyle{r}.\top)^\I$, we can find an element $d_\rstyle{r} \in (\exists \rstyle{r^-})^\I$.
	We let $\J$ be the disjoint union of $\I$ and $\preils{\kb'}{}$ in which we add the edge $(d, d_{\rstyle{r}}) \in \rstyle{r}^\I$ for each pair $(d, \rstyle{r})$ as above.
	It is easily verified that $\J$ is the desired model of $\kb$.
\end{proof}

\begin{proofsketch}
	As for $\ELIFbot$, we consider the interpretation $\J_{\kb'}$ where $\kb' = (\tbox_{\cstyle{C}}, \{ \cstyle{C}(\istyle{a} ) \})$.
	We can then connect $\J_{\kb'}$ to any model $\I$ of $\kb$ in which there is at least one instance of $\cstyle{C}$ by finding appropriate witnesses for the pending roles.
	This is possible as $\dllitef$ does not support qualified existential restrictions, thus imposing very little constraints on the types of the required witnesses.
\end{proofsketch}

\section{Complexity of computing spectra}
\label{sec:complexity}

\begin{table*}
	\centering
	\begin{tabular}{c@{\qquad}c@{\qquad}c@{\qquad}c@{\qquad}cc@{\qquad}c}
		&
		&
		& $\EL$,
		& $\ALC$, $\ELIF$,
		&
		&
		\\
		& $\dllitecore$
		& $\dllitef$
		& $\ELIbot$, $\ELFbot$
		& $\ELIFbot$, $\ALCI$, $\ALCF^*$
		& $\ALCF$
		& $\ALCIF$
		\\[1mm]\midrule
		Concept
		& in $\FP$
		& in $\FPwithlogcalls$
		& in $\FPwithconstantcalls$~~~~~
		& $\FPwithlogcalls$\complete
		& $\FPwithlogcalls$\complete
		& $\FPwithlogcalls$\hard
		\\[1mm]
		Role
		& in $\FP$
		& in $\FPwithlogcalls$
		& ~~~~\,$\FPwithlogcalls$\complete
		& $\FPwithlogcalls$\complete
		& $\FPwithlogcalls$\hard
		& $\FPwithlogcalls$\hard
	\end{tabular}
	\caption{Worst-case complexity of $\mathsf{Spectrum}(q, \tbox)$ if $q$ is a concept (resp.\ role) cardinality query and $\tbox$ is expressed in one of the DLs on the first row. \hard stands for -hard and \complete for -complete.}
	\label{table:concept-complexity}
\end{table*}

We now tackle the problem of computing the proposed effective representation of spectra, helped by our knowledge of their possible shapes.
We focus on data complexity: for 
a fixed cardinality query $q$ and a fixed TBox $\tbox$, we study the complexity
of the $\mathsf{Spectrum}(q, \tbox)$ problem, which, given an ABox $\abox$ as input, computes the output of Problem~1 from Section~\ref{section-prelims}.

We use functional complexity classes: $\FP$ is the class of functions computable in polynomial time by a Turing machine; $\FPwithconstantcalls$ is $\FP$ with $\mathcal{O}(1)$ many queries to an $\NP$ oracle; and $\FPwithlogcalls$ is allowed $\mathcal{O}(\log(n))$ many queries to $\NP$ where $n$ is the size of the input and $p$ a polynomial.
We refer to \cite{krentel-opt-88,DBLP:journals/tcs/JennerT95} for details and recall the following inclusions: 
\(
\FP \subseteq \FPwithconstantcalls \subseteq \FPwithlogcalls.
\)

Our complexity results are summarized in Table~\ref{table:concept-complexity}.
We start with upper bounds.

\begin{theoremrep}
	\label{theorem:complexity-concept-alc}
	\label{theorem:complexity-concepts-upper-bounds}
	$\mathsf{Spectrum}(q_\cstyle{C}, \tbox)$ is in:
	\begin{itemize}
		\item $\FPwithlogcalls$ if $\tbox$ is in $\ELIFbot$, $\ALCI$ or $\ALCF$.
		\item $\FPwithconstantcalls$ if $\tbox$ is in $\ELIbot$ or $\ELFbot$.
		\item $\FP$ if $\tbox$ is in $\dllitecore$.
	\end{itemize}
\end{theoremrep}
\begin{proof}
	Proof of the theorem is given in Section~\ref{sec:upper-bounds-proof} below.
\end{proof}

The backbone of the above complexity results relies on the fact that, in all the concerned cases, the possible spectra are of the form $S \cup \llbracket M , \infty \rrbracket$ where $S \subseteq \llbracket 0, M \llbracket$ and $M$ is either $\infty$ or $M$ is polynomial w.r.t.\ data complexity.

\begin{lemmarep}
	\label{lemma:polynomial-bound-for-conccepts}
		Let $\tbox$ be an $\ELIFbot$-, $\ALCI$-, or $\ALCF$-TBox and $\cstyle{C}$ a concept name.
		There exists a polynomial $p(x)$, with coefficients computable from $\tbox$ and $\cstyle{C}$, such that $p(x) \geq 1$ for every $x \geq 0$
		and, for every KB $\kb = (\tbox, \abox)$, either
		\begin{itemize}
			\item $p(|\abox|) \notin  \spectrum{\kb}{q_\cstyle{C}}$ and $\spectrum{\kb}{q_\cstyle{C}} \subseteq \{0,\infty\}$; or
			\item $p(|\abox|) \in \spectrum{\kb}{q_\cstyle{C}}$ and $\spectrum{\kb}{q_\cstyle{C}} = S \cup \llbracket M, \infty\rrbracket$
			where $S \subseteq \llbracket 0, M \rrbracket$ and $M < p(|\abox|)$.
		\end{itemize}
		Moreover, in the latter case,
		$M {\leq} \min(\spectrum{\kb}{q_\cstyle{C}}) + p(0)$.
\end{lemmarep}

\begin{proof}

    We start with the following observation.
	\begin{itemize}
		\item[$\clubsuit$] One can compute a polynomial $p'$, with coefficients depending only on $q_\cstyle{C}$  and $\tbox$, such that for every ABox $\abox$ if there is $b \in \spectrum{\kb}{q_{\cstyle{C}}}$ such that $0 < b < \infty$ then there is a number $a_\abox \in \mathbb{N}$ such that $a_\abox,a_\abox{+}1 \in \spectrum{\kb}{q_{\cstyle{C}}}$ and $a_\abox < p(|\abox|)$.
	\end{itemize}

	For the case of $\ELIF_\bot$ this claim follows immediately from Lemma~\ref{lemma-concept-elif-plus-one},
	and the fact that, as shown in the proof of the lemma, the size of $C^{\J}$ can be bound by a polynomial $p'_E$ so that $|C^{\J}| < p'_E(|\abox|)$.
    In this case we take $a_\abox = |C^{\J}|$.

	In the cases of $\ALCI$ and $\ALCF$ we first observe that
	by Theorem 2 in \cite{sanja21:bounded-predicates} we can compute a polynomial $p'_A$ so that there is a model $\I$ of $\kb$
	with $a_{\abox} = |C^{\I}| < p'_A(|\abox|)$. Then, by Lemma~\ref{lemma-concept-alci-alcf-plus-one}, we conclude that  also $|C^{\I}|+1 \in \spectrum{\kb}{q_{\cstyle{C}}}$.

    Let polynomial $p''_E$ be obtained from polynomial $p'_E$ by changing coefficients of $p'_E$ to their absolute values and, then, adding $1$ to the free coefficient. Let $p''_A$ be a polynomial obtained from $p'_A$ in the same manner. Let $p' = p''_E + p''_A$.

    Clearly, polynomial $p'$ can be computed and its coefficients depend only on $q_\cstyle{C}$  and $\tbox$. Moreover, $p'(x) > 0$ for $x \geq 0$
    as all its coefficients are non-negative and the free coefficient is at least $2$. Finally, since $p'(x) > p'_E(x) + p'_A(x)$ for $x\geq 0$, $a_{\abox} < p'(|\abox|)$ for all ABoxes $\abox$ such that there is $b \in \spectrum{\kb}{q_{\cstyle{C}}} \setminus \{0,\infty\}$.

    Polynomial $p'$ is almost the polynomial desired in the lemma.
    Unfortunately, the value $a_{\abox}$ claimed in $\clubsuit$ depends on the ABox the KB is equipped with.
    We now strengthen $\clubsuit$ by replacing this dependence
    by the dependence on the size of the ABox.

    \begin{itemize}
        \item[$\diamondsuit$] One can compute a polynomial $p$, with coefficients depending only on $q_\cstyle{C}$  and $\tbox$, such that for every $i \in \mathbb{N}$ there is a number $0 \leq a < p(|\abox|)$ such that for every ABox $\abox$ of size $|\abox| = i$: if there is $b \in \spectrum{\kb}{q_{\cstyle{C}}}$ such that $0 < b < \infty$ then for all $c \geq a$ $c \in \spectrum{\kb}{q_{\cstyle{C}}}$.
    \end{itemize}

	To show $\diamondsuit$, it is enough to observe the following.
	\begin{itemize}
		\item[$\spadesuit$] If $a_\abox, a_\abox{+}1 \in \spectrum{\kb}{q_{\cstyle{C}}}$ then $\llbracket a_\abox(a_\abox{+}1), \infty \llbracket \subseteq \spectrum{\kb}{q_{\cstyle{C}}}$.
	\end{itemize}

	Indeed, if $\llbracket a_\abox(a_\abox+1), \infty \llbracket\ \subseteq \spectrum{\kb}{q_{\cstyle{C}}}$ then $b \in \spectrum{\kb}{q_{\cstyle{C}}}$ for all $b \geq a_\abox(a_\abox+1)$.
    Thus, it is enough to take as $a$ the maximum among $a_\abox(a_\abox+1)$ for ABoxes of size $i$ and as $p$
    the polynomial $x \mapsto p'(x)(p'(x) + 1)  + 2$.

	To prove $\spadesuit$ observe that every number $n$ from the set $\llbracket a_\abox(a_\abox+1), \infty \llbracket$ is of form $n = a_\abox(a_\abox+i) + ja_\abox + i$,
	where $j \geq 0$ and $0 \leq i  < a_\abox$.
	Thus, $n = a_\abox(a_\abox+1) + ja_\abox + i =  ja_\abox + a_\abox(a_\abox+1) - ia_\abox + i + ia_\abox= ja_\abox + a_\abox(a_\abox+1-i) + i(a_\abox+1)$.
	This implies that $n \in \spectrum{\cstyle{C}}{\kb}$, as $\spectrum{\kb}{q_{\cstyle{C}}}$ is closed under addition. This proves $\spadesuit$.

    We now claim that $p$ in $\diamondsuit$ is the polynomial claimed in Lemma~\ref{lemma:polynomial-bound-for-conccepts}.
    Clearly, coefficients of $p$ are computable and depend only on $\cstyle{C}$ and $\tbox$. It is also easy to verify that $p(x) \geq 1$
    for $x \geq 0$.
    We now verify that for every KB $\kb = (\tbox, \abox)$, either
    \begin{itemize}
        \item $p(|\abox|) \notin  \spectrum{\kb}{q_\cstyle{C}}$ and $\spectrum{\kb}{q_\cstyle{C}} \subseteq \{0,\infty\}$; or
        \item $p(|\abox|) \in \spectrum{\kb}{q_\cstyle{C}}$ and $\spectrum{\kb}{q_\cstyle{C}} = S \cup \llbracket M, \infty\rrbracket$
        where $S \subseteq \llbracket 0, M \rrbracket$ and $M < p(|\abox|)$.
    \end{itemize}
    Let us fix an ABox $\abox$ and let $\kb = (\tbox, \abox)$.
    If $\spectrum{\kb}{q_\cstyle{C}} \subseteq \{0,\infty\}$
    then, of course, $p(|\abox|) \notin \spectrum{\kb}{q_\cstyle{C}}$
    as $0 < p(|\abox|) < \infty$.
    On the other hand, if $\spectrum{\kb}{q_\cstyle{C}} \not \subseteq \{0,\infty\}$ then there is $0 < b < \infty$
    such that $b \in \spectrum{\kb}{q_\cstyle{C}}$. Thus, by $\diamondsuit$ there is $a < p(|A|)$ such that $c \in \spectrum{\kb}{q_\cstyle{C}}$ for all $c > a$.
    Hence, $p(|\abox|) \in \spectrum{\kb}{q_\cstyle{C}}$
    and $\spectrum{\kb}{q_\cstyle{C}} = S \cup \llbracket M, \infty\rrbracket$ where $S = \spectrum{\kb}{q_\cstyle{C}} \cap \llbracket 0, p(|\abox|) \rrbracket$ and $M = p(|\abox|)$.

	\smallskip
	Now we prove the last line of the lemma.
	Let $\kb_\emptyset = (\tbox, \emptyset)$.
	Since for every model $\I$ of $\kb$, $\I$ is also a model of $\kb_\emptyset$,
	$\spectrum{\kb_{\emptyset}}{q_{\cstyle{C}}} = V_0 \cup \llbracket M_0, \infty \llbracket$ with $M_0 \leq p(|\emptyset|) = p(0)$.

	Let $m = \min(\spectrum{\kb_{\emptyset}}{q_{\cstyle{C}}})$
	and $\I$ be a model of $\kb$ such that $m = |\cstyle{C}^{\I}|$ and $\I_\emptyset$ be a model of $\kb_\emptyset$ such that $\domainof{\I} \cap \domainof{\I_\emptyset} = \emptyset$.
	Then $\I \cup \I_\emptyset \models \kb$ and, thus, $m + |\cstyle{C}^{\I_\emptyset}| \in \spectrum{\kb}{q}$.
	From this we can infer that $m + \spectrum{\kb_{\emptyset}}{q_{\cstyle{C}}} \subseteq \spectrum{\kb}{q_{\cstyle{C}}}$.
	Thus, $\llbracket m + M_0, \infty \llbracket = m + \llbracket M_0, \infty \llbracket \subseteq \spectrum{\kb}{q_{\cstyle{C}}}$.
	In particular, this implies that $M \leq m + M_0 \leq m + p(0)$, proving that $M -  \min(\spectrum{\kb_{\emptyset}}{q_{\cstyle{C}}})\leq p(0)$.
\end{proof}

This allows us to present a simple and uniform description of an algorithm computing spectra.
Let $\kb = (\tbox, \abox)$ be the input KB, $q$ the input cardinality query, and $p(x)$ the polynomial from Lemma~\ref{lemma:polynomial-bound-for-conccepts}.

First the algorithm tests whether $p(|\abox|) \in \spectrum{\kb}{q_{\cstyle{C}}}$.
If not then we are in the first case of Lemma~\ref{lemma:polynomial-bound-for-conccepts}, that is
$\spectrum{\kb}{q_{\cstyle{C}}} \subseteq \{0,\infty\}$.
The algorithm now tests whether $\kb$ is satisfiable: if not, it returns $\emptyset$.
Then, it checks whether $q_{\cstyle{C}}$ is satisfiable with respect to $\kb$: if not, it returns $\{ 0 \}$ (represented in the sense of Problem~\ref{problem:compute-spectrum} by the triple $(\emptyset, 0, 0)$).
Finally, the algorithm checks whether $0 \in \spectrum{\kb}{q_{\cstyle{C}}}$.
If yes, it returns $\{0,\infty\}$ and $\{\infty\}$ otherwise (respectively represented
in the sense of Problem~\ref{problem:compute-spectrum}
by triples $(\{ 0 \}, \infty, 0)$ and $(\emptyset, \infty, 0)$).

If $p(|\abox|) \in \spectrum{\kb}{q_{\cstyle{C}}}$ then we are in the second case of Lemma~\ref{lemma:polynomial-bound-for-conccepts}, that is $M < p(|\abox|)$ and $S \subseteq \llbracket 0 , M \rrbracket$.
The algorithm first performs a binary search on the interval $\llbracket 0, p(|\abox|) \rrbracket$
to find $\min(\spectrum{\kb}{q_{\cstyle{C}}})$. In each step of the search, the
algorithm is given a number $n \in \llbracket 0, p(|\abox|) \rrbracket$ and
performs a minimality test that verifies whether $n > \min(\spectrum{\kb}{q_{\cstyle{C}}})$.
Note that once the value of $\min(\spectrum{\kb}{q_{\cstyle{C}}})$ is found, the additional remark in Lemma~\ref{lemma:polynomial-bound-for-conccepts} guarantees that $M \leq \min(\spectrum{\kb}{q_{\cstyle{C}}}) + p(0)$ and thus $S \subseteq \llbracket \min(\spectrum{\kb}{q_{\cstyle{C}}}), \min(\spectrum{\kb}{q_{\cstyle{C}}}) + p(0) \rrbracket$.
The algorithm performs membership tests on this latter interval to compute $M$ and $S$.
Finally, the algorithm returns the set $S \cup \llbracket M, \infty \rrbracket$ (represented in the sense of Problem~\ref{problem:compute-spectrum} by the triple $(S, M, 1)$).

The correctness of the algorithm follows directly from Lemma~\ref{lemma:polynomial-bound-for-conccepts}.
The exact computational complexity depends on the number and cost of satisfiability checks, membership tests, and minimality tests.
For instance, for $\ELIFbot$ the two initial satisfiability checks can by performed by $\NP$ oracles \cite{birte2008cqa-in-SHIQ}.
Similarly, the membership tests can be resolved by an $\NP$ oracle as they can be seen as instances of \emph{closed predicates} problem, see \cite{sanja21:bounded-predicates}.
The minimality tests can be performed by an $\NP$ oracle that guesses $n'{<}n$ and performs a membership~test.
Since the algorithm uses logarithmically many minimality tests to compute $\min(\spectrum{\kb}{q_{\cstyle{C}}})$ and no more than $p(0){+}1$ membership tests to compute $S$, the desired upper bound holds.

In $\dllitecore$, we can perform the satisfiability checks, the membership tests, and the minimality tests in polynomial time, see~\cite{DBLP:conf/kr/CalvaneseGLLR06} and \cite{maniere:thesis} [Theorem~51] respectively, resulting in overall polynomial running time.

The following theorem provides two lower bounds, notably establishing $\FPwithlogcalls$-completeness in several cases.
\begin{theoremrep}
	\label{theorem:complexity-lower-bounds}
	There exists an $\ELIF$ (resp.\ $\ALC$) TBox $\tbox$ such that $\mathsf{Spectrum}(q_\cstyle{C}, \tbox)$ is $\FPwithlogcalls$-hard.
\end{theoremrep}

\begin{proof}
	{\bf Concepts on $\ALC$.}
	We finish the proof presented in the main part of the paper.
	Recall we consider the concept cardinality query $q_\cstyle{C}$ and the $\ALC$ TBox $\tbox := \{ \lnot \cstyle{C} \sqsubseteq \forall \rstyle{r}. \cstyle{C} \}$.
	Given a graph $G=\langle V, E \rangle$, we construct an ABox $\abox$ consisting in assertions $\rstyle{r}(\istyle{u},\istyle{v})$ for every  $\{ u,v \} \in E$.
	We want to prove that $k$ is the maximal size of an independent set in $G$ iff $\spectrum{(\tbox, \abox)}{q_\cstyle{C}} = \llbracket \sizeof{V} - k, \infty \rrbracket$.
	First, note that every model $\I$ of $\kb := (\tbox, \abox)$ describes an independent set $U_\I$ of $G$ given by $(\lnot \cstyle{C})^\I \cap V$.
	Consequently, in every model of $\kb$, we have ${V} \setminus {U_\I} \subseteq \cstyle{C^\I}$, that is $\sizeof{C^\I} \geq \sizeof{V} - \sizeof{U_\I}$ ($\star$).
	Therefore, for the $(\Rightarrow)$ direction, we assume that $k$ is the maximal size of an independent set in $G$, and we thus get $\sizeof{C^\I} \geq \sizeof{V} - \sizeof{U_\I} \geq \sizeof{V} - k$.
	Furthermore, $\sizeof{V} - k$ can easily be realized in a model $\I_U$ describing exactly an independent set $U$ of $G$ with size $k$, and any integer greater than $\sizeof{V} - k$ can further be realized by adding as many extra instances of $\cstyle{C}$ outside of $V$.
	This proves $\spectrum{(\tbox, \abox)}{q_\cstyle{C}} = \llbracket \sizeof{V} - k, \infty \rrbracket$.
	For the $(\Leftarrow)$ direction now, assume that $k$ is not the maximal size of an independent set in $G$.
	If there is a bigger independent set $U'$, say with size $k' > k$, then the model $\I_{U'}$ realizing it satisfies $\sizeof{\cstyle{C}^{\I_{U'}}} = \sizeof{V} - k' < \sizeof{V} - k$,  thus $\sizeof{V} - k' \in \spectrum{(\tbox, \abox)}{q_\cstyle{C}}$, and therefore $\spectrum{(\tbox, \abox)}{q_\cstyle{C}} \neq \llbracket \sizeof{V} - k, \infty \rrbracket$.
	Otherwise, every independent set of $G$ has a size smaller that $k$, in which case equation ($\star$) from above concludes as it yields $\sizeof{V} - k \notin \spectrum{(\tbox, \abox)}{q_\cstyle{C}}$ and therefore $\spectrum{(\tbox, \abox)}{q_\cstyle{C}} \neq \llbracket \sizeof{V} - k, \infty \rrbracket$.
	\medskip

	{\bf Concepts on $\ELIF$.}
	Consider the concept cardinality query $q_\cstyle{C}$ and the $\ELIF$ TBox $\tbox$ containing the four following axioms:
	\[
	\cstyle{Vertex} \sqsubseteq \exists \rstyle{s}.\cstyle{C}
	\qquad
	\top \sqsubseteq\ \leq 1\ \rstyle{s^-}.\top
	\qquad
	\exists \rstyle{t}.(\cstyle{C} \sqcap \exists \rstyle{s}.\cstyle{Indep}) \sqsubseteq \cstyle{C}.
	\]
	Now, given a graph $G = (V, E)$, we construct an ABox $\abox$ containing facts:
	\begin{itemize}
		\item $\cstyle{Vertex}(v)$, $\cstyle{Indep}(i_v)$ and $\cstyle{C}(i_v)$ for each $v \in V$;
		\item $\rstyle{edge}(u, v)$ for each $\{ u, v \} \in E$;
		\item $\rstyle{t}(u, v)$ for each $u, v \in V$.
	\end{itemize}
	We let $\kb = (\tbox, \abox)$ and prove that $k$ is the maximal size of an independent set in $G$ iff $\spectrum{\kb}{q_\cstyle{C}} = \llbracket 2\sizeof{V} - k, \infty \rrbracket$, that is the triple $(S, M, \alpha) := (\emptyset, 2\sizeof{V} - k, 1)$ is an effective representation of $\spectrum{(\tbox, \abox)}{q_\cstyle{C}}$.

	The first axiom requires every element $v$ representing a vertex to have a $\rstyle{s}$ successor serving as a witness for the query concept $\cstyle{C}$.
	Note that, from the second axiom, all these witnesses must be distinct.
	Along with the auxiliary elements $i_v$, also instances of $\cstyle{C}$, this leads to, a priori, at least $2\sizeof{V}$ instances of $\cstyle{C}$ in every model of $\kb$.
	It is possible, however, for the witnesses to be chosen among the $i_v$'s, thus reducing the total number of instances by as many witnesses are chosen among those elements.
	However, due to the third axiom, this cannot possibly work for all vertices at once.
	Indeed, the third axiom guarantees that, if two adjacent vertices both chose a witness among the $i_u$'s, then all the elements $u$ will be instances of $\cstyle{C}$, enforcing at least $2\sizeof{V}$ instances of $\cstyle{C}$.
	It is now clear that we can safely use $k$ elements $i_u$'s as witnesses, leading to only $2\sizeof{V} - k$ instances of $\cstyle{C}$ iff there exists an independent set of size $k$ in $G$.
	The fact that the spectrum contains every integer above its minimum value is simply obtained by adding fresh disconnected instances of $\cstyle{C}$ to a model realizing the minimum value.

\end{proof}

We reduce from the problem of computing the maximal size of an independent set in a graph, known to be $\FP^{\NP[\log]}$-hard \cite{krentel-opt-88}.
We briefly sketch the proof for $\ALC$: consider the TBox $\tbox := \{ \lnot \cstyle{C} \sqsubseteq \forall \rstyle{r}. \cstyle{C} \}$.
Given a graph $G=\langle V, E \rangle$, we construct an ABox $\abox$ consisting in $\rstyle{r}({u},{v})$ for every  $\{ u,v \} \in E$.
Intuitively, $\lnot \cstyle{C}$ describes an independent set and we prove that $k$ is the maximal size of an independent set in $G$ iff $\spectrum{(\tbox, \abox)}{q_\cstyle{C}} = \llbracket \sizeof{V} - k, \infty \rrbracket$, that is the triple $(\emptyset, \sizeof{V} - k, 1)$ is our representation of $\spectrum{(\tbox, \abox)}{q_\cstyle{C}}$.

\section{The case of role cardinality queries}
\label{section:roles}

\begin{toappendix}

	Table~\ref{table:roles} summarizes our realizability results for role cardinality queries, echoing to Table~\ref{table:concepts} from the concept case.
	\begin{table*}
		\centering
		\begin{tabular}{ccccccccc}
			$\mathcal{L}$
			&
			&
			$\llbracket m, \infty \rrbracket$
			&
			$\emptyset$
			&
			$\{ 0 \}$
			&
			$\{ 0 \} \cup \llbracket m, \infty \rrbracket$
			&
			$\{ \infty \}$
			&
			$\{ 0, \infty \}$
			&
			$V \cup \{ \infty \}$
			\\ \midrule
			$\ALCIF$
			& \nos		& \checkmark & \checkmark & \checkmark & \checkmark & \checkmark & \checkmark & \checkmark
			\\
			$\ALCF$
			& \nos		& \checkmark & \checkmark & \checkmark & \checkmark & \nope & \nope & \checkmark
			\\
			$\ELIFbot$
			& \nope	& \checkmark & \checkmark & \checkmark & \checkmark & \checkmark & \checkmark & \nope
			\\
			$\dllitef$
			& \nos 		& \checkmark & \checkmark & \checkmark & \checkmark & \checkmark & \checkmark  & \nope
			\\
			$\ELbot$, $\ELFbot$, $\ALCI$, $\ALCF^*$
			& \nope 	& \checkmark & \checkmark & \checkmark & \checkmark & \nope & \nope & \nope
			\\
			$\dllitecore$
			& \nos 		& \checkmark & \checkmark & \checkmark & \checkmark & \nope & \nope  & \nope
			\\
			$\ELIF$
			& \nos 		& \checkmark & \checkmark & \nope & \nope & \checkmark & \nope & \nope
			\\
			$\ELF$
			& \nos 		& \checkmark & \checkmark & \nope & \nope & \nope & \nope & \nope
			\\
			$\EL$, $\ELI$
			& \nos 		& \checkmark & \nope & \nope & \nope & \nope & \nope & \nope
			\\
			\bottomrule
		\end{tabular}
		\caption{$\mathcal{L}$-role realizable subsets of $\ninfty$, where $m \in \NN$ and $V$ is any subsemigroup of $\ninfty$ containing a non-zero integer. \nos\ indicates no other shape is possible.}
		\label{table:roles}
	\end{table*}
	We also recall the simple semantical remark that connects cardinality of $\rstyle{r}^\I$ with those of $(\exists \rstyle{r})^\I$ and $(\exists \rstyle{r}^-)^\I$ in every interpretation $\I$.
	\begin{remark}
		\label{remark:role-bounded}
		In every interpretation $\I$, we have:
		\[
		\max\left(\sizeof{(\exists \rstyle{r})^\I}, \sizeof{(\exists \rstyle{r^-})^\I}\right)
		\leq
		\sizeof{\rstyle{r}^\I}
		\leq
		\sizeof{(\exists \rstyle{r})^\I} \cdot \sizeof{(\exists \rstyle{r^-})^\I}.
		\]
	\end{remark}

\end{toappendix}

In this section, we briefly mention the similar results we obtain regarding role cardinality queries, \emph{i.e.}\ CCQs with form $q_\rstyle{r} := \exists z_1 \exists z_2\ \rstyle{r}(z_1, z_2)$, where $\rstyle{r} \in \rnames$ is a role name and $z_1, z_2$ are counting variables.
Computing the spectrum of $q_\rstyle{r}$ on a KB $\kb$ thus corresponds to deciding the possible values of $\sizeof{\rstyle{r}^\I}$ across models $\I$ of $\kb$.
Every such query satisfies preconditions of Lemma~\ref{lemma-closure-under-addition} and thus
its spectrum can be represented as in Problem~\ref{problem:compute-spectrum}.
We say a set $V$ is \emph{$\mathcal{L}$-role realizable} if there is a role $\rstyle{r}$ and a $\mathcal{L}$ KB $\kb$ s.t.\ $\spectrum{\kb}{q_\rstyle{r}} = V$.

We first highlight that, if a DL of interest can express that a role is functional, then there is a strong connection between role-realizable and concept-realizable sets.
Indeed, if $\tbox \models \top \sqsubseteq\ \leq 1\ \rstyle{r}.\top$, then in every model $\I$ of $\tbox$, we have $\sizeof{\rstyle{r}^\I} = \sizeof{(\exists \rstyle{r})^\I}$.
The following is an immediate consequence.

\begin{lemmarep}
	\label{lemma:concept-realizable-gives-role-realizable}
	Let $\mathcal{L}$ be a fragment of $\ALCIF$ and $V \subseteq \ninfty$.
	If $V$ is $\mathcal{L}$-concept realizable and axioms $\top \sqsubseteq\ \leq 1\ \rstyle{r}.\top$, $\exists \rstyle{r} \sqsubseteq \cstyle{C}$ and $\cstyle{C} \sqsubseteq \exists \rstyle{r}$ are permitted in $\mathcal{L}$, then $V$ is $\mathcal{L}$-role realizable.
\end{lemmarep}

\begin{proof}
	Let $\mathcal{L}$ be a fragment of $\ALCIF$ and $V \subseteq \ninfty$.
	Assume $V$ is $\mathcal{L}$-concept realizable and axioms $\top \sqsubseteq\ \leq 1\ \rstyle{r}.\top$ and $\cstyle{C} \sqsubseteq \exists \rstyle{r}$ are permitted in $\mathcal{L}$.
	From $V$ being $\mathcal{L}$-concept realizable, there exists a concept name $\cstyle{C}$ and an $\mathcal{L}$-KB $\kb = (\tbox, \abox)$ such that $\spectrum{\kb}{q_\cstyle{C}} = V$.
	Consider a fresh role name $\rstyle{r}$ and form the KB $\kb' := (\tbox', \abox)$ where $\tbox' := \tbox \cup \{ \top \sqsubseteq\ \leq 1\ \rstyle{r}.\top, \cstyle{C} \sqsubseteq \exists \rstyle{r} \}$.
	Notice that our assumptions guarantee that $\kb'$ is an $\mathcal{L}$-KB.
	The first axiom we added in $\tbox'$ guarantees that in every model $\I$ of $\kb'$, we have $\sizeof{\rstyle{r}^\I} = \sizeof{(\exists \rstyle{r})^\I}$, while the second and third enforce $(\exists \rstyle{r})^\I = \cstyle{C}^\I$.
	Therefore, seeing every model of $\kb'$ as a model of $\kb$ by dropping the interpretation of $\rstyle{r}$ yields $\spectrum{\kb'}{q_\rstyle{r}} \subseteq \spectrum{\kb}{q_\cstyle{C}}$.
	Conversely, extending every model $\I$ of $\kb$ into a model $\I'$ of $\kb'$ by defining $\rstyle{r}^{\I'} := \{ (d, d) \mid d \in \cstyle{C}^\I \}$ yields $\spectrum{\kb}{q_\cstyle{C}} \subseteq \spectrum{\kb'}{q_\rstyle{r}}$.
	Therefore $\spectrum{\kb'}{q_\rstyle{r}} = \spectrum{\kb}{q_\cstyle{C}} = V$, and $V$ is $\mathcal{L}$-role realizable.
\end{proof}

The above notably applies to $\dllitef$ and $\ALCF$ KBs.
For $\ALCIF$ KBs and joint with Theorem~\ref{theorem-shape-concept-alcif}, we obtain:

\begin{corollaryrep}
	\label{corollary-shape-role-alcif}
	A non-trivial subset of $\ninfty$ is $\ALCIF$-role realizable iff it is a subsemigroup of $\ninfty$ containing $\infty$.
\end{corollaryrep}

\begin{proof}
	The `if' direction is the consequence of combining Theorem~\ref{theorem-shape-concept-alcif} with Lemma~\ref{lemma:concept-realizable-gives-role-realizable}.
	For the `only-if' direction: consider a non-trivial subset $V$ of $\ninfty$ that is $\ALCIF$-role realizable.
	Therefore, there exists a role name $\rstyle{r}$ and an $\ALCIF$ KB $\kb = (\tbox, \abox)$ such that $\spectrum{\kb}{q_\rstyle{r}} = V$.
	Since $V$ is non-trivial, Remark~\ref{remark:trivial-sets} ensures we can apply Lemma~\ref{lemma-closure-under-addition} on $q_\rstyle{r}$ and $\kb$.
	Therefore, $V$ is a subsemigroup of $\ninfty$ that contains $\infty$ and we are done.
\end{proof}

In the case of concept names, we established identical results for $\ALCI$ and $\ALCF$ KBs (Theorem~\ref{theorem-shape-concept-alci-alcf}).
This does not hold with a role name.
In fact, we prove that $\ALCF$ KBs realize the same sets as $\ALCIF$, except for $\{ \infty \}$ and $\{ 0, \infty \}$.

\begin{toappendix}
	Before moving to Theorems~\ref{theorem-shape-role-alcf} and \ref{theorem:shape-role-alci-elfbot}, we start by proving a consequence of Lemma~\ref{lemma:fmp-alci-alcf} for roles instead of concepts.
	\begin{lemma}
		\label{lemma:fmp-alci-alcf-roles}
		Let $\kb = (\tbox, \abox)$ be an $\ALCI$ (resp.\ $\ALCF$) KB and $\rstyle{r} \in \rnames$.
		If there exists a model $\I$ of $\kb$ with ${\rstyle{r}^\I} \neq \emptyset$,
		then there exists a model $\J$ of $\kb$ with ${\rstyle{r}^{\J}} \neq \emptyset$ and $\sizeof{\domainof{\J}} \leq \sizeof{\individuals(\abox)} + 4^{\sizeof{\tbox} + 5}$.
		\emph{A fortiori}, $1 \leq \sizeof{\rstyle{r}^{\J}} \leq (\sizeof{\individuals(\abox)} + 4^{\sizeof{\tbox} + 5})^2$.
	\end{lemma}

	\begin{proof}
		Let $\kb = (\tbox, \abox)$ be an $\ALCI$ (resp.\ $\ALCF$) KB and $\rstyle{r} \in \rnames$.
		Assume there exists a model $\I$ of $\kb$ with ${\rstyle{r}^\I} \neq \emptyset$.
		Consider a fresh concept $\cstyle{C}$ and extend $\kb$ as $\kb' := (\tbox', \abox)$ where $\tbox' := \tbox \cup \{ \cstyle{C} \sqsubseteq \exists \rstyle{r} \}$.
		Note that $\kb'$ is an $\ALCI$ (resp.\ $\ALCF$) KB and that $\sizeof{\tbox'} \leq \sizeof{\tbox} + 5$.
		We can extend $\I$ into a model of $\kb'$ by setting $\cstyle{C}^{\I'} = (\exists \rstyle{r})^\I$.
		In particular, we can now apply Lemma~\ref{lemma:fmp-alci-alcf} with $\kb'$, $\cstyle{C}$ and model $\I'$.
		We thus obtain a model $\J'$ of $\kb'$ with $\sizeof{\domainof{\J}} \leq \sizeof{\individuals(\abox)} + 4^{\sizeof{\tbox}}$ and $\cstyle{C}^{\J'} \neq  \emptyset$.
		Since $\J'$ is a model of $\kb'$, we have $\cstyle{C}^{\J'} \subseteq (\exists \rstyle{r})^{\J'}$ and thus $(\exists \rstyle{r})^{\J'} \neq \emptyset$.
		Therefore, from the semantics of $\exists \rstyle{r}$, we also have $\rstyle{r}^{\J'} \neq \emptyset$.
		Restricting $\J'$ to a model $\J$ of $\kb$ by dropping the interpretation of $\cstyle{C}$, we obtain the desired model.
	\end{proof}
\end{toappendix}

\begin{theoremrep}
	\label{theorem-shape-role-alcf}
	A non-trivial subset of $\ninfty$ is $\ALCF$-role realizable iff it is a subsemigroup of $\ninfty$ containing $\infty$ and at least a non-zero natural.
\end{theoremrep}

\begin{proof}

	We start with the `only-if' direction: consider a non-trivial subset $V$ of $\ninfty$ that is $\ALCF$-role realizable.
	Therefore, there exists a role name $\rstyle{r}$ and an $\ALCF$ KB $\kb = (\tbox, \abox)$ such that $\spectrum{\kb}{q_\rstyle{r}} = V$.
	Since $V$ is non-trivial, Remark~\ref{remark:trivial-sets} ensures we can apply Lemma~\ref{lemma-closure-under-addition} on $q_\rstyle{r}$ and $\kb$.
	Therefore, $V$ is a subsemigroup of $\ninfty$ that contains $\infty$.
	It remains to verify that is also contains a non-zero integer.
	Since $\spectrum{\kb}{q_\rstyle{r}} = V$ and $V$ contains $\infty$, there exists a model $\I$ of $\kb$ with $\sizeof{\rstyle{r}^\I} = \infty$.
	With this model at hand, we apply Lemma~\ref{lemma:fmp-alci-alcf-roles} and obtain a model $\J$ of $\kb$ in which $\rstyle{r}^\J$ is finite and non-empty.
	Therefore $\spectrum{\kb}{q_\rstyle{r}}$ contains a non-zero integer and so does $V$ since $\spectrum{\kb}{q_\rstyle{r}} = V$.

	We now prove the `if' direction.
	Let $V$ be such a subsemigroup of $\ninfty$ and $\alpha_1, \dots, \alpha_n$ its generators that are neither $0$ nor $\infty$ (by assumption, $n \geq 1$).
	We build an $\ALCF$ TBox $\tbox$ s.t.\ $\spectrum{(\tbox, \emptyset)}{q_\rstyle{r}} = V$ containing axioms:
	\[
	\begin{array}{c}
		\exists \rstyle{r}.\top \sqsubseteq \bigsqcup_{i = 1}^n \cstyle{A_i}
		\qquad \quad
		\cstyle{A_{i}} \sqcap \cstyle{A_{i'}} \sqsubseteq \bot
		\medskip\\
		\cstyle{A_i} \sqsubseteq \exists \rstyle{r}.\cstyle{B_{i, j}}
		\qquad\quad
		\top \sqsubseteq\ \leq 1\ \rstyle{r}.\cstyle{B_{i, j}}
		\qquad\quad
		\exists \rstyle{r}.\cstyle{B_{i, j}} \sqsubseteq \cstyle{A_i}
		\medskip\\
		\top \sqsubseteq \bigsqcup_{i = 1}^n \bigsqcup_{j = 1}^{\alpha_i} \cstyle{B_{i, j}}
		\qquad\quad
		\cstyle{B_{i, j}} \sqcap \cstyle{B_{k, \ell}} \sqsubseteq \bot
	\end{array}
	\]
	where $i, i', k \in \llbracket 1, n \rrbracket$, $j \in \llbracket 1, \alpha_i \rrbracket$, $\ell \in \llbracket 1, \alpha_k \rrbracket$ with $i \neq i'$ and $(i, j) \neq (k, \ell)$.
	If $0 \notin V$, we also add:
	\(
	\lnot \exists \rstyle{r}.\top \sqsubseteq \exists \rstyle{s}.(\exists \rstyle{r}.\top)
	\).
\end{proof}

To establish the above, we strongly rely on qualified functional dependencies, \emph{i.e.}\ axioms $\cstyle{B_1} \sqsubseteq\ \leq 1\ \rstyle{r}.\cstyle{B_2}$ where $\cstyle{B_1}, \cstyle{B_2}$ are not just $\top$.
As $\ALCIF$ is sometimes
defined to only support
unqualified functionality, \emph{i.e.}\ only $\cstyle{B_1} = \cstyle{B_2} = \top$,
we also treat this fragment,

here denoted $\ALCIF^*$.

\begin{theoremrep}
	\label{theorem:shape-role-alci-elfbot}
	\label{theorem-shape-role-elifbot}
	If a non-trivial subset of $\ninfty$ is $\ALCI$- (resp.\ $\ALCF^*$-) role realizable, then it has shape $S \cup \llbracket M, \infty \rrbracket$ for some $M \in \NN$ and $S \subseteq \llbracket 0, M \rrbracket$.
	In the case of $\ELIFbot$, shapes $\{ \infty \}$ and $\{ 0, \infty \}$ are also permitted.
\end{theoremrep}

\begin{toappendix}

	The proof of this theorem for $\ALCI$ and $\ALCF^*$ essentially follows the proof of Theorem~\ref{theorem-shape-concept-alci-alcf}.
	We already proved Lemma~\ref{lemma:fmp-alci-alcf-roles}, echoing to Lemma~\ref{lemma:fmp-alci-alcf} and it only remains to prove a statement echoing to Lemma~\ref{lemma-concept-alci-alcf-plus-one}.
	In the present case, however, we cannot achieve a `$+1$' starting from any model, but can prove that every non-trivial $\ALCI$ (resp.\ $\ALCF^*$) role-realizable set contains two consecutive integers, which is sufficient to our purpose (see for example the proof of Theorem~\ref{theorem-shape-concept-elifbot}).

	\begin{lemma}
		\label{lemma:ultimately-one-role-alci}
		Let $\kb$ be an $\ALCI$ (resp.\ $\ALCF^*$) KB and $\rstyle{r}$ a role name.
		If there exists a model $\I$ of $\kb$ with $1 \leq \sizeof{\rstyle{r}^\I} < \infty$,
		then there exists a model $\J$ of $\kb$ with $\sizeof{\rstyle{r}^{\J}} = 2\sizeof{\rstyle{r}^\I} + 1$.
	\end{lemma}

	\begin{proof}
		For an $\ALCF^*$ KB $\kb$, s.t.\ $\kb \models \top \sqsubseteq\ \leq 1\ \rstyle{r}.\top$, we have $\sizeof{\rstyle{r}^\I} = \sizeof{(\exists \rstyle{r})^\I}$ in every model and we can thus rely on Lemma~\ref{lemma-concept-alci-alcf-plus-one}.
		For other $\ALCF^*$ KBs and for $\ALCI$ KBs, we do the following.
		Assume there exists a model $\I$ of $\kb$ with $1 \leq \sizeof{\rstyle{r}^\I} < \infty$.
		In particular, $\rstyle{r}^\I$ is non-empty so there is some $(u, v) \in \rstyle{r}^\I$.
		We obtain $\J$ by duplicating the model $\I$ and adding one $\rstyle{r}$-edge between $u$ and the copy of $v$.
		More formally, we set $\domainof{\J} := \domainof{\I} \cup \{ e' \mid e \in \domainof{\I} \}$ where $e'$ is a fresh element.
		The interpretation $\J$ is now defined for each concept name $\cstyle{A}$ and role name $\rstyle{p}$ as:
		\begin{align*}
			\cstyle{A}^\J & := \cstyle{A}^\I \cup \{ e' \mid e \in \cstyle{A}^\I \}
			\\
			\rstyle{p}^\J & := \, \rstyle{p}^\I \cup \{ (e_1', e_2') \mid (e_1, e_2) \in \rstyle{p}^\I \} \cup \{ (u, v') \}
		\end{align*}
		It is trivial to verify that $\J$ is a model of $\kb$ and that $\sizeof{\rstyle{r}^{\J}} = 2\sizeof{\rstyle{r}^\I} + 1$ as desired.
	\end{proof}

	For $\ELIFbot$, we proceed as follows.

		Notice that $\{ \infty \}$ and $\{ 0, \infty \}$ are indeed realizable, from Lemma~\ref{lemma:concept-realizable-gives-role-realizable} joint with Examples~\ref{example-concept-elif-infinity} and \ref{example-dllitef-zero-infinity}.
		Now, assume we have a non-trivial subset $V$ of $\ninfty$ that is $\ELIFbot$-role realizable and that is neither $\{ \infty \}$ nor $\{ 0, \infty \}$.
		In particular, we have an $\ELIFbot$ KB $\kb = (\tbox, \abox)$ and a role name $\rstyle{r}$ s.t.\ $V = \spectrum{\kb}{q_\rstyle{r}}$.
		If $\tbox \models \top \sqsubseteq \ \leq 1\ \rstyle{r}.\top$ or $\tbox \models \top \sqsubseteq \ \leq 1\ \rstyle{r^-}.\top$, then $\spectrum{\kb}{q_\rstyle{r}} = \spectrum{\kb}{q_{\exists \rstyle{r}}}$ or $\spectrum{\kb}{q_\rstyle{r}} = \spectrum{\kb}{q_{\exists \rstyle{r^-}}}$, and we can directly conclude with Theorem~\ref{theorem-shape-concept-elifbot}.
		We can thus assume $(\dagger)$: we have $\tbox \not\models \top \sqsubseteq \ \leq 1\ \rstyle{r}.\top$ and $\tbox \not\models \top \sqsubseteq \ \leq 1\ \rstyle{r^-}.\top$ .

		Now, since $V$ is non-trivial and neither $\{ \infty \}$ nor $\{ 0, \infty \}$, it contains some $n \in \NN$, that is $\sizeof{\rstyle{r}^\I} = n$ for some model $\I$ of $\kb$.
		Notice it implies that both $(\exists \rstyle{r})^\I$ and $(\exists \rstyle{r^-})^\I$ are also finite
		by Remark~\ref{remark:role-bounded}.
		We define a TBox $\tbox_\rstyle{r}$ by successively reversing the $\exists \rstyle{r}$- and $\exists \rstyle{r^-}$-generating cycles.
		Formally, we set $\tbox_{0} := \tbox$ and $\tbox_{k+1} := ((\tbox_{k})_{\exists \rstyle{r}})_{\exists \rstyle{r^-}}$ for all $k \geq 0$.
		There are only finitely many possible axioms, thus there exists $K \geq 0$ s.t.\ $\tbox_{K+1} = \tbox_K$ and we set $\tbox_\rstyle{r} := \tbox_{K}$.
		It is readily verified that $\I$ is a model of $\tbox_\rstyle{r}$ and that both $\exists \rstyle{r}$ and $\exists \rstyle{r^-}$ are safe in $\tbox_\rstyle{r}$.

		Therefore, by applying Lemma~\ref{lemma:ils}, we obtain two models $\I_1 := \ils{(\tbox_\rstyle{r}, \{ (\exists \rstyle{r})(\istyle{a})\})}{}$ and $\I_2 := \ils{(\tbox_\rstyle{r}, \{ (\exists \rstyle{r^-})(\istyle{b})\})}{}$ in which both the interpretations of $\exists \rstyle{r}$ and $\exists \rstyle{r^-}$ are finite.
		In particular, using Remark~\ref{remark:role-bounded}, $\rstyle{r}^{\I_1}$ and $\rstyle{r}^{\I_2}$ are also finite.
		We now form the disjoint union $\J := \I \cup \I_1 \cup \I_2$.
		It is easily verified that $\J$ is a model of $\kb$ and that ${\rstyle{r}^\J} =: m$ is finite.
		Consider the interpretation $\J'$ obtained from $\J$ by adding $(\istyle{a}, \istyle{b})$ to $\rstyle{r}^\J$.
		From $(\dagger)$, this cannot violate any functionality constraints, thus $\J'$ is a model of $\kb$.
		Since $\J'$ has exactly $m + 1$ instances of $\rstyle{r}$, we can finish the proof of Theorem~\ref{theorem-shape-role-elifbot} by setting $M = m (m+1)$ and $S = \llbracket 0, M \rrbracket \cap \spectrum{\kb}{q_\rstyle{r}}$.

\end{toappendix}

A key ingredient is to reuse techniques from Section~\ref{section:concepts} on concept cardinality queries, notably for $\ELIFbot$ KBs, is the observation that in
every interpretation $\I$, we have:
\begin{center}
	\(
	\max\left(\sizeof{(\exists \rstyle{r})^\I}, \sizeof{(\exists \rstyle{r^-})^\I}\right)
	\leq
	\sizeof{\rstyle{r}^\I}
	\leq
	\sizeof{(\exists \rstyle{r})^\I} \cdot \sizeof{(\exists \rstyle{r^-})^\I}.
	\)
\end{center}%
Moreover, note that Theorem~\ref{theorem:shape-role-alci-elfbot} is not a complete characterization.
Indeed, as for concept cardinality queries on $\ELIFbot$ KBs (see Example~\ref{example-bad-concept}), we exhibit a simple setting in which the spectrum contains a non-trivial part $S$, already for an $\EL_\bot$ TBox and the empty ABox.

\begin{example}
	\label{example-bad-role}
	$\tbox = \{ \top \sqsubseteq \exists \rstyle{r}.\cstyle{A_1},
	\top \sqsubseteq \exists \rstyle{r}.\cstyle{A_2},
	\cstyle{A_1} \sqcap \cstyle{A_2} \sqsubseteq \bot \}$ is an $\ELbot$ TBox and we have $\spectrum{(\tbox, \emptyset)}{q_\rstyle{r}} = \{ 4 \} \cup \llbracket 6, \infty \rrbracket$.
\end{example}

In the remaining fragments of $\ALCIF$, our results mirror those of the concept cardinality case, as summarized by the following two theorems echoing Theorems~\ref{theorem-shape-concept-eli-elf-elif} and \ref{theorem-shape-concept-dllitef}.

\begin{theoremrep}
	\label{theorem-shape-role-eli-elf-elif}
	A non-trivial subset of $\ninfty$ is $\ELI$- (resp.\ $\ELF$-) role realizable iff it has shape $\llbracket M, \infty \rrbracket$ for some $M \in \NN$.
	For $\ELIF$, the shape $\{ \infty \}$ is also permitted.
\end{theoremrep}%

\begin{proof}
	The proof of the above follows the one of Theorem~\ref{theorem-shape-concept-eli-elf-elif}.
	Notably, the proof of Lemma~\ref{lemma-concept-elif-plus-one} still gives a fitting construction.
\end{proof}

\begin{theoremrep}
	\label{theorem-shape-role-dllitef}
	A non-trivial subset of $\ninfty$ is $\dllitef$-role realizable iff it has shape $\{ \infty \}$, $\{ 0, \infty \}$, $\{ 0 \} \cup \llbracket M, \infty \rrbracket$ or $\llbracket M, \infty \rrbracket$ for some $M \in \NN$.
	The same holds for $\dllitecore$ but without shapes $\{ \infty \}$ and $\{ 0, \infty \}$.
\end{theoremrep}

\begin{toappendix}

	The `if' direction for $\dllitef$ follow from Lemma~\ref{lemma:concept-realizable-gives-role-realizable} joint with Theorem~\ref{theorem-shape-concept-dllitef}.
	For $\dllitecore$, it is proved by the following example, echoing to Example~\ref{example:concept-one-from-min}
	\begin{example}
		\label{example:role-one-from-min}
		Let $m \in \NN$.
		Consider the $\ELbot$ TBox $\tbox$ that contains the following axioms:
		\[
		\begin{array}{c}
			\exists \rstyle{r} \sqsubseteq \exists \rstyle{r}_k.\cstyle{A}_k
			\qquad
			\cstyle{A}_k \sqsubseteq \exists \rstyle{r}
			\qquad
			(1 \leq k \leq M)
			\smallskip \\
			\cstyle{A}_i \sqcap \cstyle{A}_j \sqsubseteq \bot
			\qquad
			(1 \leq i < j \leq M)
		\end{array}
		\]
		It is immediate to verify that $\spectrum{(\tbox, \emptyset)}{q_\rstyle{r}} = \{ 0 \} \cup \llbracket M, \infty \rrbracket$ while $\spectrum{(\tbox, \{ \rstyle{r}(\istyle{a}, \rstyle{b}) \})}{q_\rstyle{r}} = \llbracket M, \infty \rrbracket$.
		By replacing each axiom $\exists \rstyle{r}.\top \sqsubseteq \exists \rstyle{r}_k.\cstyle{A_k}$ in $\tbox$ by the two axioms $\exists \rstyle{r} \sqsubseteq \exists \rstyle{r}_k$ and $\exists \rstyle{r_k^-} \sqsubseteq \cstyle{A}_k$, we obtain the same result for a $\dllitecore$ KB.
	\end{example}

	The converse is obtained by the construction below (Lemma~\ref{lemma:role-dllitef-plus-one}), echoing Lemma~\ref{lemma-concept-dllitef-plus-one} from the case of concept names.
	An additional ingredient is the observation that, if the role of interest is neither functional nor inverse functional in the KB (in which cases it reduces to the case of concept names), then it cannot occur along an IFP, giving us more flexibility to construct well-behaved models.

	\begin{lemmarep}
		\label{lemma:role-dllitef-plus-one}
		Let $\kb$ be an $\dllitef$ KB and $\rstyle{r}$ a role name.
		If there exists a model $\I$ of $\kb$ with $1 \leq \sizeof{\rstyle{r}^\I} < \infty$,
		then there exists a model $\J$ of $\kb$ with $\sizeof{\rstyle{r}^{\J}} = \sizeof{\rstyle{r}^\I} + 1$.
	\end{lemmarep}

	\begin{proof}
		First note that if $\tbox \models \top \sqsubseteq \ \leq 1\ \rstyle{r}.\top$ or $\tbox \models \top \sqsubseteq \ \leq 1\ \rstyle{r^-}.\top$, then $\spectrum{\kb}{q_\rstyle{r}} = \spectrum{\kb}{q_{\exists \rstyle{r}}}$ or $\spectrum{\kb}{q_\rstyle{r}} = \spectrum{\kb}{q_{\exists \rstyle{r^-}}}$, and we can directly conclude with Theorem~\ref{theorem-shape-concept-dllitef}.
		Otherwise, let $\I$ be a model of $\kb$ with $1 \leq \sizeof{\rstyle{r}^\I} < \infty$.
		From Lemma~\ref{lemma:cycle-reversion}, $\I$ is a model of $\tbox_\rstyle{r}$ and by construction both $\exists \rstyle{r}$ and $\exists \rstyle{r^-}$ are safe concepts in $\tbox_\rstyle{r}$.

		Consider the interpretation $\preils{\kb_1}{}$ obtained by applying the construction from Section~\ref{subsection:concept-elifbot} on $\kb_1 := (\tbox_\rstyle{r}, \{ (\exists \rstyle{r})(\istyle{a})\})$.
		From Lemma~\ref{lemma:ils-and-preils}, we know $(\exists \rstyle{r})^{\J_{\kb_1}} = \{ \istyle{a} \}$.
		If $(\exists \rstyle{r^-})^{\J_{\kb_1}} = \emptyset$ also holds, then we can form the disjoint union $\J_1 := \I \cup \preils{\kb_1}{}$ and find a witness $d_\rstyle{s} \in (\exists \rstyle{s}^-)^\I$ for each $d \in \domainof{\preils{\kb_1}{}}$ and each $\rstyle{s} \in \posroles$ s.t.\  $\tbox \models \type_{\J_1}(d) \sqsubseteq \exists\rstyle{s}.\top$ and $d \notin (\exists \rstyle{s}.\top)^{\J_1}$.
		Adding the pair $(d, d_\rstyle{s})$ to $\rstyle{s}^{\J_1}$ for each such $d$ and $\rstyle{s}$, we obtain the desired model of $\kb$ with exactly one extra instance of $\rstyle{r}$ being $(\istyle{a}, d_\rstyle{r})$.
		Otherwise we have $(\exists \rstyle{r^-})^{\J_{\kb_1}} \neq \emptyset$.
		It follows that there must exists an IFP $\mathcal{P}_1$ from $\type_{\preils{\kb_1}{}}(\istyle{a})$ to some type $t$ containing $\exists \rstyle{r}^-$.

		Similarly, consider the interpretation $\preils{\kb_2}{}$ obtained by applying the construction from Section~\ref{subsection:concept-elifbot} on $\kb_2 := (\tbox_\rstyle{r}, \{ (\exists \rstyle{r^-})(\istyle{b})\})$.
		If $(\exists \rstyle{r})^{\J_{\kb_2}} = \emptyset$, we can conclude by constructing an interpretation $\J_2$ as we constructed $\J_1$.
		Otherwise, we have another IFP $\mathcal{P}_2$ from $\type_{\preils{\kb_2}{}}(\istyle{b})$ to some type $t'$ containing $\exists \rstyle{r}$.
		Since $\type_{\preils{\kb_2}{}}(\istyle{b})$ is the type induced by $\exists \rstyle{r}^-$ and that $t$ contains $\exists \rstyle{r^-}$, we have $\type_{\preils{\kb_2}{}}(\istyle{b}) \subseteq t$.
		We can thus concatenate $\mathcal{P}_1$ and $\mathcal{P}_2$ to obtain an IFP from $\type_{\preils{\kb_1}{}}(\istyle{a})$ to $t'$.
		Again, since $\type_{\preils{\kb_1}{}}(\istyle{a})$ is the type induced by $\exists \rstyle{r}$ and that $t'$ contains $\exists \rstyle{r}$, we have $\type_{\preils{\kb_1}{}}(\istyle{a}) \subseteq t'$.
		Lemma~\ref{lemma:cycling-types-are-equal-types} thus applies and gives $\type_{\preils{\kb_1}{}}(\istyle{a}) = t'$.
		We can thus concatenate $\mathcal{P}_2$ and $\mathcal{P}_1$ to obtain an IFP from $\type_{\preils{\kb_2}{}}(\istyle{b})$ to $t$, and Lemma~\ref{lemma:cycling-types-are-equal-types} applies again and ensures $\type_{\preils{\kb_2}{}}(\istyle{b}) = t$.
		This notably guarantees that all instances of $\exists \rstyle{r^-}$ in $\preils{\kb_1}{}$ have the same type $\type_{\preils{\kb_2}{}}(\istyle{b})$.
		Furthermore, notice the concatenation of $\mathcal{P}_1$ and $\mathcal{P}_2$ forms an $\exists \rstyle{r}$-generating cycle, which has thus been reversed in $\tbox_\rstyle{r}$.
		In particular, it guarantees $\type_{\preils{\kb_1}{}}(\istyle{a})$ and $\type_{\preils{\kb_2}{}}(\istyle{b})$ are in the same type class $P$.
		Arguing as in the proof of Lemma~\ref{lemma-concept-dllitef-plus-one}, only one instance of each type of $P$ is present in $\preils{\kb_1}{}$.
		It follows that $(\exists \rstyle{r^-})^{\preils{\kb_1}{}} = \{ e \}$ for some $e \in \domainof{\preils{\kb_1}{}}$ (note it is possible that $e = \istyle{a}$).

		To conclude, it now suffices to consider again the disjoint union $\J_1 := \I \cup \preils{\kb_1}{}$, in which, this time, we add the pair $(\istyle{a}, e)$ to $\rstyle{r}^{\J_1}$ and pairs $(d, d_\rstyle{s})$ to $\rstyle{s}^{\J_1}$ for each $d \in \domainof{\preils{\kb_1}{}}$ and each $\rstyle{s} \in \posroles \setminus \{ \rstyle{r}, \rstyle{r^-} \}$ s.t.\  $\tbox \models \type_{\J_1}(d) \sqsubseteq \exists\rstyle{s}.\top$ and $d \notin (\exists \rstyle{s}.\top)^{\J_1}$.
		It gives the desired model of $\kb$ with exactly one extra instance of $\rstyle{r}$ being $(\istyle{a}, e)$.

	\end{proof}

\end{toappendix}

Based on these results, we classify 
the complexity of computing the proposed representation.
Our complexity results also appear in Table~\ref{table:concept-complexity}.
For the upper bounds, we follow the same approach as presented 
in Section~\ref{sec:complexity} for concept cardinality queries, and obtain the following:

\begin{theoremrep}
	\label{theorem:complexity-roles-upper-bounds}
	$\mathsf{Spectrum}(q_\rstyle{r}, \tbox)$ is in:
	\begin{itemize}
		\item $\FPwithlogcalls$ if $\tbox$ is in $\ELIFbot$, $\ALCI$ or $\ALCF^*$.
		\smallskip
		\item $\FP$ if $\tbox$ is in $\dllitecore$.
	\end{itemize}
\end{theoremrep}

\begin{toappendix}

	The proof of Theorem~\ref{theorem:complexity-roles-upper-bounds} is given in Section~\ref{sec:upper-bounds-proof},
	together with the proof of upper bounds for concept queries. Here we provide the main ingredient, expressed by the following lemma.
	
	\begin{lemma}
		\label{lemma:polynomial-bound-for-roles}
			Let $\tbox$ be an $\ELIFbot$-, $\ALCI$-, or $\ALCF^\ast$-TBox and $\rstyle{r}$ a role name.
			There exists a polynomial $p(x)$, with coefficients computable from $\tbox$ and $\rstyle{r}$, such that $p(x) \geq 1$ for every $x \geq 0$
			and, for every KB $\kb = (\tbox, \abox)$, either
			\begin{itemize}
				\item $p(|\abox|) \notin  \spectrum{\kb}{q_\rstyle{r}}$ and $\spectrum{\kb}{q_\rstyle{r}} \subseteq \{0,\infty\}$; or
				\item $p(|\abox|) \in \spectrum{\kb}{q_\rstyle{r}}$ and $\spectrum{\kb}{q_\rstyle{r}} = S \cup \llbracket M, \infty\rrbracket$
				where $S \subseteq \llbracket 0, M \rrbracket$ and $M < p(|\abox|)$.
			\end{itemize}
			Moreover, in the latter case,
			$M {\leq} \min(\spectrum{\kb}{q_\rstyle{r}}) + p(0)$.
	\end{lemma}
	

	The proof of the lemma is analogous to the one of Lemma~\ref{lemma:polynomial-bound-for-conccepts} from the concept case.
	We notably rely on the observation that in
	every interpretation $\I$, we have
	\begin{center}
		\(
		\max\left(\sizeof{(\exists \rstyle{r})^\I}, \sizeof{(\exists \rstyle{r^-})^\I}\right)
		\leq
		\sizeof{\rstyle{r}^\I}
		\leq
		\sizeof{(\exists \rstyle{r})^\I} \cdot \sizeof{(\exists \rstyle{r^-})^\I}.
		\)
	\end{center}%
	to derive the (intermediate) bounds on $\sizeof{\rstyle{r}^\I}$ from the bounds on $\sizeof{(\exists \rstyle{r})^\I}$ and $\sizeof{(\exists \rstyle{r^-})^\I}$.
\end{toappendix}

The following theorem provides a lower bound that applies already for $\EL$ KBs.
\begin{theoremrep}
	\label{theorem:complexity-role-alc}
	There exists an $\EL$ TBox $\tbox$ such that $\mathsf{Spectrum}(q_\rstyle{r}, \tbox)$ is $\FPwithlogcalls$-hard.
\end{theoremrep}

\begin{proof}
	Consider the role cardinality query $q_\rstyle{r}$ and the $\EL$ TBox $\tbox$ containing the four following axioms:
	\[
	\cstyle{Vertex} \sqsubseteq \exists \rstyle{r}.\cstyle{Chosen}
	\qquad
	\exists \rstyle{t}. (\exists \rstyle{r}.(\cstyle{Chosen} \sqcap \cstyle{Indep}) \sqcap \exists \rstyle{edge}.(\exists \rstyle{r}.(\cstyle{Chosen} \sqcap \cstyle{Indep}))) \sqsubseteq \exists \rstyle{r}.\top.
	\]
	Now, given a graph $G = (V, E)$, we construct an ABox $\abox$ containing facts:
	\begin{itemize}
		\item $\cstyle{Vertex}(v)$, $\rstyle{r}(v, i_v)$, $\cstyle{Indep}(i_v)$ for each $v \in V$;
		\item $\rstyle{edge}(u, v)$ for each $\{ u, v \} \in E$;
		\item $\rstyle{t}(i_u, v)$ for each $u, v \in V$.
	\end{itemize}
	We let $\kb = (\tbox, \abox)$ and prove that $k$ is the maximal size of an independent set in $G$ iff $\spectrum{\kb}{q_\rstyle{r}} = \llbracket 2\sizeof{V} - k, \infty \rrbracket$, that is the triple $(S, M, \alpha) := (\emptyset, 2\sizeof{V} - k, 1)$ is an effective representation of $\spectrum{(\tbox, \abox)}{q_\rstyle{r}}$.

	The first axiom requires every element $v$ representing a vertex to have a $\rstyle{r}$-successor for the concept $\cstyle{Chosen}$.
	Note that, along with the $\rstyle{r}$ edge from $v$ to elements $i_v$, this leads to, \emph{a priori}, at least $2\sizeof{V}$ instances of $\rstyle{r}$ in every model of $\kb$.
	It is possible, however, for the witness of $v$ to be exactly $i_v$, thus reducing the total number of instances by $1$ per such $v$ reusing $i_v$.
	However, due to the second axiom, this cannot possibly work for all vertices at once.
	Indeed, the second axiom guarantees that, if two adjacent vertices $v_1$, $v_2$ have $i_{v_1}$, resp.\ $i_{v_2}$, as their witnesses, then all elements $i_u$, for $u \in V$, also require a $\rstyle{r}$ successor, thus enforcing at least $2\sizeof{V}$ instances of $\rstyle{r}$.
	It is now clear that we can safely use $k$ elements $i_v$'s as witnesses, leading to only $2\sizeof{V} - k$ instances of $\rstyle{r}$ iff there exists an independent set of size $k$ in $G$.
	The fact that the spectrum contains every integer above its minimum value is a consequence of Theorem~\ref{theorem-shape-role-eli-elf-elif}.

\end{proof}

\begin{toappendix}
	\section{Proof of the upper bounds}
	\label{sec:upper-bounds-proof}
	Here we will prove both Theorem~\ref{theorem:complexity-concepts-upper-bounds} and Theorem~\ref{theorem:complexity-roles-upper-bounds} as they share the same general algorithm.
	We will now provide more detailed description and complexity bounds of the algorithms, allowing us to prove proposed upper bounds.

	Let $\kb= (\tbox, \abox)$ be the KB and $q_\cstyle{C}$, $q_\rstyle{r}$ be the concept and role query, respectively.
	Let $q = q_\cstyle{C}$ or $q = q_\rstyle{r}$.
	Let $p$ be the polynomial obtained from Lemma~\ref{lemma:polynomial-bound-for-conccepts}, for concept queries, or Lemma~\ref{lemma:polynomial-bound-for-roles}, for roles.
	The algorithm is given in Algorithm~\ref{alg:computing-spectra}.

	\begin{algorithm}
		\caption{Computing spectra}\label{alg:computing-spectra}
		\begin{algorithmic}
			\STATE $\ell \leftarrow p(|\abox|)$
			\IF{$\ell \notin \spectrum{\kb}{q_{\cstyle{C}}}$}
			\STATE // $\spectrum{\kb}{q_{\cstyle{C}}} \subseteq \{0,\infty\}$
			\IF{$\kb \not \models \top$}
			\STATE return $\emptyset$
			\ENDIF
			\IF{$\kb \not \models q_{\cstyle{C}}$}
			\STATE return $\{0\}$
			\ENDIF
			\IF{$0 \in \spectrum{\kb}{q_{\cstyle{C}}}$}
			\STATE return $\{0, \infty\}$
			\ELSE
			\STATE return $\{\infty\}$
			\ENDIF
			\ENDIF
			\STATE // $\ell \in \spectrum{\kb}{q_{\cstyle{C}}}$
			\STATE // $\spectrum{\kb}{q_{\cstyle{C}}} = S \cup \llbracket M, \infty \rrbracket$
			\STATE $m = \text{computeMIN}(\spectrum{\kb}{q_{\cstyle{C}}})$
			\STATE $S \leftarrow \emptyset$
			\FOR{ $i \in \llbracket m, m {+} p(0) \rrbracket$ }
			\IF{$i \in \spectrum{\kb}{q_{\cstyle{C}}}$}
			\STATE $S = S \cup \{i\}$
			\ENDIF
			\ENDFOR
			\STATE return $S \cup \llbracket m + p(0) + 1, \infty \rrbracket$
		\end{algorithmic}
	\end{algorithm}

	Before we proceed, let us make some observations on membership tests.
	\begin{lemma}
		\label{lemma:membership-test-concept}
		Given $\ALCIF$ KB $\kb = (\tbox,\abox)$, a concept name $\cstyle{C}$, and a number $n< \infty$ in unary, deciding if
		$n \in \spectrum{\kb}{q_{\cstyle{C}}}$ is in $\NP$ w.r.t. data complexity.
	\end{lemma}

	\begin{proof}
		To decide the membership extend the ABox $\abox$ to the ABox $\abox' \supseteq \abox$ by guessing which individuals $\istyle{c}$ in the ABox $\abox$ should belong to $\cstyle{C}^\J$ for some model $\J$ of $\kb$ and add assertions $\cstyle{C}(\istyle{c})$ to $\abox'$.
		Then add to $\abox'$ assertions $\cstyle{C}(\istyle{c'})$ with new fresh individuals $\istyle{c'}$ so that $|\cstyle{C}^{\abox'}| = n$.
		Finally, verify if there is a model $\J$ extending $\abox'$ and $\tbox$ such that $\cstyle{C}^\J = \cstyle{C}^{\abox'}$.
		If there is then $n$ belongs to the spectrum, if not then $n$ does not belong to the spectrum.
		Since the number of guessed individuals is polynomial in the size of $\abox$ and verifying the existence of the model $\J$ can be done in $\NP$ using closed predicates, see Theorem~4 in~\cite{sanja21:bounded-predicates}, we can combine the two into a single $\NP$ call. Thus, membership check whether $n$ belongs to the spectrum can be done in $\NP$.
	\end{proof}

	\begin{lemma}
		\label{lemma:membership-test-role}
		Given $\ALCIF$ KB $\kb = (\tbox,\abox)$, a role name $\rstyle{r}$, and a number $n < \infty$ in unary, deciding if $n \in \spectrum{\kb}{q_{\rstyle{r}}}$ is in $\NP$ w.r.t. data complexity.
	\end{lemma}

	\begin{proof}
		As in the concept case, to decide the membership extend the ABox $\abox$ to the ABox $\abox' \supseteq \abox$ by guessing a list of $n$ pairs $(c',d')$ that should constitute $\rstyle{r}^\J$ for some model $\J$ of $\kb$ and add assertions $\rstyle{r}({c}',{d}')$ to $\abox'$.
		Then, verify if there is a model $\J$ extending $\abox'$ and $\tbox$ such that $\cstyle{C}^\J = \cstyle{C}^{\abox'}$.
		If there is then $n$ belongs to the spectrum, if not then $n$ does not belong to the spectrum.
		Since the number of guessed pairs is polynomial in the size of $\abox$ and verifying the existence of the model $\J$ can be done in $\NP$ using closed predicates, see Theorem~4 in~\cite{sanja21:bounded-predicates}, we can combine the two into a single $\NP$ call. Thus, membership check whether $n$ belongs to the spectrum can be done in $\NP$.
	\end{proof}

	\begin{lemma}
		\label{lemma:minimum-test}
		Given $\ALCIF$ KB $\kb = (\tbox,\abox)$, a query $q_\cstyle{C}$ (resp. $q_\rstyle{r}$), and a number $n<\infty$ in unary, deciding if
		$n \ge \min(\spectrum{\kb}{q_{\cstyle{C}}})$ is in $\coNP$ w.r.t. data complexity.
	\end{lemma}

	\begin{proof}
		Simply guess a number $0 \leq n' < n$ such that $n' \in \spectrum{\kb}{q_{\cstyle{C}}}$.  Since the membership test can be performed in $\NP$, see Lemma~\ref{lemma:membership-test-concept} and Lemma~\ref{lemma:membership-test-role}, lemma holds.
	\end{proof}

	We can now proceed with the analysis of the complexity of the main algorithm.

	\subsection{$\FPwithlogcalls$}%
Here we argue the complexity upper bounds for the case of concept counting  queries w.r.t. $\ELIFbot$, $\ALCI$ or $\ALCF$ TBoxes and
role counting queries for the $\ELIFbot$, $\ALCI$ or $\ALCF^*$ TBoxes.

\paragraph{Concept queries}
Let $\kb$ be an $\ELIFbot$, $\ALCI$, or $\ALCF$ KB.
The algorithm consist of $2$ satisfiability checks, $p(0)+2$ membership tests and a subroutine \emph{computeMIN} that computes $\min(\spectrum{\kb}{q_{\cstyle{C}}})$.

The two satisfiability checks can be decided by two calls to an $\NP$ oracle, see \emph{e.g.}\ \cite{birte2008cqa-in-SHIQ}.
Similarly, by Lemma~\ref{lemma:membership-test-concept}, each membership test can be resolved via a call to an $\NP$ oracle.

Since the algorithm performs $p(0)+2$ membership calls, which is a constant independent from $|\abox|$, it is enough to show that the subroutine \emph{computeMIN} can be realised by a polynomial time algorithm that uses logarithmically many calls to a $\NP$ oracle to prove the desired bound.
This is the case as \emph{computeMIN} can be achieved via a binary search over the interval $\llbracket 0, p(|\abox|) \rrbracket$
that uses minimality tests, see Lemma~\ref{lemma:minimum-test}, to navigate the interval.
Since each minimality test can be resolved via a call to an $\NP$ oracle, see Lemma~\ref{lemma:minimum-test}, the subroutine can be resolved in the desired complexity.

\paragraph{Role queries}
As above, by Lemma~\ref{lemma:membership-test-role} membership test can be resolved via a call to an $\NP$ oracle and so can minimality tests, see Lemma~\ref{lemma:minimum-test}. This proves the bullet in the case of the role queries.

\subsection{$\FPwithconstantcalls$}
Here we argue the complexity upper bound for concept counting queries on $\ELIbot$ and $\ELFbot$.
Observe that it is enough to argue that the subroutine \emph{computeMIN} can be realized by an $\FPwithconstantcalls$ algorithm.

For the two ontology languages, one can compute in polynomial time the set $\cstyle{C}^\kb$ of so-called certain answers, \emph{i.e.}\ individuals $\istyle{c}$ in the ABox such that
$\istyle{c} \in \cstyle{C}^\J$ holds for every model $\J$ of the KB \cite{ullrichmotiksattler07}. 
From the standard name assumption, it is thus clear that $\sizeof{\cstyle{C}^\kb} \leq \min(\spectrum{\kb}{q_{\cstyle{C}}})$.
Moreover, there exists a so called universal model $\I$ for which $\cstyle{C}^\I \cap \individuals(\abox) = \cstyle{C}^{\kb}$ \cite{egos2008}.
This model can be further collapsed to a model $\J'$ with a domain of size $|\abox| + d$, where $d$ is a constant not depending on $\abox$,
such that $|\cstyle{C}^{\J'}| \cap \individuals(\abox) = \cstyle{C}^\kb$ (we here rely on both $\ELIbot$ and $\ELFbot$ not supporting the combination of inverse roles and functionality restrictions). 
Thus, $\min(\spectrum{\kb}{q_{\cstyle{C}}}) \leq |\cstyle{C}^\kb| + d$.
Hence, $\min(\spectrum{\kb}{q_{\cstyle{C}}})$ satisfies $|\cstyle{C}^\kb| \leq \min(\spectrum{\kb}{q_{\cstyle{C}}}) \leq |\cstyle{C}^\kb|+d$ and, thus, can be computed using $d+1$ membership tests.
This implies that the algorithm uses a constant number of calls to an $\NP$ oracle to compute the spectrum and ends the proof of the bullet.

\subsection{$\FP$ for concepts and roles on $\dllitecore$}

For the last case recall that for $\dllitecore$ KBs we can can perform the satisfiability tests in polynomial time, see~\cite{DBLP:conf/kr/CalvaneseGLLR06}.

To compute $m$ observe that if there is some individual $\istyle{c}$ such that $\cstyle{C}(\istyle{c}) \in \abox$ then $0 \notin \spectrum{\kb}{q_\cstyle{C}}$.
Thus the shape of the spectrum is $\llbracket m , \infty \rrbracket$.
By \cite{maniere:thesis}, there is a polynomial time algorithm deciding $\spectrum{\kb}{q_\cstyle{C}} \subseteq \llbracket i,\infty \rrbracket$
and, thus, $m$ can be found in polynomial time by inspecting all numbers in the interval $\llbracket 1, \ell \rrbracket$.

If there is no such individual, we first test if $0 \notin \spectrum{\kb}{q_\cstyle{C}}$ by checking whether $\spectrum{\kb}{q_\cstyle{C}} \subseteq \llbracket 0,\infty \rrbracket$
and then we add a fresh individual $\istyle{c}$ and assertion $\cstyle{C}(\istyle{c})$ to the ABox to determine $m$ as previously.
This provides a polynomial time procedure and proves the bullet.

\end{toappendix}

\section{Conclusion}
\label{sec:conclusion}

We have characterized almost exhaustively the possible shapes of spectra for cardinality queries and proved that, in many settings, computing the proposed effective representation is $\FPwithlogcalls$-complete w.r.t.\ data complexity.
Whether an effective representation for the spectrum of a cardinality query over an $\ALCIF$ KB can be computed remains an open question, despite our work fully characterizing its possible shapes.
For $\dllitef$ KBs, we conjecture $\FP$ membership as the use of $\NP$ oracles may not be necessary and might be replaced by direct checks in $\P$ as those employed for $\dllitecore$ KBs, here used in a black-box manner.

	Departing from data complexity, it is readily verified that the algorithm proposed in Section~\ref{sec:complexity} provides a uniform procedure to compute the representation of a spectrum from an input KB and cardinality query, in all cases covered by Theorem~\ref{theorem:complexity-concepts-upper-bounds} (resp.\ by Theorem~\ref{theorem:complexity-roles-upper-bounds} for role cardinality queries).
	This notably relies on the polynomial provided by Lemma~\ref{lemma:polynomial-bound-for-conccepts} being computable given a TBox and a query.
	Hence, a careful inspection of the proof of the correctness of the algorithm could provide upper bounds for the combined complexity of spectrum computation.
	On the other hand, we believe that further obtaining meaningful lower bounds for such high-complexity functional classes is challenging.

We also emphasize that our investigation covers the `standard' meaning of the spectrum for a logical formula, being the possible sizes of its models:
it suffices to set $\cstyle{C} = \top$ in our results for concept cardinality queries.

We believe that it could be interesting to study the impact of our results on the closely related problem of answering (Boolean atomic) queries under the bag semantics.
While the semantics adopted in the present paper does not 
coincide with bag semantics, as discussed for example in \cite{nikolauetal:bag,DBLP:conf/ijcai/CalvaneseCLR20}, considerations regarding the spectra and some of the corresponding techniques might be adapted to this setting.

\clearpage

\section*{Acknowledgments}
The authors acknowledge the financial support by the Federal Ministry of Education and Research of Germany
and by the Sächsische Staatsministerium für Wissenschaft Kultur und Tourismus in the program Center of Excellence for AI-research
``Center for Scalable Data Analytics and Artificial Intelligence Dresden/Leipzig'',
project identification number: ScaDS.AI

Second author was supported by the DFG project LU1417/3-1 QTEC.

\bibliography{main}


\end{document}